\documentclass{article} 
\usepackage{iclr2027_conference,times}

\usepackage{amsmath,amsfonts,bm}

\def\eqref#1{Eq.~(\ref{#1})}

\def\1{\bm{1}}

\DeclareMathAlphabet{\mathsfit}{\encodingdefault}{\sfdefault}{m}{sl}
\SetMathAlphabet{\mathsfit}{bold}{\encodingdefault}{\sfdefault}{bx}{n}

\newcommand{\E}{\mathbb{E}}

\newcommand{\R}{\mathbb{R}}

\usepackage[utf8]{inputenc}
\usepackage[T1]{fontenc}
\usepackage[hidelinks]{hyperref}
\hypersetup{
    colorlinks,
    linkcolor={red!80!black},
    citecolor={blue!60!black},
    urlcolor={blue!60!black}
}
\usepackage{url}
\usepackage{booktabs}
\usepackage{multirow}
\usepackage{amsfonts}
\usepackage{nicefrac}
\usepackage{microtype}
\usepackage{xcolor}
\usepackage{amsmath, amssymb, amsthm}
\usepackage{graphicx}
\usepackage{algorithm}
\usepackage{algpseudocode}

\usepackage{sidecap}

\def\f{\mathbf{f}}
\def\w{\mathbf{w}}
\def\x{\mathbf{x}}

\def\z{\mathbf{z}}
\def\y{\mathbf{y}}
\def\E{\mathcal{E}}
\newcommand{\bmu}{\boldsymbol{\mu}}
\newcommand{\bsig}{\boldsymbol{\sigma}}
\newcommand{\Prob}{\mathbb{P}}
\newcommand{\Expec}{\mathbb{E}}

\newtheorem{theorem}{Theorem}
\newtheorem{corollary}{Corollary}
\newtheorem{lemma}{Lemma}
\newtheorem{proposition}{Proposition}
\newtheorem{conjecture}{Conjecture}
\theoremstyle{definition}
\newtheorem{assumption}{Assumption}
\theoremstyle{remark}
\newtheorem{remark}{Remark}

\title{Gradient-Free Sampling from Generative Models via Stochastic Bounded Extremum Seeking}

\author{Alexander Scheinker \thanks{Webpage: \url{https://alexscheinker.github.io/}} \\
Instrumentation and Controls Group \\
Los Alamos National Laboratory \\ 
Los Alamos, NM 87544, USA \\
\texttt{ascheink@lanl.gov}
}

\iclrfinalcopy 
\begin{document}

\maketitle           

\begin{abstract}
We introduce a sampling approach for energy- and score-based generative models that requires no gradient evaluations of the model. Replacing the drift term that would normally contain the score $\nabla_\mathbf{x} \log p_\theta(\mathbf{x})$ with a high-frequency dithered cosine of the model's \textit{value}, $\sqrt{\alpha\omega}\,\cos(\omega t + k \log p_\theta(\mathbf{x}))$, produces, in the high-frequency averaging limit, Langevin Markov chain Monte Carlo for energy-based models and the reverse-time SDE of score-based diffusion. We prove that trajectories of the dithered It\^{o} SDE converge to those of the target SDE, driven by the same Brownian motion, uniformly on compact time intervals in probability, by an averaging argument that extends bounded extremum seeking (ES) to It\^{o} processes, with an explicit $O(\omega^{-1/2})$ mean-square rate under global bounds. The approach is not confined to smooth targets: it extends to $C^{1,1}$ energies with discontinuous curvature (without ellipticity requirement) and to Sobolev energies whose Hessians exist only off measure zero sets; for Lipschitz energies with gradient kinks the averaged limit remains well posed; the Krylov-R\"ockner integrability class is the boundary of provability. The approach provides a hard \textit{a priori} bound on the per-step update rate and applies to explicitly time-varying targets on finite horizons. Gradient-free pixel-space sampling is not competitive with well-tuned backpropagation-based samplers at practical evaluation budgets; the regime where the approach offers an advantage is latent-space sampling when the model is a black box and the target drifts in time. We demonstrate latent-space tracking for time-varying images on CelebA-HQ ($256{\times}256$) from limited 1D projection measurements and latent-space EBM sampling on CIFAR-10.
\end{abstract}

\section{Introduction}

Modern generative models are trained, sampled, and conditioned primarily through gradients of large neural networks. Score-based
diffusion~\citep{song2021scorebased,ho2020ddpm,songermon2019smld}, latent diffusion~\citep{rombach2022ldm}, flow matching~\citep{lipman2023flowmatching,liu2023rectifiedflow,albergo2023interpolants}, and energy-based models (EBMs)~\citep{du2019ebm,grathwohl2020jem,duvincentmordatch2021}
all rely on gradients of a trained network: guidance or likelihood signals for diffusion and flow models, and the energy gradient itself for EBMs. When the model is available inside an automatic-differentiation framework this is efficient: a backward pass costs a small constant multiple of a forward pass. Our interest is in settings where even this is unavailable or insufficient.

Three such regimes recur in practice. \emph{(i)}: The quantity of interest is computed by a simulator~\citep{cranmer2020sbi}, a black-box scoring function such as the round-trip self-consistency of a bidirectional diffusion rollout~\citep{scheinker2026round}, a quantized or compiled inference engine, or a model served behind an API, so that only function \emph{values} can be obtained~\citep{nesterov2017random}. \emph{(ii)}: The target distribution drifts continuously, which is the operating condition of feedback-based generative reconstruction in scientific instruments, such as evolving microstructures in dynamic coherent diffraction imaging~\citep{scheinker2024dynamic}, time-varying particle accelerator beam diagnostics~\citep{scheinker2026phaseflow4d,scheinker2024cdvae}, and more broadly of streaming inverse problems in which each new measurement redefines the posterior. Restarting an MCMC chain at every measurement is wasteful, and running a single gradient-based chain against a drifting target has no tracking guarantee. \emph{(iii)}: Sampling in the latent space of a pretrained encoder-decoder whose internals are frozen or otherwise outside the user's autodiff graph, as in
generative-prior inverse problems~\citep{bora2017csgm} and latent-space
molecular design~\citep{gomezbombarelli2018automatic}.

\begin{figure}[t]
\centering
\includegraphics[width=1.0\textwidth]{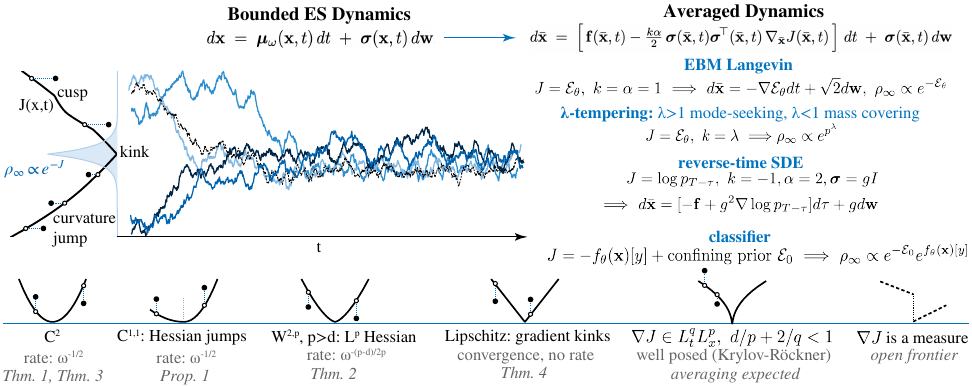}
\caption{Bounded ES dynamics results in various samplers depending on the choice of $(\mathbf{f},\boldsymbol{\sigma},k,\alpha,J)$, where $\bsig(\x,t) = \operatorname{diag}\!\big(g_1(\x,t),\dots,g_d(\x,t)\big)$, $\mu_{\omega,i}(\x,t) = f_i(\x,t) \;+\; g_i(\x,t)\,\sqrt{\alpha\omega_i}\,\cos\!\big(\omega_i t + k\,J(\x,t)\big)$. Averaged dynamics not changed by rough energies, but the provable results are. Above $C^2$ means the a.e.-equivalence class of $J$ contains a $C^2$ function. $J$ can have measure zero jumps.}
\label{fig:overview}
\end{figure}

This paper develops a single sampling procedure for all three regimes by replacing model gradients with a \emph{high-frequency dithered cosine} of the model's value. Given a function $J(\x,t)$ that can be evaluated but not differentiated by the user, consider the stochastic feedback law
\begin{equation}
    dx_i \;=\; \sqrt{2}\,\sqrt{\alpha\omega_i}\,\cos\!\big(\omega_i t + k\, J(\x,t)\big)\,dt \;+\; \sqrt{2}\,dw_i, \qquad i = 1,\ldots,d,
    \label{eq:es-intro}
\end{equation}
with distinct dither frequencies $\omega_i=\omega r_i$. As $\omega \to \infty$ the trajectories of \eqref{eq:es-intro} converge (in the precise sense of Theorem~\ref{thm:main}) to those of $d\bar\x = -k\alpha\,\nabla_{\bar\x} J(\bar\x, t)\,dt + \sqrt{2}\,d\w$. With $J = \E_\theta$ and $k\alpha=1$ this is Langevin MCMC for the EBM $p\propto e^{-\E_\theta}$; embedded in the appropriate drift with $J=\log p_\theta$ it reproduces the reverse-time SDE of score-based diffusion~\citep{anderson1982reverse,song2021scorebased}. The gradient $\nabla_\x J$ is never computed, only evaluated once per time step, shared across all $d$ coordinates. Flow-matching velocity networks are score networks up to an affine map~\citep{lipman2023flowmatching,albergo2023interpolants}, so they are covered exactly as score-only diffusion models are in Sec.~\ref{sec:inst}, (II): through the line-integral surrogate, not directly.  The dithered feedback law is \emph{bounded extremum seeking}, which originated in adaptive control of analytically unknown time-varying systems~\citep{scheinker2013BoundedES} and was generalized to discontinuous dithers in~\citep{scheinker2016bounded}. For generative-model sampling we extend the averaging analysis to stochastic differential equations, which is the subject of Theorem~\ref{thm:main} and its extensions, as summarized in Figure \ref{fig:overview}. The main contributions of this paper are:

\textbf{Theory:} We prove that multivariate stochastic bounded ES SDE converges to Langevin / reverse-time diffusion SDE, driven by the same Brownian motion, uniformly on compact time intervals in probability as the dither frequencies tend to infinity and prove convergence in distribution on path space. This is, to our knowledge, the first averaging analysis of bounded ES in the It\^{o} setting. Under global bounds the same proof yields the quantitative mean-square rate $\Expec\sup_{t\le T}\|\x_\omega-\bar\x\|^2\le C/\omega$, and we show that the same sampler converges on $C^{1,1}$ energies with discontinuous curvature and on Sobolev energies whose Hessian exists only off a null set; for Lipschitz energies with gradient kinks the averaged limit remains well posed. The boundary of provability is the Krylov-R\"ockner class.

\textbf{Algorithm 1:} The finite-step sampler is provided and guidance for choosing $\omega$, $\{r_i\}$, $\Delta t$, $k$, and $\alpha$.

\textbf{Instantiations and tempering, as corollaries:} Choices of $(\f,\bsig,k,\alpha,J)$ recover gradient-free Langevin for EBMs, gradient-free reverse-SDE sampling for models with an explicit log-density, gradient-free sampling from a classifier via the EBM--classifier correspondence~\citep{grathwohl2020jem}, and a single-knob $\lambda$-tempering whose stationary law is $p^\lambda$.

\textbf{Experiments mapping the operating regime:} We evaluate against gradient-based (SGLD, MALA) and zeroth-order \citep{roy2022zeroth} baselines at matched evaluation budgets, on a synthetic testbed, on latent-space EBM sampling (CIFAR-10), and on a time-varying latent tracking problem (CelebA-HQ).

\textbf{Scope of claims.} We do not claim that gradient-free ES sampling outperforms well-tuned gradient-based Langevin samplers in pixel space on standard benchmarks; the need to resolve the dither makes the per-step / per-quality tradeoff unfavorable there (Sec.~\ref{sec:practice}). The claim is a \emph{regime expansion}: the same algorithm, with guarantees, operates in settings where backpropagation is unavailable.

\section{Background}
\label{sec:background}

\subsection{Langevin MCMC and score-based diffusion}

Sampling from a target $p(\x) \propto \exp(-\E(\x))$ on $\R^d$ via overdamped Langevin dynamics~\citep{roberts1996langevin,welling2011sgld} simulates
\begin{equation}
    d\bar\x \;=\; -\nabla_{\bar\x}\E(\bar\x)\,dt \;+\; \sqrt{2}\,d\w \;=\; \nabla_{\bar\x}\log p(\bar\x)\,dt + \sqrt{2}\,d\w,
    \label{eq:langevin}
\end{equation}
whose unique stationary law is $p$. (The $\sqrt{2}$ matters: with unit diffusion the stationary law of $d\x=\nabla\log p\,dt+d\w$ is proportional to $p^2$, not $p$.) For modern EBMs $\E_\theta$ trained via contrastive-divergence-style objectives~\citep{du2019ebm,duvincentmordatch2021,grathwohl2020jem,nijkamp2020anatomy}, $\nabla_\x \E_\theta(\x)$ is obtained by automatic differentiation. Score-based generative models~\citep{songermon2019smld,song2021scorebased,ho2020ddpm} extend this to a continuum of noise levels; sampling integrates the reverse-time SDE~\citep{anderson1982reverse}
\begin{equation}
    d\bar\x \;=\; \big[\,\f(\bar\x,t) - g^2(t)\,\nabla_{\bar\x}\log p(\bar\x,t)\,\big]\,dt + g(t)\,d\w,
    \label{eq:reverse-sde}
\end{equation}
where the score $\nabla_\x \log p(\x,t)$ is approximated by a network $\mathbf{s}_\theta(\x,t)$ trained by denoising score matching~\citep{vincent2011dsm,songermon2019smld}. In \eqref{eq:reverse-sde} time runs backward from $T$ to $0$ ($dt<0$, $\w$ a reverse-time Brownian motion); in the forward computational time $\tau=T-t$ used by every integrator, including Algorithm~\ref{alg:es} below, the same dynamics read $d\bar\x=[-\f(\bar\x,T-\tau)+g^2(T-\tau)\nabla_{\bar\x}\log p(\bar\x,T-\tau)]\,d\tau+g(T-\tau)\,d\w$, with the score drift \emph{contracting} mass onto the data modes. Both \eqref{eq:langevin} and \eqref{eq:reverse-sde} are ordinarily implemented by differentiating through $\theta$.

\subsection{Bounded extremum seeking}

Bounded ES~\citep{scheinker2013BoundedES,scheinker2016bounded} is a model-independent feedback law that steers a parameter vector toward the (possibly time-varying) minimizer of an unknown cost $J(\x,t)$ using only evaluations of $J$. Its defining feature is the dithered cosine update of \eqref{eq:es-intro}, with two properties that distinguish it from generic finite-difference, simultaneous-perturbation, or smoothing-based gradient estimates \citep{spall1992spsa,nesterov2017random,roy2022zeroth}: (i)~the update rate obeys the hard \emph{a priori} bound $|dx_i/dt - \text{noise}| \le \sqrt{\alpha\omega_i}\,|g_i|$ \emph{regardless of the magnitude of $J$}, and (ii)~a single scalar evaluation of $J$ per time instant drives all $d$ coordinates simultaneously, rather than $d{+}1$ or $2$ evaluations per estimated gradient. Existing averaging analyses of bounded ES concern deterministic ODEs; we extend the analysis to It\^{o} SDEs.

\section{Gradient-Free Sampling via Stochastic ES}
\label{sec:method}

\subsection{Setup and main result}
\label{sec:main}

Throughout, $\w$ is a standard $d$-dimensional Brownian motion on a filtered probability space, and we write the dithered system in standard It\^{o} drift--diffusion form. For $\x\in\R^d$ consider
\begin{equation}
    d\x \;=\; \bmu_\omega(\x,t)\,dt \;+\; \bsig(\x,t)\,d\w,
    \qquad
    \bsig(\x,t) \;=\; \operatorname{diag}\!\big(g_1(\x,t),\dots,g_d(\x,t)\big),
    \label{eq:es-sde}
\end{equation}
\begin{equation}
    \mu_{\omega,i}(\x,t) \;=\; f_i(\x,t) \;+\; g_i(\x,t)\,\sqrt{\alpha\omega_i}\,\cos\!\big(\omega_i t + k\,J(\x,t)\big),
    \qquad \omega_i = \omega\, r_i,
    \label{eq:es-drift}
\end{equation}
with $\f=(f_1,\dots,f_d)$ the drift, $\bsig$ a diagonal diffusion, $k$,$\alpha$ gains, and $J$ a black-box target function.

\begin{assumption}[Regularity]\label{ass:reg}
$\f(\cdot,t)$ is continuous and locally Lipschitz in $\x$, uniformly for $t$ in compacts; each $g_i$ and $J$ is $C^2$ in $\x$ and $C^1$ in $t$, with continuous mixed derivatives $\partial_t\nabla_\x g_i$ and $\partial_t\nabla_\x J$, and with $\f$, $g_i$, $J$, all first- and second-order $\x$-derivatives, the first-order $t$-derivatives, and the mixed derivatives of $g_i,J$ bounded on $K\times[0,T]$ for every compact $K\subset\R^d$ and $T>0$.
\end{assumption}

\begin{assumption}[Dither frequencies]\label{ass:freq}
$\omega_i=\omega r_i$ with $r_i\neq r_j \ \forall \ i\neq j$ $r_i>0$, $i=1,\dots,d$.
\end{assumption}

\begin{assumption}[Well-posedness]\label{ass:growth}
$\f$, $\bsig$, and $\bsig\bsig^{\!\top}\nabla_\x J$ satisfy a linear-growth bound $\|\cdot\|\le C(1+\|\x\|)$ uniformly for $t\in[0,T]$, so that \eqref{eq:es-sde} and \eqref{eq:averaged-sde} below admit unique global strong solutions.
\end{assumption}

\begin{theorem}[Stochastic ES averaging]
\label{thm:main}
Under Assumptions~\ref{ass:reg}--\ref{ass:growth}, let $\x_\omega$ solve \eqref{eq:es-sde} and let $\bar\x$ solve
\begin{equation}
    d\bar\x \;=\; \underbrace{\Big[\,\f(\bar\x,t) - \tfrac{k\alpha}{2}\,\bsig(\bar\x,t)\bsig^{\!\top}(\bar\x,t)\,\nabla_{\bar\x} J(\bar\x,t)\,\Big]}_{\textstyle \bar\bmu(\bar\x,t)}\,dt \;+\; \bsig(\bar\x,t)\,d\w,
    \label{eq:averaged-sde}
\end{equation}
driven by the \emph{same} Brownian motion and same initial condition $\x_\omega(0)=\bar\x(0)=\x_0$. For any $T>0$,
\begin{equation}
    \sup_{t\in[0,T]}\big\|\x_\omega(t)-\bar\x(t)\big\|
    \;\longrightarrow\; 0 \quad \text{in probability as } \omega\to\infty;
\end{equation}
for every $\delta>0$, $\varepsilon>0$ $\exists \ \omega^*=\omega^*(\delta,\varepsilon,T)$ such that
$\Prob\big(\sup_{t\le T}\|\x_\omega(t)-\bar\x(t)\|>\delta\big)<\varepsilon$, $\forall \omega\ge\omega^*$.
\end{theorem}

The proof is in Appendix~\ref{supp:proof} with main steps: (i) localization by exit times from a large ball; (ii) It\^{o} integration by parts; (iii) an oscillatory-integral lemma (Lemma~\ref{lem:osc-main-ref}) showing that integrals of the form $\int_0^t \Phi(\x_\omega(s),s)\,e^{\mathrm i c\omega s}\,ds$ with $c\neq0$ vanish in probability; (iv) extraction of the resonant terms, which produce the drift correction $-\tfrac{k\alpha}{2}\bsig\bsig^{\!\top}\nabla J$; and (v) an $L^2$ Gr\"onwall comparison with $\bar\x$. Intuitively, the state has a small ripple of amplitude $O(\omega^{-1/2})$ so that $x_i(t) \approx \bar x_i(t) + g_i\sqrt{\alpha/\omega_i}\,\sin(\omega_i t + kJ)$, where $g_i$ is the coefficient of the dither ($g_i=\sqrt2$ in \eqref{eq:es-intro}). Plugging this ripple into the cosine, the dither ``feels'' the change in $J$ that its own motion induces; the product of the dither and its induced phase shift has a nonvanishing time average, and the average is $-\tfrac{k\alpha}{2}\,g_i^2\,\partial J/\partial x_i$.

\subsection{The finite-step algorithm and practical guidance}
\label{sec:practice}

Algorithm~\ref{alg:es} is the procedure used in all experiments. The methods for choosing ES gain $k$, amplitude $\alpha$, and frequencies $\omega_i$ are studied in more detail in Appendix \ref{supp:ES_params} and \ref{supp:numerics}. Any pairwise-distinct positive constants satisfy Assumption~\ref{ass:freq}; we use $r_i$ evenly spaced (or i.i.d.\ uniform) in $[1,N]$ with modest $N\in[1.75,2]$. Euler-Maruyama must resolve the fastest dither: we use $\omega r_{\max}\,\Delta t \le 2\pi/10$ so that at least $10$ steps are required per fastest dither period. The product $k\alpha$ sets the strength of the emergent drift (Eq.~\ref{eq:averaged-sde}) and the temperature of the averaged law ($\rho_\infty\propto e^{-k\alpha\E_\theta}$ in (I$_\lambda$)). Furthermore, if $J$ is changing slowly, we see from Eq. \ref{eq:es-drift} that components will oscillate with amplitude proportional to $\sqrt{\alpha/\omega_i}$, therefore in practice once $\omega_i$ is fixed, $\alpha$ is slowly increased until the dither is big enough to produce a good signal in the cost function and then $k$ is slowly increased until a reasonable rate of convergence is achieved. Excessive gain $k\alpha$ at any finite $\omega$ can destabilize the ES dynamics.

\subsection{Instantiations}
\label{sec:inst}

\begin{table}[t]
\caption{Parameter dictionary: choices of $(\f,\bsig,k,\alpha,J)$ in \eqref{eq:es-sde}--\eqref{eq:es-drift} and the resulting averaged dynamics \eqref{eq:averaged-sde}. $\E_\theta$ denotes an energy, $p$ a density, $\lambda>0$ a tempering exponent.}
\label{tab:dict}
\begin{center}
\small\setlength{\tabcolsep}{4pt}
\begin{tabular}{@{}llllll@{}}
\toprule
\textbf{Instantiation} & $\f$ & $\bsig$ & $k$ & $\alpha$ & \textbf{averaged SDE / stationary law} \\
\midrule
(I) EBM Langevin & $\mathbf{0}$ & $\sqrt2\,I$ & $1$ & $1$ & $d\bar\x=-\nabla\E_\theta\,dt+\sqrt2\,d\w$; \ $\rho_\infty\propto e^{-\E_\theta}$\\
(I$_\lambda$) tempered & $\mathbf{0}$ & $\sqrt2\,I$ & $\lambda$ & $1$ & $d\bar\x=-\lambda\nabla\E_\theta\,dt+\sqrt2\,d\w$; \ $\rho_\infty\propto e^{-\lambda\E_\theta}$\\
(II) reverse SDE ($\tau{=}T{-}t$) & $-\f(\x,T{-}\tau)$ & $g(T{-}\tau)I$ & $-1$ & $2$ & $d\bar\x=[-\f+g^2\nabla\log p_{T-\tau}]d\tau+g\,d\w$ \\
(III) classifier & $\mathbf{0}$ & $\sqrt2\,I$ & $1$ & $1$ & (I) with $\E_\theta(\x\,|\,y)=-f_\theta(\x)[y]$ \\
\bottomrule
\end{tabular}
\end{center}
\end{table}

\begin{algorithm}[t]
\caption{ES sampling (Euler--Maruyama form). One evaluation of $J$ per step, for any $d$.}
\label{alg:es}
\begin{algorithmic}[1]
\Require evaluator $J(\cdot,\cdot)$; base drift $\f$, diffusion $\bsig=\operatorname{diag}(g_i)$; gains $k$, $\alpha$; base frequency $\omega$; distinct ratios $\{r_i\}_{i=1}^d\subset[1,r_{\max}]$; step $\Delta t$; steps $M$; initial $\x^{(0)}$
\For{$m=0,\dots,M-1$}
  \State $t_m \gets m\,\Delta t$; \quad $J_m \gets J(\x^{(m)}, t_m)$ \Comment{single forward evaluation}
  \State draw $\boldsymbol{\xi}\sim\mathcal N(0, I_d)$
  \For{$i=1,\dots,d$} \Comment{vectorized in practice}
    \State $x^{(m+1)}_i \gets x^{(m)}_i + \big[f_i(\x^{(m)},t_m) + g_i(\x^{(m)},t_m)\sqrt{\alpha\omega r_i}\,\cos(\omega r_i t_m + k J_m)\big]\Delta t + g_i(\x^{(m)},t_m)\sqrt{\Delta t}\,\xi_i$
  \EndFor
\EndFor
\State \Return $\{\x^{(m)}\}$ (samples after burn-in, or the full trajectory for tracking)
\end{algorithmic}
\end{algorithm} 

\paragraph{(I) Gradient-free Langevin MCMC for EBMs.}
With $\f\equiv\mathbf 0$, $g_i\equiv\sqrt2$, $k,\alpha=1$, \& $J=\E_\theta$,
\begin{equation}
    dx_i \;=\; \sqrt{2\omega_i}\,\cos\!\big(\omega_i t + \E_\theta(\x,t)\big)\,dt \;+\; \sqrt{2}\,dw_i,
    \label{eq:es-ebm}
\end{equation}
with average Langevin diffusion dynamics \eqref{eq:langevin} with potential $\E_\theta$ and stationary law $\propto e^{-\E_\theta}$.

\paragraph{(I$_\lambda$) Single-knob $\lambda$-tempering.}
Replacing the unit gain in \eqref{eq:es-ebm} by $k=\lambda>0$ yields averaged dynamics $d\bar\x=-\lambda\nabla\E_\theta\,dt+\sqrt2\,d\w$ with stationary law
\begin{equation}
    \rho_\infty(\x)\;\propto\;e^{-\lambda \E_\theta(\x)}\;\propto\;p^{\lambda}(\x)
    \label{eq:tempering}
\end{equation}
(proof via the stationary Fokker--Planck equation, Appendix~\ref{supp:corollaries}). $\lambda=1$ recovers the target; $\lambda>1$ is mode-seeking; $\lambda<1$ is mass-covering. Unlike parallel tempering~\citep{earl2005parallel} or annealed Langevin~\citep{songermon2019smld}, no schedule, replica exchange, or noise-conditional training is required---the tempering exponent \emph{is} the ES gain. Figure~\ref{fig:alpha} illustrates the effect.

\paragraph{(II) Gradient-free reverse SDE for models with explicit log-density.}
Write the reverse-time SDE \eqref{eq:reverse-sde} in the forward computational time $\tau=T-t$ (Sec.~\ref{sec:background}), so that it is integrated with positive increments. With base drift $-\f(\x,T-\tau)$, $\bsig=g(T-\tau)I$ from the forward diffusion, $k=-1$, $\alpha=2$, and $J(\x,\tau)=\log p(\x,T-\tau)$, \eqref{eq:es-sde} becomes
\begin{equation}
    dx_i \;=\; \big[\,-f_i(\x,T-\tau) + g(T-\tau)\sqrt{2\omega_i}\,\cos\!\big(\omega_i \tau - \log p(\x,T-\tau)\big)\,\big]\,d\tau + g(T-\tau)\,dw_i,
    \label{eq:es-diffusion}
\end{equation}
whose averaged dynamics, $d\bar\x=[-\f+g^2\nabla\log p_{T-\tau}]\,d\tau+g\,d\w$, are exactly \eqref{eq:reverse-sde} run in forward time: the sign $k=-1$ (equivalently, $J=-\log p$ with $k=1$) is what makes the emergent drift $+g^2\nabla\log p$ contract onto the data modes when Algorithm~\ref{alg:es} steps forward. This instantiation requires a \emph{scalar} log-density evaluator: normalizing flows~\citep{rezende2015nf}, energy-parameterized diffusion models~\citep{salimans2021ebm}, or any model exposing $\log p_t$ up to an additive constant (constants shift the cosine phase globally and drop out of the averaged drift). Score-only networks of the DDPM/EDM type do not expose $\log p_t$ and fall outside this instantiation, unless one forms the line-integral surrogate $\hat J(\x)=\int_0^1 \langle \mathbf s_\theta(\x_0+u(\x-\x_0),t),\,\x-\x_0\rangle\,du$, exact when $\mathbf s_\theta$ is conservative, at the cost of extra evaluations per step and a bias when it is not.

\paragraph{(III) Gradient-free sampling from a classifier.}
By the EBM-classifier correspondence~\citep{grathwohl2020jem}, a $K$-class classifier $f_\theta:\R^d\to\R^K$ defines class-conditional energies $\E_\theta(\x\,|\,y)=-f_\theta(\x)[y]$. Substituting into \eqref{eq:es-ebm} samples from any pretrained classifier using only its forward pass, which need not sit in an autodiff graph. Since a logit $f_\theta(\x)[y]$ trained by cross-entropy is in general unbounded above and $e^{f_\theta(\x)[y]}$ need not be integrable on $\R^d$, this instantiation presupposes a confining prior: either an intrinsically dissipative classifier \citep{grathwohl2020jem} or a non-zero base drift $\f=-\nabla\E_0$ implementing a prior density $\propto e^{-\E_0}$ (e.g., a Gaussian in latent space), so that the averaged dynamics target the well-posed density $\propto e^{-\E_0}\,e^{f_\theta(\x)[y]}$.

\subsection{Latent-space sampling}
\label{sec:latent}

Applied in the latent space of a pretrained encoder--decoder $(E,D)$, ES sampling becomes practical. Given an energy or log-density on latent codes $\z\in\R^{d_z}$, $d_z\ll d$, run \eqref{eq:es-ebm} or \eqref{eq:es-diffusion} on $\z$ and decode $\hat\x=D(\z)$. Three points: (i)~$d_z$ is small ($3$--$512$ in our experiments), so the frequency-gap constraint of Sec.~\ref{sec:practice} is mild and moderate $\omega$ suffices; (ii)~the decoder is used strictly as a black box---relevant when $D$ is frozen, exported, or quantized; (iii)~latent diffusion~\citep{rombach2022ldm} and latent EBMs are the dominant deployment pattern, so this is the natural setting. The geometry of the latent energy landscape governs whether ES (or any local sampler) succeeds there.

\subsection{CelebA-HQ tracking statistics and derivative-free benchmark}
\label{sec:tracking}

\begin{figure}[t]
\centering
\includegraphics[width=1.0\textwidth]{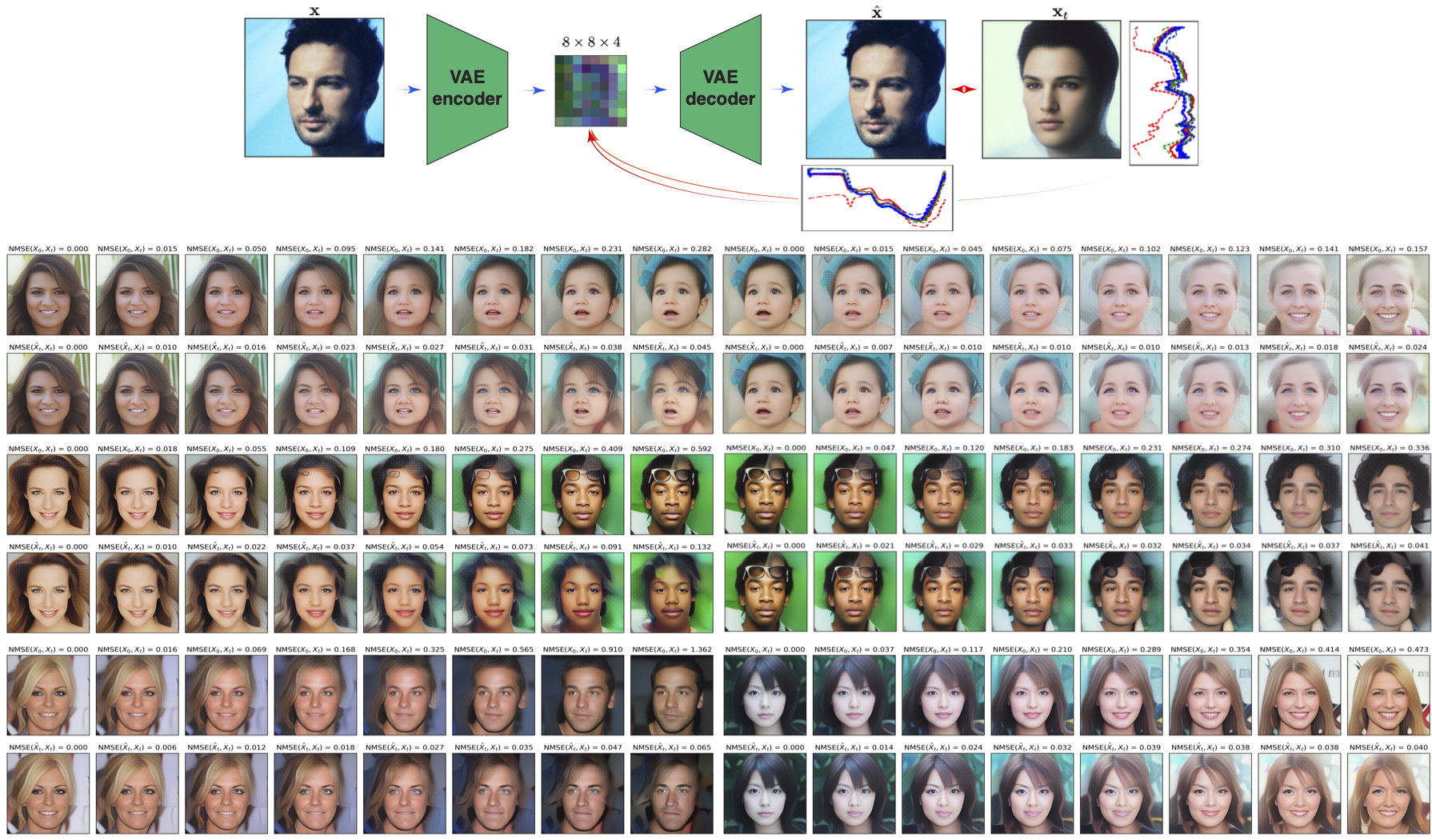}
\caption{Latent-space tracking from 1D RGB projections of time-varying images.}
\label{fig:CELEB_VAE}
\end{figure}

For tracking time-varying CELEB-HQ images, the setup is shown in Figure \ref{fig:CELEB_VAE}. Time-varying targets are generated by
spherical interpolation in latent space between pairs of CelebA-HQ images over $8192$ steps, for $500$ images each from the training and the held-out data, with $3$ seeds per image. The tuned latent has shape
$8{\times}8{\times}4$ ($d=256$ parameters) and is decoded to an RGB image $\hat I[n]=D(z[n])$ of size $256{\times}256{\times}3$. The bounded-ES dynamics used in this experiment are the quadrature-pair variant
\begin{equation}
dx_i=\cos\!\big(\omega_i t+k\,J(\x,t)\big)\sqrt{\alpha\,\omega_i}\,dt+\sigma\,dW_i
\quad(\text{even }i;\ \sin(\cdot)\ \text{for odd }i),
\label{eq:supp-track-sde}
\end{equation}
with $x=\mathrm{vec}(z)$, integrated by Euler--Maruyama with step $h$:
\begin{equation}
x_i[n{+}1]=x_i[n]+\cos\!\big(\omega_i nh+kJ[n]\big)\sqrt{\alpha\,\omega_i}\,h
+\sigma\sqrt h\,\xi_i[n],\qquad
\xi_i[n]\sim\mathcal N(0,1)\ \text{i.i.d.},
\label{eq:supp-track-em}
\end{equation}
($\sin$ for odd $i$), with $k=1$, $\alpha=0.03$, $\sigma=10^{-3}$, the $\omega_i$
linearly spaced over $[1.0,3.5]$ with $8{\times}8{\times}2=128$ distinct
values (each $\cos/\sin$ pair shares one frequency), and $h=2\pi/17.5$.
With $\sigma=10^{-3}$ this benchmark runs in the near-deterministic
mode-seeking regime of Appendix~\ref{supp:num-optim}: it tests the tracking behavior of bounded ES as a derivative-free optimizer of a kinked objective, and the non-degenerate-noise theory of the Lipschitz rung (Proposition~\ref{prop:lip}, Theorem~\ref{thm:occ}), whose ellipticity constant would here be $\lambda_0=10^{-6}$, does not formally cover it. Hyperparameters were selected by a sweep over repeated optimization of a training image ($3$ seeds, first $2048$ steps).

\paragraph{Feedback cost.} $J[n]$ has four components; images are in
$[0,255]$ for the projection terms and in the decoder's native $[-1,1]$
range (with standard VGG16 preprocessing) for the perceptual term.
(i)--(ii)~Projection terms: for an image $I$ with pixels $I_{r,c,k}$
($r,c$ spatial, $k$ channel), the projections are
\begin{equation}
P_{x/y}I(c,k)=\sum_{r/c=1}^{256} I_{r,c,k}\ \in\ \R^{256\times3},
\
c_{x/y}[n]=\tfrac1{255}\,
\mathrm{MAE}\big(P_{x/y}(\hat I[n]),\,P_{x/y}(I^{\mathrm{tgt}}[n])\big),
\end{equation}
with $\mathrm{MAE}$ the mean absolute difference over all entries.
(iii)~A perceptual term from an ImageNet-pretrained frozen VGG16: with
$\phi_\ell$ the feature maps of layers
\texttt{block1\_conv2}, \texttt{block2\_conv2}, \texttt{block3\_conv3},
\texttt{block4\_conv3}, \texttt{block5\_conv3},
$c_p[n]=\sum_{\ell=1}^5\mathrm{mean}\big[(\phi_\ell(\hat I[n])
-\phi_\ell(I^{\mathrm{tgt}}[n]))^2\big]$.
(iv)~A latent-norm regularizer
$c_z[n]=\big(\lVert z[0]\rVert_2-\lVert z[n]\rVert_2\big)^2$. The total is
\begin{equation}
J[n]=c_x[n]+c_y[n]+c_p[n]/10^{8}+c_z[n],
\label{eq:supp-track-J}
\end{equation}
the $10^{-8}$ weight compensating the scale of the raw summed VGG feature MSE. $J$ drives ES; the performance metric is the NMSE computed directly between full images, as summarized in Table \ref{tab:supp-celeba-track500}.

\begin{table}[h]
\centering\small
\caption{Tracking study: $500$ held-out pairs $\times$ $3$
seeds, mean$\pm$SEM over the $1{,}500$ runs.}
\label{tab:supp-celeba-track500}
\scriptsize\setlength{\tabcolsep}{2.2pt}
\begin{tabular}{@{}lccccccc@{}}
\toprule
NMSE & $t{=}1170$ & $2340$ & $3510$ & $4680$ & $5850$ & $7020$ & $8191$\\
\midrule
Val, drift (vs.\ $t{=}0$) & $.025{\pm}.0006$ & $.078{\pm}.0022$ & $.146{\pm}.0040$ & $.221{\pm}.0061$ & $.305{\pm}.0086$ & $.400{\pm}.0119$ & $.510{\pm}.0167$\\
Val, bounded ES & $.010{\pm}.0002$ & $.019{\pm}.0004$ & $.029{\pm}.0005$ & $.040{\pm}.0007$ & $.052{\pm}.0009$ & $.067{\pm}.0013$ & $.083{\pm}.0017$\\
\bottomrule
\end{tabular}
\end{table}

\begin{table}[h]
\centering\small
\caption{Tracked vs.\ target image quality (left), and its evolution over the tracking horizon (right).}
\label{tab:supp-celeba-quality}
\begin{tabular}{@{}lc@{}}
\toprule
Metric & Validation ($500$ pairs)\\
\midrule
PSNR (dB) $\uparrow$ & $21.08\pm3.30$\\
SSIM $\uparrow$ & $0.855\pm0.095$\\
FID $\downarrow$ & $8.68\pm0.11$\\
KID $\times10^{3}$ $\downarrow$ & $5.23\pm0.47$\\
\bottomrule
\end{tabular}
\qquad
\scriptsize\setlength{\tabcolsep}{2.2pt}
\begin{tabular}{@{}lcccccccc@{}}
\toprule
& $t{=}0$ & $1170$ & $2340$ & $3510$ & $4680$ & $5850$ & $7020$ & $8191$\\
\midrule
SSIM $\uparrow$ & $1.000$ & $0.949$ & $0.907$ & $0.868$ & $0.832$ & $0.797$ & $0.761$ & $0.729$\\
PSNR (dB) $\uparrow$ & $\infty$ & $26.50$ & $23.62$ & $21.84$ & $20.48$ & $19.36$ & $18.33$ & $17.41$\\
\bottomrule
\end{tabular}
\end{table}

FID and KID are computed per seed against the corresponding $4{,}000$ target images ($500$ pairs $\times$ $8$ checkpoints along the trajectory), mean$\pm$std across $3$ seeds; PSNR/SSIM are paired and pooled over pairs$\times$seeds$\times$checkpoints (PSNR over the seven checkpoints $t>0$, where it is finite; SSIM over all eight, where $t{=}0$ contributes a trivial value of $1$). Summary in Table \ref{tab:supp-celeba-quality}. We also compared bounded ES against, all at an identical budget of $8192$ forward measurements per run: SPSA~\citep{spall1992spsa}, Spall's two-measurement simultaneous perturbation with constant gains for tracking; GS-ES, the Gaussian-smoothing evolution-strategies estimator with antithetic sampling~\citep{nesterov2017random,salimans2017es};
central coordinate-wise finite differences (FD); and an
automatic-differentiation oracle (exact gradients backpropagated through the trained decoder, with Adam). The benchmark covers $219$ image pairs $\times$ $3$ seeds per method ($657$ runs each). SPSA, GS-ES, and FD each received \emph{two rounds} of grid search over their gains and estimator hyperparameters (a coarse round followed by a refined round around the best cell), and are reported at their best-tuned settings; bounded ES runs at the fixed setting of the tracking protocol above. Results summarized in Table \ref{tab:supp-celeba-df}.

\begin{table}[h]
\centering\small
\caption{Final image-NMSE at
$t{=}8191$. $\pm$ is the std over all $219\times3$ runs;
median (IQR) across the $219$ pairs after seed averaging (the per-pair
distributions are right-skewed). ``Updates'' is the number of parameter
updates each method uses, $m$ is the tuned number of GS-ES antithetic pairs per update.}
\label{tab:supp-celeba-df}
\setlength{\tabcolsep}{4pt}
\begin{tabular}{@{}lcccc@{}}
\toprule
Method & Meas./update & Updates & Final NMSE (mean$\pm$std) & median (IQR)\\
\midrule
Bounded ES (ours) & $1$ & $8192$ & $\mathbf{0.080\pm0.035}$ & $\mathbf{0.072}$ ($0.055$--$0.101$)\\
SPSA              & $2$ & $4096$ & $0.142\pm0.059$ & $0.131$ ($0.104$--$0.181$)\\
FD (central)      & $512$ & $16$ & $0.276\pm0.312$ & $0.219$ ($0.169$--$0.307$)\\
GS-ES (antithetic)& $2m$ & $8192/2m$ & $1.168\pm0.627$ & $1.028$ ($0.773$--$1.348$)\\
AutoDiff oracle   & --- & --- & $0.048\pm0.032$ & $0.041$ ($0.024$--$0.063$)\\
\bottomrule
\end{tabular}
\end{table}

\subsection{Latent representation and diffusion energy-based prior}

Images $x\in[-1,1]^{32\times32\times3}$ from CIFAR-10~\citep{krizhevsky2009cifar} are represented in the latent space of a frozen
convolutional variational autoencoder with self-attention~\citep{kingma2014vae,higgins2017betavae}, whose encoder
$q_\phi(z\,|\,x)=\mathcal N\!\big(\mu_\phi(x),\operatorname{diag}\sigma^2_\phi(x)\big)$ and decoder $D_\theta$ act on a latent
$z\in\mathbb R^{8\times8\times4}$ ($d=256$, a $12\times$ compression of the pixel space) and are trained with
\begin{equation}
  \mathcal L_{\mathrm{VAE}} = \mathbb E_{q_\phi(z|x)}\!\left[\lVert x - D_\theta(z)\rVert_2^2\right]
  + \beta\,\mathrm{KL}\!\left(q_\phi(z\,|\,x)\,\|\,\mathcal N(0,I)\right), \qquad \beta = 10^{-4}. 
  \label{eq:vae}
\end{equation}
On the standardized latents (see Appendix \ref{supp:CIFAR10}), we train a diffusion energy-based model, a latent diffusion model~\citep{rombach2022ldm,ho2020ddpm} whose score is the gradient of an explicit energy~\cite{salimans2021ebm,gao2021drl,du2023reduce}. With forward process $z_t=\sqrt{\bar\alpha_t}\,z_0+\sigma_t\,\epsilon$, $\sigma_t=\sqrt{1-\bar\alpha_t}$, $\epsilon\sim\mathcal N(0,I)$, energy is parameterized through a denoising U-Net $f_\theta$
\begin{equation}
  E_\theta(z,t) = \lVert z - f_\theta(z,t)\rVert_2^2 / 2\,\sigma_t, \qquad
  \hat z_0(z,t) = f_\theta(z,t)/\sqrt{\bar\alpha_t},
  \label{eq:energy}
\end{equation}
so $-\nabla_z E_\theta(z,t)\approx\sigma_t\,\nabla_z\log p_t(z)$ is a conservative score estimate and $\hat z_0$ the corresponding denoised estimate. The network is trained across noise levels with a denoising objective and its EMA weights are used at test time. Decoding $50{,}000$ samples of the latent EBM through $D_\theta$ gives a Fr\'echet Inception distance~\cite{heusel2017fid} of $14.3$ against the CIFAR-10 training set and $16.4$ against the test set under the legacy TensorFlow protocol of the original FID implementation. For tracking, the prior is evaluated at a single low noise level $t_p=0.05$ and standardised by its mean $\mu_p$ and standard deviation $s_p$ over $2{,}000$ held-out real latents,
\begin{equation}
  J_p(z) = \sigma_{t_p}\,E_\theta(z,t_p) = \tfrac12\,\lVert z - f_\theta(z,t_p)\rVert_2^2.
  \label{eq:prior}
\end{equation}

\begin{figure}[t]
\centering
\includegraphics[width=0.49\textwidth]{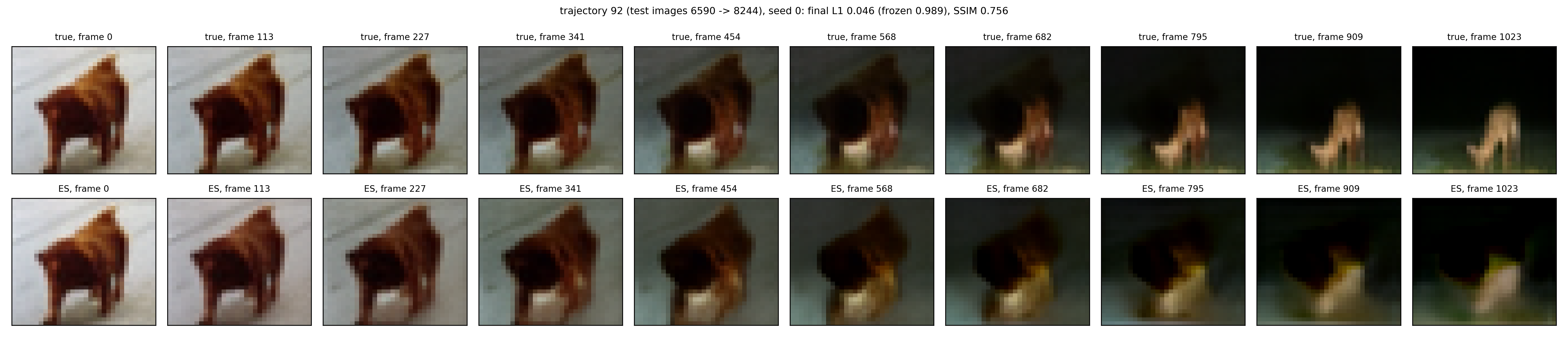}
\includegraphics[width=0.49\textwidth]{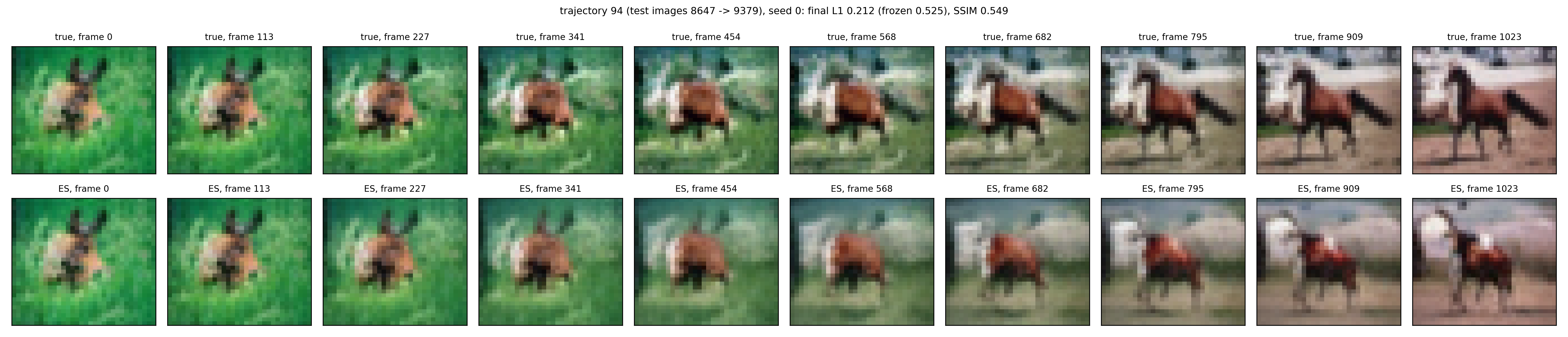}
\includegraphics[width=0.49\textwidth]{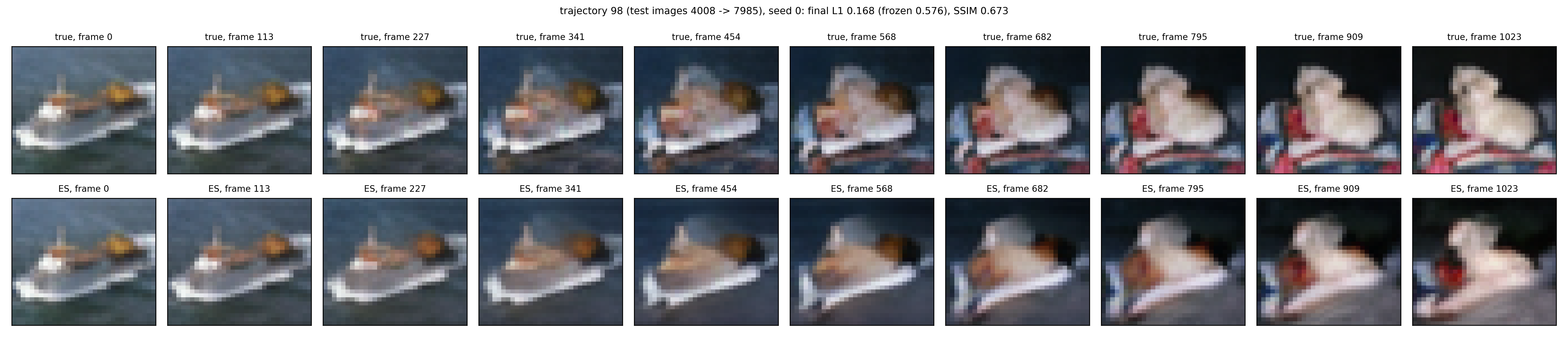}
\includegraphics[width=0.49\textwidth]{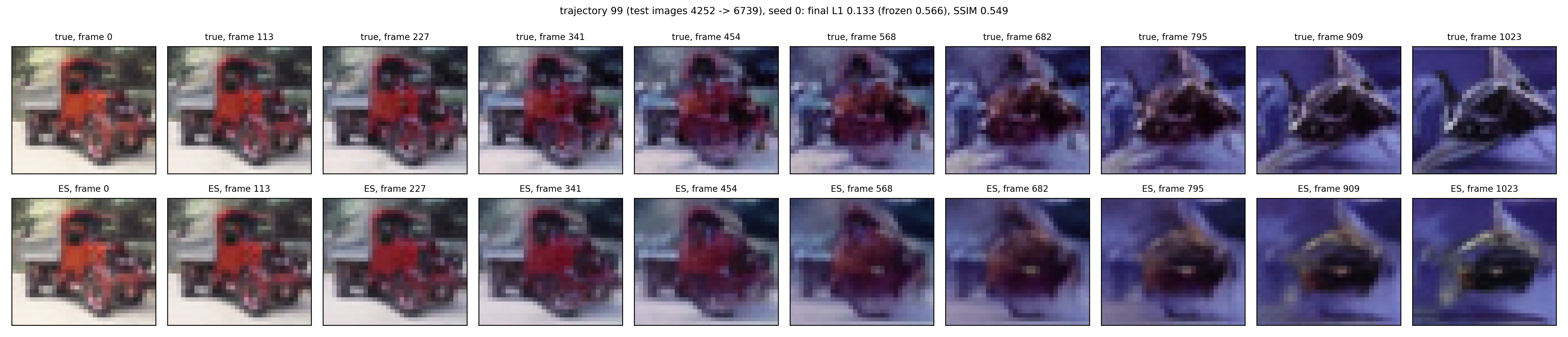}
\caption{ES dynamics results for tracking time-varying CIFAR-10 images based on 1D projections.}
\label{fig:CIFAR10}
\end{figure}
 
The unknown target is a moving image $x^*_n=D_\theta(z^*_n)$, $n=0,\dots,1023$, whose latent follows a spherical interpolation between
the latents of two random test images, a constant-speed path with $\lVert z^*_{n+1}-z^*_n\rVert_2\approx0.0125$ per frame. The tracker
never sees $x^*_n$; it observes only four one-dimensional projections, $p_n=A\,x^*_n$, the sums along rows, columns, diagonals and
anti-diagonals ($32+32+63+63=190$ bins per colour channel, $570$ numbers per frame), the measurement model of few-view tomography.
The data-fit cost of a candidate latent is the $\ell_1$ projection mismatch, scaled by $C_{\mathrm{scale}}$, the median mismatch
between projections of unrelated real images. Tracking cost is
\begin{equation}
  C(z;p_n)=\lVert A\,D_\theta(z)-p_n\rVert_1, \qquad
  J_n(z)=\frac{J_p(z)-\mu_p}{s_p}+\lambda\,\frac{C(z;p_n)}{C_{\mathrm{scale}}}, \qquad \lambda=100 .
  \label{eq:cost}
\end{equation}
Bounded ES minimizes $J_n$ with every latent coordinate $j=1,\dots,256$ dithered at its own frequency $\omega_j\in[1,\,1.75]$, $\Delta t=2\pi / 10\omega_{\max}$, $\alpha=4\times10^{-3}$, and $k=0.03$. Between frames a filtered velocity feed-forward carries the estimate along, $v\leftarrow v+(\Delta\bar z-v)/30$ and $z\leftarrow z+g\,v$ with $g=1$ after the first $30$ frames. Latents are clipped to
$[-4,4]$, the tracker is warm-started at $z^*_0$, and spends $N_s=100$ cost evaluations on each frame. The results are summarized in Table~\ref{tab:tracking} for tracking $5{,}000$ trajectories of $1{,}024$ frames. Errors are between the reported estimate $D_\theta(\bar z)$ and the true image
of the same frame: mean absolute error (L1) on the $[-1,1]$ pixel scale, RMSE and SSIM~\cite{wang2004ssim} on $[0,1]$ images, and the
Fr\'echet Inception distance (FID)~\cite{heusel2017fid} between the set of final tracked images and the set of true final images.  FID uses Inception-v3 pool features~\cite{szegedy2016inception} with Keras ImageNet weights on bilinearly upsampled $299\times299$ images. Some tracking results are shown in Fig. \ref{fig:CIFAR10} and many more in Fig. \ref{fig:CIFAR10_more1}.
 
\begin{table}[t]
\centering
\footnotesize
\setlength{\tabcolsep}{4pt}
\newcommand{\pmcell}[2]{\begin{tabular}[t]{@{}c@{}}#1\\[-2pt]{\scriptsize$\pm$#2}\end{tabular}}   
\caption{Final errors $(\mu \pm \sigma)$ between tracked and true images, and FID between final tracked and final true images at the sample size $N$ of the run, with the same-$N$ floor (FID between two independent samples of the true distribution). L1 on $[-1,1]$ pixel scale, RMSE/SSIM on $[0,1]$ images.}
\label{tab:tracking}
\begin{tabular}{l c c c c c c}
\toprule
Tracker & $N_s$ & $N$ & L1 final $\downarrow$ &  RMSE $\downarrow$ & SSIM $\uparrow$ & FID (floor) $\downarrow$ \\
\midrule
Frozen        & --  & 5000            & 0.568 $\pm$ 0.151 & 0.342 $\pm$ 0.082 & 0.046 $\pm$ 0.068 & 9.95 (9.95) \\ \addlinespace[2pt] 
Bounded ES    & 200 & 5000            & 0.160 $\pm$ 0.070            & 0.110 $\pm$ 0.044 & 0.556 $\pm$ 0.130 & 34.5 (9.95) \\
\bottomrule
\end{tabular}
\end{table}

\subsection{Beyond smooth energies: a regularity ladder}
\label{sec:rough}

The $C^2$ hypothesis of Theorem~\ref{thm:main} is a convenience of the proof, not a property the method needs. Algorithm~\ref{alg:es} only ever evaluates $\cos(\omega_i t+kJ)$, which is well defined for any measurable $J$; what the \emph{guarantee} needs is that the averaged drift's gradient $\nabla J$ be an admissible drift for the sampler's own noise. Here the noise is not an obstacle but the enabling mechanism: when the diffusion is uniformly non-degenerate, the law of the process is absolutely continuous (Lemma~\ref{lem:krylov}), so the trajectory spends zero time on any Lebesgue-null set, and derivatives defined only \emph{almost everywhere} are as good as continuous ones. Here the unmodified $J$ is dithered directly, no proximal oracle or convexity is assumed. Throughout this subsection we adopt the global form of the regularity hypotheses on $\f$ and $\bsig$ (Assumption~\ref{ass:glob}, Appendix~\ref{supp:nonsmooth}), under which the localization in the proof of Theorem~\ref{thm:main} is unnecessary and the argument sharpens to an explicit mean-square rate, $\Expec\sup_{t\le T}\|\x_\omega-\bar\x\|^2\le C/\omega$ (Theorem~\ref{thm:quant}), without the $C^3$ smoothness of Remark~\ref{rem:rate-supp}; the hypothesis on $J$ is then relaxed rung by rung, with all statements and proofs in Appendix~\ref{supp:nonsmooth} and the full ladder summarized in Table~\ref{tab:ladder}.

\begin{proposition}[$C^{1,1}$ energies]
\label{prop:c11}
In Assumption~\ref{ass:glob}, replace the hypothesis on $J$ by: $J(\cdot,t)\in C^{1,1}(\R^d)$ uniformly in $t$, i.e., $\nabla J$ bounded and globally Lipschitz (so $\nabla^2J$ exists a.e.\ and is bounded, but may be discontinuous; $J$ itself may be unbounded). Then, with no non-degeneracy assumption on $\bsig$, the conclusion of Theorem~\ref{thm:quant} holds unchanged, including the rate: $\Expec\sup_{t\le T}\|\x_\omega-\bar\x\|_2^2\le C/\omega$.
\end{proposition}

\begin{theorem}[Sobolev energies]
\label{thm:sobolev}
Assume the diffusion is uniformly non-degenerate, $\bsig\bsig^{\!\top}\succeq\lambda_0 I$ for some $\lambda_0>0$ (i.e., $\min_i\inf g_i^2\ge\lambda_0$), with $\f,\bsig$ bounded and Lipschitz, and that $J(\cdot,t)\in W^{2,p}_{\mathrm{uloc}}(\R^d)$ uniformly in $t$ for some $p>\max(d,2)$ (so $\nabla J$ is continuous and $\nabla^2J\in L^p_{\mathrm{uloc}}$), with $\nabla J$ and $\partial_tJ$ bounded on $\R^d\times[0,T]$. Then $\x_\omega\to\bar\x$ in probability, uniformly on $[0,T]$, with
\begin{equation*}
    \Expec\Big[\sup_{t\le T}\big\|\x_\omega(t)-\bar\x(t)\big\|_2^2\Big]\;\le\;C\,\omega^{-(p-d)/p},
\end{equation*}
where $\bar\x$ is the unique strong solution of the averaged SDE \eqref{eq:averaged-sde}, whose drift is continuous because $\nabla J$ is (Morrey's embedding, $p>d$); the a.e.-defined Hessian $\nabla^2J$ enters only the proof, along the trajectories, where the non-degenerate noise makes it unambiguous (Appendix~\ref{supp:rough-noise}).
\end{theorem}

The proof is given in Appendix~\ref{supp:rough-sobolev}, it utilizes the machinery of Theorem~\ref{thm:quant} on a mollified corrector. Theorem~\ref{thm:sobolev} covers energies that are $C^2$ off a null set as long as the gradient stays continuous: piecewise-quadratic energies glued in $C^1$ and, more generally, any energy with a merely $L^p$ Hessian. When the gradient kinks, $\|\cdot\|_1$ or ReLU-type energies, or the absolute-value errors that drive tracking energy of Eq. (\ref{eq:supp-track-J}), $J$ leaves $W^{2,p}$: $\nabla^2J$ acquires a singular surface part, and neither result above applies. The averaged limit remains well posed by Veretennikov's theorem for bounded measurable drifts of non-degenerate diffusions (Proposition~\ref{prop:lip}). Appendix~\ref{supp:rough-lip} gives the two-scale scheme for this, and Appendix~\ref{supp:rough-occ} proves that bound (Theorem~\ref{thm:occ}). Convergence holds without a rate whenever $J$ is Lipschitz in time as well as in space; no rate is claimed, and none is expected from the argument, since $\|\nabla J-\nabla J_\varepsilon\|$ carries no rate for merely Lipschitz $J$.

\begin{table}[t]
\caption{Provability degrades gracefully as $J$ roughens. Statements and proofs in Appendix~\ref{supp:nonsmooth}.}
\label{tab:ladder}
\begin{center}
\small
\setlength{\tabcolsep}{4.5pt}
\begin{tabular}{@{}llll@{}}
\toprule
\textbf{Regularity of $J(\cdot,t)$} & \textbf{Permits} & \textbf{Drift $\nabla J$} & \textbf{Status} \\
\midrule
$C^2$ (Assumption~\ref{ass:glob}) & classical smoothness & continuous & Thm~\ref{thm:main}; rate $\omega^{-1/2}$ (Thm~\ref{thm:quant}) \\
$C^{1,1}=W^{2,\infty}$ & \begin{tabular}[t]{@{}l@{}}bounded, non-$C^1$ Hessian\end{tabular} & Lipschitz & \begin{tabular}[t]{@{}l@{}}Prop.~\ref{prop:c11} (proved) \end{tabular} \\
\begin{tabular}[t]{@{}l@{}}$W^{2,p}_{\mathrm{uloc}}$, $p>d$, $\nabla J$ bounded\end{tabular} & unbounded $L^p$ Hessian & continuous & \begin{tabular}[t]{@{}l@{}}Thm~\ref{thm:sobolev} (proved) \end{tabular} \\
Lipschitz, $W^{1,\infty}$ & gradient kinks & a.e., $\in L^\infty$ & \begin{tabular}[t]{@{}l@{}}Prop.~\ref{prop:lip}; convergence, no rate \end{tabular} \\
$\nabla J\in L^q_tL^p_\x$, $\frac dp{+}\frac2q<1$ & \begin{tabular}[t]{@{}l@{}}unbounded gradient;\\ non-Lipschitz cusps\end{tabular} & $\in L^p$ & \begin{tabular}[t]{@{}l@{}}limit well posed (Krylov--\\R\"ockner); averaging expected\end{tabular} \\
$\nabla J$ distributional, $C^{-\beta}$ & jumps in $J$ & surface measure & \begin{tabular}[t]{@{}l@{}}open frontier\end{tabular} \\
\bottomrule
\end{tabular}
\end{center}
\end{table}

\textbf{How rough can $J$ be?} By Krylov-R\"ockner~\citep{krylov2005strong}, the averaged SDE \eqref{eq:averaged-sde} is strongly well posed when the weak gradient satisfies the Ladyzhenskaya-Prodi-Serrin-type condition
\begin{equation}
    \nabla J\ \in\ L^q\big([0,T];\,L^p(\R^d)\big),
    \qquad
    d/p+2/q\;<\;1,
    \label{eq:LPS}
\end{equation}
which permits $\nabla J$ to be unbounded, defined only almost everywhere, and merely $p$-integrable in space with $q$-integrability in time; this is the most general regularity under which the limit is provably well posed by classical means, and we expect the averaging to extend throughout it (Remark~\ref{rem:boundary}). Requiring $J(\cdot,t)$ to be merely \emph{locally Lipschitz}, the weakest classical condition under which $\nabla J$ exists a.e.\ at all, by Rademacher's theorem~\citep{evans1992measure}, still forces $\nabla J$ to be locally bounded, hence locally in $L^p$ for every $p\le\infty$, so in local regularity \eqref{eq:LPS} is strictly more general than any Lipschitz-type hypothesis: it additionally admits gradients with genuine local blow-ups of merely $L^p$ type (cusps $|\nabla J|\asymp|\x-\x_0|^{-\gamma}$, $0<\gamma<d/p$), which no Lipschitz condition tolerates.

\clearpage

\subsection*{AI use statement}
In this work, generative AI tools were used to proof-read the paper, to find typos or mathematical errors, for polishing the writing, and for discussion, suggestions, and collaboration on some of the mathematical results. We have reviewed all AI-assisted work. We checked LLM-suggested references through a manual literature survey. We take responsibility for the final content of this work, including text, claims or artifacts produced with the aid of generative AI.

\subsection*{Reproducibility statement}
For all of the theoretical results, clear explanations of any assumptions and complete proofs of the claims are included in the appendix. For the proposed algorithm a clear explanation of exactly how it is applied has been provided. A complete description of data processing steps has also been provided.


\bibliography{iclr2027_conference_v5}
\bibliographystyle{iclr2027_conference}

\clearpage

\appendix

The appendix includes all detailed proofs and a large amount of numerical experiments testing various aspects of this approach such as dithering frequency spread versus state dimension, as well as additional details about the experiments which could not fit in the main text.

\section{Proof of Theorem~\ref{thm:main} and related results}
\label{supp:proof}

\subsection{Notation and standing assumptions}
\label{supp:notation}

For $\x\in\R^{N}$ (we write $N$ for the state dimension; $N=d$ in the notation of the main text) consider the coupled nonlinear system of It\^{o} stochastic differential equations
\begin{equation}
    d\x \;=\; \bmu_\omega(\x,t)\,dt \;+\; \bsig(\x,t)\,d\w,
    \qquad
    \bsig(\x,t)\;=\;\operatorname{diag}\big(g_1(\x,t),\dots,g_N(\x,t)\big),
    \label{eq:supp-sde}
\end{equation}
on a filtered probability space $(\Omega,\mathcal F,(\mathcal F_t)_{t\ge0},\Prob)$ carrying a standard $N$-dimensional Brownian motion $\w$, with drift components
\begin{equation}
    \mu_{\omega,n}(\x,t)
    \;=\;
    f_n(\x,t)\;+\;g_n(\x,t)\,\sqrt{\alpha\omega_n}\,
    \cos\big(\omega_n t + k\,J(\x,t)\big),
    \qquad
    \omega_n=\omega r_n ,
    \label{eq:supp-drift}
\end{equation}
and the associated averaged system
\begin{equation}
    d\bar\x \;=\; \bar\bmu(\bar\x,t)\,dt + \bsig(\bar\x,t)\,d\w,
    \qquad
    \bar\bmu \;=\; \f-\tfrac{k\alpha}{2}\,\bsig\bsig^{\!\top}\nabla_{\x}J ,
    \label{eq:supp-averaged}
\end{equation}
driven by the \emph{same} $\w$ and started from the same $\mathcal F_0$-measurable initial condition $\x_\omega(0)=\bar\x(0)=\x_0$ with $\Expec\|\x_0\|^2<\infty$, under Assumptions~\ref{ass:reg}--\ref{ass:growth} of the main text. Under Assumption~\ref{ass:growth} both \eqref{eq:supp-sde} (for each fixed $\omega$) and \eqref{eq:supp-averaged} have unique global strong solutions with $\Expec\sup_{t\le T}\|\cdot\|^2<\infty$. Because the dither term is bounded by $\sqrt{\alpha\omega r_{\max}}\,|g_n|$, $\bmu_\omega$ satisfies a linear-growth bound with an $\omega$-dependent constant, which guarantees global existence at each fixed $\omega$ \citep[Ch.~5]{karatzas1991brownian}.

Throughout the analysis that follows we fix $T>0$, $C$ will be used to denote a constant that may change from line to line, and $C_R$ denotes a constant depending only on $(R,T,N,\alpha,k,r_1,\dots,r_N)$ and on the bounds from Assumption~\ref{ass:reg} of $\f$, $g_n$, $J$ and their derivatives on $\bar B_R\times[0,T]$, where $\bar B_R=\{\x:\|\x\|\le R\}$.

\subsection{Warm-up: the deterministic one-dimensional mechanism}
\label{supp:warmup}

The following is meant to give intuition about how this algorithm is able to
extract gradients, without yet presenting any stochastic technicalities. We
consider $N=1$, $g\equiv1$, $\f\equiv0$, and the dynamics
$\dot x=\sqrt{\alpha\omega}\cos\theta$ with phase
$\theta(t)=\omega t+kJ(x(t),t)$, where $J\in C^{2}$ has bounded derivatives.
Consider the near-identity variable
\begin{equation}
    y = x - \sqrt{\alpha/\omega} \sin\theta ,
    \label{eq:supp-warmup-ydef}
\end{equation}
which differs from $x$ by at most $\sqrt{\alpha/\omega}$. Integrating
$\dot x=\sqrt{\alpha\omega}\cos\theta$ at leading order ($\theta\approx\omega t$)
gives $x\approx\mathrm{const}+\sqrt{\alpha/\omega}\sin\theta$, so the
subtraction in \eqref{eq:supp-warmup-ydef} removes the fast
oscillation. Differentiating \eqref{eq:supp-warmup-ydef} and plugging in
$\dot\theta=\omega+k(\partial_xJ\,\dot x+\partial_tJ)$, we get
\begin{align}
    \dot y
    &= \dot x-\sqrt{\alpha/\omega}\,\cos\theta\,\dot\theta
    \nonumber\\
    &= \dot x-\sqrt{\alpha/\omega}\cos\theta\,\big[\omega+k(\partial_xJ\,\dot x+\partial_tJ)\big]
    \nonumber\\
    &= \underbrace{\dot x-\sqrt{\alpha\omega}\cos\theta}_{=\,0}
      \;-\;\sqrt{\alpha/\omega}\,k\,\partial_xJ\,\big(\sqrt{\alpha\omega}\cos\theta\big)\cos\theta
      \;-\;\sqrt{\alpha/\omega}\,k\,\partial_tJ\cos\theta
    \nonumber\\
    &= -\underbrace{\sqrt{\alpha/\omega}\,\sqrt{\alpha\omega}}_{=\,\alpha}\,k\,\partial_xJ\cos^2\theta
      \;-\;\sqrt{\alpha/\omega}\,k\,\partial_tJ\cos\theta
    \nonumber\\
    &= -\tfrac{k\alpha}{2}\,\partial_xJ\,\big(1+\cos2\theta\big)
      \;-\;\sqrt{\alpha/\omega}\,k\,\partial_tJ\cos\theta
    \nonumber\\
    &= -\tfrac{k\alpha}{2}\,\partial_xJ(x,t)
       \;-\;\tfrac{k\alpha}{2}\,\partial_xJ(x,t)\cos2\theta
       \;-\;\sqrt{\alpha/\omega}\,k\,\partial_tJ(x,t)\cos\theta
    \nonumber\\
    &= -\tfrac{k\alpha}{2}\,\partial_xJ(x,t)
       \;-\;\tfrac{k\alpha}{2}\,\partial_xJ(x,t)\cos2\theta
       \;+\;O(\omega^{-1/2}).
    \label{eq:supp-warmup-y}
\end{align}
The final line follows because, for fixed $k$ and $\alpha$,
\begin{equation}
    \big|\sqrt{\alpha/\omega}\,k\,\partial_tJ(x,t)\cos\theta\big|
    \;\le\; k\sqrt{\alpha}\,\sup_{x,t}\lvert\partial_tJ\rvert\;\omega^{-1/2}
    \;=\;O(\omega^{-1/2})
    \qquad\text{as }\omega\to\infty.
    \label{eq:supp-warmup-bound}
\end{equation}
The first term of \eqref{eq:supp-warmup-y} is the desired gradient flow. The $\cos2\theta$ term oscillates at frequency $2\omega$ with slowly varying coefficient; one further integration by parts in time shows its running integral is $O(\omega^{-1/2})$ on $[0,T]$ when $\partial_xJ$ is Lipschitz along the trajectory (the boundary term is $O(\omega^{-1})$, but the coefficient $\partial_xJ(x(t),t)$ varies at the dither speed $\dot x=O(\sqrt\omega)$, which costs $\sqrt\omega/\omega$). Hence $y$, and therefore $x=y+O(\omega^{-1/2})$, converges uniformly on $[0,T]$ to the solution of
$\dot{\bar x}=-\tfrac{k\alpha}{2}\partial_xJ(\bar x,t)$ as
$\omega\to\infty$. The mechanism is as described in the main text: the
state carries an $O(\omega^{-1/2})$ ripple phase-locked to the dither;
re-inserting the ripple into the cosine produces a product whose time
average is exactly $-\tfrac{k\alpha}{2}\,\partial_x J$. The remainder of
this appendix carries this argument out for the coupled It\^{o} system
\eqref{eq:supp-sde}, where three complications arise: the ripple
interacts with the Brownian forcing, the coordinates are coupled through
the shared phase $kJ$, and the oscillatory remainders must be controlled
\emph{along the $\omega$-dependent random trajectory}.

\subsection{An oscillatory-integral lemma along ES trajectories}
\label{supp:lemmas}

We start with a helper Lemma that will be used in the proof of Theorem 1.

For $R>0$ let
\begin{equation}
\begin{aligned}
    \tilde\tau_R &\;=\; \inf\{t\ge0:\ \|\x_\omega(t)\|\ge R \text{ or } \|\bar\x(t)\|\ge R\}\in[0,\infty],\\
    \tau_R &\;=\; \tilde\tau_R\wedge T, \quad \mathrm{where} \quad a \wedge b:= \min \{a,b\}.
\end{aligned}
    \label{eq:supp-tau}
\end{equation}
Since the paths are continuous, $\|\x_\omega(t)\|\le R$ and
$\|\bar\x(t)\|\le R$ for all $t\le\tau_R$, and $\|\x_\omega(\tilde\tau_R)\|=R$ or $\|\bar\x(\tilde\tau_R)\|=R$ whenever $\tilde\tau_R<\infty$.

\begin{lemma}[Oscillatory integrals along the dithered trajectory]
\label{lem:osc-main-ref}
Let Assumptions~\ref{ass:reg}--\ref{ass:growth} hold, fix $R,T>0$, and
let $\Phi:\R^N\times[0,T]\to\R$ be bounded by $C_\Phi$ and jointly
Lipschitz with constant $L_\Phi$ on $\bar B_R\times[0,T]$. Then for every
real $c\neq0$ and all $\omega\ge1$,
\begin{equation}
    \Expec\left[\sup_{t\le T}\left|\int_0^{t\wedge\tau_R}
    \Phi\big(\x_\omega(s),s\big)\,e^{\mathrm i c\omega s}\,ds\right|^2\right]
    \;\le\; C(R,T,N,\Phi,|c|,\alpha,r_{\max})\ \omega^{-1/2} .
    \label{eq:supp-osc}
\end{equation}
The same bound holds with $e^{\mathrm i c\omega s}$ replaced by
$\cos(c\omega s)$ or $\sin(c\omega s)$.
\end{lemma}

\begin{proof}
Write $\x=\x_\omega$ and abbreviate $t\wedge\tau_R$ by $t^\tau$. Partition
$[0,T]$ into $M=\lceil T/\Delta\rceil$ blocks $[t_j,t_{j+1}]$ of length
$\le\Delta\in(0,1]$, $t_j=j\Delta$, and write
\begin{align}
    \int_0^{t^\tau}\Phi(\x_s,s)e^{\mathrm i c\omega s}ds
    \;=\;&
    \underbrace{\sum_j \Phi(\x_{t_j^\tau},t_j)\int_{t_j\wedge t^\tau}^{t_{j+1}\wedge t^\tau}
    e^{\mathrm i c\omega s}ds}_{A(t)}
    \nonumber\\
    &+\;
    \underbrace{\sum_j\int_{t_j\wedge t^\tau}^{t_{j+1}\wedge t^\tau}
    \big[\Phi(\x_s,s)-\Phi(\x_{t_j^\tau},t_j)\big]e^{\mathrm i c\omega s}ds}_{B(t)} .
\end{align}
In $A(t)$, each trigonometric integral over a (possibly
truncated) block is at most $2/(|c|\omega)$ in modulus, and there are at
most $M\le T/\Delta+1$ blocks, so path-wise
$\sup_{t\le T}|A(t)|\le 2C_\Phi(T/\Delta+1)/(|c|\omega)$.

For $B(t)$, for $s$ in block $j$ with $s\le\tau_R$,
\begin{equation}
    \big|\Phi(\x_s,s)-\Phi(\x_{t_j},t_j)\big|
    \;\le\; L_\Phi\big(\|\x_s-\x_{t_j}\|+\Delta\big),
\end{equation}
and from \eqref{eq:supp-sde},
$\x_s-\x_{t_j}
 =\int_{t_j}^{s}\bmu_\omega\,du+\mathbf N_j(s)$,
$\mathbf N_j(s):=\int_{t_j}^{s}\bsig\,d\w$, so that on
$\{s\le\tau_R\}$,
\begin{equation}
    \|\x_s-\x_{t_j}\|
    \;\le\;
    C_R\big(1+\sqrt{\alpha\omega r_{\max}}\big)\Delta
    \;+\;\|\mathbf N_j(s)\|,
    \qquad
    \Expec\|\mathbf N_j(s\wedge\tau_R)\|^2\le C_R^2 N\Delta .
\end{equation}
Because the integrand of $B$ is dominated by a nonnegative process,
$\sup_{t\le T}|B(t)|$ is at most
$L_\Phi\int_0^{T\wedge\tau_R}(\|\x_s-\x_{t_{j(s)}}\|+\Delta)\,ds$,
and by Cauchy--Schwarz in $s$,
\begin{equation}
    \Expec\sup_{t\le T}|B|^2
    \;\le\;
    3L_\Phi^2T^2\Big[C_R^2\big(1+\alpha\omega r_{\max}\big)\Delta^2
    +\Delta^2\Big]
    \;+\;3L_\Phi^2\,T\int_0^{T}\Expec\|\mathbf N_{j(s)}(s\wedge\tau_R)\|^2ds ,
\end{equation}
the last term being at most $3L_\Phi^2T^2C_R^2N\Delta$. Combining,
\begin{equation}
    \Expec\sup_{t\le T}\Big|\int_0^{t^\tau}\Phi e^{\mathrm ic\omega s}ds\Big|^2
    \;\le\;
    C\Big[\frac{1}{\Delta^2\omega^2}
    +\omega\Delta^2+\Delta^2+\Delta\Big].
\end{equation}
Choosing $\Delta=\omega^{-3/4}$ makes the first two brackets
$O(\omega^{-1/2})$ and the last two $O(\omega^{-3/4})$, which proves
\eqref{eq:supp-osc}; real and imaginary parts give the trigonometric
versions.
\end{proof}

\begin{remark}[Sharper rate under one more derivative]
\label{rem:rate-supp}
If $g_n,J\in C^3$ in $\x$ (so that $\Phi$ below is $C^1$ with Lipschitz
derivative on $\bar B_R$), the partition argument can be replaced by a
genuine It\^{o} integration by parts against
$e^{\mathrm ic\omega s}/(\mathrm ic\omega)$: the boundary terms are
$O(\omega^{-1})$, the drift part of $d\Phi(\x_s,s)$ is
$O(\sqrt\omega)$ on the localized set, contributing
$O(\omega^{-1}\cdot\sqrt\omega)=O(\omega^{-1/2})$, and the martingale
part contributes $O_{L^2}(\omega^{-1})$ by Doob's inequality. This
sharpens \eqref{eq:supp-osc} to $C\omega^{-1}$, i.e.,
$O_{L^2}(\omega^{-1/2})$ for the integral itself, and correspondingly
sharpens the overall rate in Theorem~\ref{thm:main} from
$O_{\Prob}(\omega^{-1/4})$ to $O_{\Prob}(\omega^{-1/2})$.
\end{remark}

\subsection{Proof of Theorem~\ref{thm:main}}
\label{supp:proof-main}

\paragraph{Step 0: localization.}
Fix $R\ge2$ and let $\tau_R$ be as in \eqref{eq:supp-tau}. All estimates
in Steps~1--3 are on $[0,\tau_R]$, where both processes lie in
$\bar B_R$ and all bounds $C_R$ of Assumption~\ref{ass:reg} apply.

\paragraph{Step 1: It\^{o} integration by parts on the dither term.}
Fix a coordinate $n$ and write the mild form of \eqref{eq:supp-sde}:
\begin{equation}
    x_{\omega,n}(t)=x_{0,n}
    +\int_0^t f_n\,ds
    +I_{2,n}(t)
    +\int_0^t g_n\,dw_n,
    \qquad
    I_{2,n}(t)=\int_0^t g_n\sqrt{\alpha\omega_n}\,
    \cos\big(\omega_ns+kJ\big)\,ds,
\end{equation}
all integrands evaluated along $(\x_\omega(s),s)$. Expand the cosine and
define the $C^{2,1}$ compositions
\begin{equation}
    u^c \;=\; g_n\cos(kJ),
    \qquad
    u^s \;=\; g_n\sin(kJ),
\end{equation}
so that $I_{2,n}=\int u^c\,dv_c-\int u^s\,dv_s$ with the deterministic,
finite-variation dither antiderivatives
\begin{equation}
    v_c(s)=\sqrt{\alpha/\omega_n}\,\sin(\omega_ns),
    \qquad
    v_s(s)=-\sqrt{\alpha/\omega_n}\,\cos(\omega_ns),
    \qquad |v_c|,|v_s|\le\sqrt{\alpha/\omega_n}.
\end{equation}
Since $v_{c},v_{s}$ are $C^1$ and deterministic, the It\^{o} product
rule has vanishing covariation, and
\begin{equation}
    I_{2,n}(t)
    \;=\;
    \big[u^cv_c-u^sv_s\big]_0^t
    \;-\;\int_0^t v_c\,du^c
    \;+\;\int_0^t v_s\,du^s .
    \label{eq:supp-ibp}
\end{equation}
By It\^{o}'s formula applied to $u\in\{u^c,u^s\}$ along
\eqref{eq:supp-sde},
\begin{equation}
    du \;=\;
    \Big[\partial_t u+(\nabla u)^{\!\top}\f
    +\tfrac12\operatorname{tr}\big(\bsig\bsig^{\!\top}\nabla^2u\big)\Big]ds
    \;+\;(\nabla u)^{\!\top}\bsig\,d\w
    \;+\;(\nabla u)^{\!\top}\mathbf D\,ds,
\end{equation}
where $\mathbf D$ is the dither drift,
$D_i=g_i\sqrt{\alpha\omega_i}\cos(\omega_is+kJ)$. Insert this into
\eqref{eq:supp-ibp} and collect all terms whose prefactor carries
$\sqrt{\alpha/\omega_n}$ \emph{without} a compensating
$\sqrt{\omega_i}$: the boundary terms (bounded by
$4C_R\sqrt{\alpha/\omega_n}$, noting $v_s(0)\neq0$), the
$\partial_tu$, $(\nabla u)^{\!\top}\f$, and Hessian terms (each bounded by
$C_R\sqrt{\alpha/\omega_n}\,T$), and the stochastic integrals
$-\int v_c(\nabla u^c)^{\!\top}\bsig\,d\w
 +\int v_s(\nabla u^s)^{\!\top}\bsig\,d\w$, whose supremum has second
moment at most $16(\alpha/\omega_n)C_R^2T$ by Doob's maximal inequality
and the It\^{o} isometry (a factor $4$ from Doob and a factor $2$ from
the sum of the two martingales). Denoting their sum by $E^{(1)}_n(t)$,
\begin{equation}
    \Expec\Big[\sup_{t\le T}\big|E^{(1)}_n(t\wedge\tau_R)\big|^2\Big]
    \;\le\; C_R\,\omega^{-1} .
    \label{eq:supp-E1}
\end{equation}

\paragraph{Step 2: resonant and non-resonant dither products.}
The remaining terms of \eqref{eq:supp-ibp} are those where $du$
contributes the dither $\mathbf D$:
\begin{align}
    -\sum_{i=1}^N \alpha\sqrt{\tfrac{r_i}{r_n}}
    \int_0^t \sin(\omega_ns)\,\partial_iu^c\,g_i
    \big[\cos(\omega_is)\cos kJ-\sin(\omega_is)\sin kJ\big]ds&
    \nonumber\\
    \;-\;\sum_{i=1}^N \alpha\sqrt{\tfrac{r_i}{r_n}}
    \int_0^t \cos(\omega_ns)\,\partial_iu^s\,g_i
    \big[\cos(\omega_is)\cos kJ-\sin(\omega_is)\sin kJ\big]ds&,
    \label{eq:supp-cross}
\end{align}
where the second group, which carries the same overall sign as the first, arises from the term $+\int v_s(\nabla u^s)^{\!\top}\mathbf D\,ds$ with $v_s=-\sqrt{\alpha/\omega_n}\cos(\omega_ns)$. Every product of two
trigonometric factors in \eqref{eq:supp-cross} reduces by the product
formulas to $\tfrac12[\text{trig}((r_n{-}r_i)\omega s)\pm\text{trig}((r_n{+}r_i)\omega s)]$;
for $i\neq n$ all resulting frequencies $(r_n\pm r_i)\omega$ are nonzero
by Assumption~\ref{ass:freq}, and for $i=n$ the frequencies are $0$ and
$2r_n\omega$. Each nonzero-frequency term is an oscillatory integral of
the form covered by Lemma~\ref{lem:osc-main-ref} with
$\Phi\in\{\partial_iu^{c/s}\,g_i\cos kJ,\ \partial_iu^{c/s}\,g_i\sin kJ\}$,
which are bounded and jointly Lipschitz on $\bar B_R\times[0,T]$ because
$g_n,J\in C^{2,1}$ with continuous mixed derivatives
$\partial_t\nabla_\x g_n,\partial_t\nabla_\x J$ (Assumption~\ref{ass:reg}),
so that $\partial_iu^{c/s}\in C^{1}$ jointly in $(\x,t)$, hence Lipschitz
on the compact set. Their total, $E^{(2)}_n(t)$, therefore satisfies
\begin{equation}
    \Expec\Big[\sup_{t\le T}\big|E^{(2)}_n(t\wedge\tau_R)\big|^2\Big]
    \;\le\; C_R\,\omega^{-1/2} .
    \label{eq:supp-E2}
\end{equation}
The zero-frequency (resonant) halves survive averaging. From the first
group of \eqref{eq:supp-cross} with $i=n$, using
$\sin(\omega_ns)\sin(\omega_ns)=\tfrac12-\tfrac12\cos(2\omega_ns)$, the
surviving term is
$+\tfrac{\alpha}{2}\int_0^t g_n\,\partial_nu^c\,\sin(kJ)\,ds$; from the
second group, using $\cos^2(\omega_ns)=\tfrac12+\tfrac12\cos(2\omega_ns)$,
it is $-\tfrac{\alpha}{2}\int_0^t g_n\,\partial_nu^s\,\cos(kJ)\,ds$.
Substituting
$\partial_nu^c=\partial_ng_n\cos kJ-kg_n\,\partial_nJ\,\sin kJ$ and
$\partial_nu^s=\partial_ng_n\sin kJ+kg_n\,\partial_nJ\,\cos kJ$,
\begin{equation}
    g_n\big[\partial_nu^c\sin kJ-\partial_nu^s\cos kJ\big]
    \;=\;
    -k\,g_n^2\,\partial_nJ ,
\end{equation}
the $\partial_ng_n$ terms cancelling identically. The resonant
contribution to $I_{2,n}$ is therefore
\begin{equation}
    -\frac{k\alpha}{2}\int_0^t g_n^2\big(\x_\omega(s),s\big)\,
    \partial_nJ\big(\x_\omega(s),s\big)\,ds .
    \label{eq:supp-resonant}
\end{equation}
Collecting \eqref{eq:supp-E1}--\eqref{eq:supp-resonant} and writing
$\mathbf E=(E^{(1)}_n+E^{(2)}_n)_{n=1}^N$, we obtain on $[0,\tau_R]$ the
exact representation
\begin{align}
    \x_\omega(t)
    \;=\;
    \x_0+\int_0^t\bar\bmu\big(\x_\omega(s),s\big)ds
    +\int_0^t\bsig\big(\x_\omega(s),s\big)d\w
    +\mathbf E(t),
    \nonumber\\
    \Expec\sup_{t\le T}\big\|\mathbf E(t\wedge\tau_R)\big\|^2
    \;\le\; C_R\,\omega^{-1/2} .\hspace{4.2em}
    \label{eq:supp-mild}
\end{align}
Thus, up to the explicitly controlled remainder $\mathbf E$, the dithered
process satisfies the \emph{averaged} SDE along its own trajectory.

\paragraph{Step 3: Gr\"onwall comparison under the shared Brownian motion.}
Let $\z=\x_\omega-\bar\x$. Subtracting the mild form of
\eqref{eq:supp-averaged} from \eqref{eq:supp-mild},
\begin{equation}
    \z(t)=\int_0^t\big[\bar\bmu(\x_\omega)-\bar\bmu(\bar\x)\big]ds
    +\int_0^t\big[\bsig(\x_\omega)-\bsig(\bar\x)\big]d\w+\mathbf E(t).
\end{equation}
On $\bar B_R\times[0,T]$, $\bar\bmu$ and $\bsig$ are Lipschitz with a
constant $L_R$: for $\bsig$ this is Assumption~\ref{ass:reg} directly,
and for $\bar\bmu=\f-\tfrac{k\alpha}{2}\bsig\bsig^{\!\top}\nabla J$ it
holds because products of $C^1$ functions of $(g_n,\partial_nJ)$ are
Lipschitz on compacts. Set
$\varphi(t)=\Expec\big[\sup_{s\le t}\|\z(s\wedge\tau_R)\|^2\big]$. By
Cauchy--Schwarz on the drift difference and Doob plus the It\^{o}
isometry on the (diagonal) diffusion difference,
\begin{equation}
    \varphi(t)\;\le\;3(T+4)L_R^2\int_0^t\varphi(s)\,ds
    \;+\;3\,\Expec\sup_{s\le T}\|\mathbf E(s\wedge\tau_R)\|^2 ,
\end{equation}
and Gr\"onwall's inequality with \eqref{eq:supp-mild} gives
\begin{equation}
    \Expec\Big[\sup_{t\le T}\|\z(t\wedge\tau_R)\|^2\Big]
    \;\le\;
    C'_R\,\omega^{-1/2},
    \qquad
    C'_R=3\,C_R\,e^{3(T+4)L_R^2T}.
    \label{eq:supp-gronwall}
\end{equation}

\paragraph{Step 4: removing the localization.}
Let $\delta\in(0,1]$ and $\varepsilon\in(0,1)$ be given. By the standard
moment bound for \eqref{eq:supp-averaged} under
Assumption~\ref{ass:growth}, $\Expec\sup_{t\le T}\|\bar\x\|^2<\infty$,
so we may fix $R$ with
$\Prob\big(\sup_{t\le T}\|\bar\x(t)\|\ge R-2\big)\le\varepsilon/2$. By
\eqref{eq:supp-gronwall} and Chebyshev's inequality, choose
$\omega^*$ so that for all $\omega\ge\omega^*$,
$\Prob\big(\sup_{t\le T}\|\z(t\wedge\tau_R)\|>\delta\big)
 \le C'_R\,\omega^{-1/2}\delta^{-2}\le\varepsilon/2$.
On the event
$A=\{\sup_{t\le T}\|\bar\x\|\le R-2\}\cap
   \{\sup_{t\le T}\|\z(t\wedge\tau_R)\|\le\delta\}$,
suppose that the uncapped exit time satisfies $\tilde\tau_R\le T$; then
$\tau_R=\tilde\tau_R$ and by continuity one of the processes has norm
exactly $R$ at $\tau_R$, yet
$\|\bar\x(\tau_R)\|\le R-2$ and
$\|\x_\omega(\tau_R)\|\le\|\bar\x(\tau_R)\|+\delta\le R-1<R$, a
contradiction. Hence $\tilde\tau_R>T$ on $A$, so $\tau_R=T$ there, the
stopped and unstopped suprema coincide, and
\begin{equation}
    \Prob\Big(\sup_{t\le T}\|\x_\omega(t)-\bar\x(t)\|>\delta\Big)
    \;\le\;\Prob(A^c)\;\le\;\varepsilon
    \qquad\text{for all }\omega\ge\omega^* .
\end{equation}
This proves Theorem~\ref{thm:main}; \eqref{eq:supp-gronwall} also yields
the localized rate $O_{\Prob}(\omega^{-1/4})$, sharpened to $O_{\Prob}(\omega^{-1/2})$ under the extra smoothness of Remark~\ref{rem:rate-supp}. \qed

\subsection{Corollaries}
\label{supp:corollaries}

\begin{corollary}[Convergence in distribution on path space]
\label{cor:weak}
Under the assumptions of Theorem~\ref{thm:main},
$\mathrm{Law}(\x_\omega)\Rightarrow\mathrm{Law}(\bar\x)$ on
$C([0,T];\R^N)$ as $\omega\to\infty$.
\end{corollary}

\begin{proof}
Theorem~\ref{thm:main} states convergence in probability of
$\x_\omega\to\bar\x$ in the Polish space $C([0,T];\R^N)$ with the
uniform norm; convergence in probability implies convergence in
distribution \citep{billingsley1999convergence}.
\end{proof}

\begin{corollary}[Sampling in the iterated limit]
\label{cor:sampling}
Take instantiation (I$_\lambda$): $\f\equiv\mathbf 0$,
$\bsig=\sqrt2\,I$, $k\alpha=\lambda>0$, $J=\E\in C^2$ independent of
$t$ with $\nabla\E$ of linear growth (Assumption~\ref{ass:growth}; e.g.,
$\nabla\E$ globally Lipschitz), and assume additionally the dissipativity
condition $\langle\nabla\E(\x),\x\rangle\ge c_1\|\x\|^2-c_2$ for constants
$c_1>0,c_2\ge0$ and all $\|\x\|$ large. (Superlinear potentials such as
quartic double wells violate the linear-growth condition and would require
a Lyapunov-based localization of Theorem~\ref{thm:main}.) Then the averaged process is the
Langevin diffusion with unique stationary law
$\pi_\lambda\propto e^{-\lambda\E}$, it is exponentially ergodic
\citep{roberts1996langevin}, and for every bounded Lipschitz
$h:\R^N\to\R$,
\begin{equation}
    \lim_{T\to\infty}\ \lim_{\omega\to\infty}\
    \Expec\big[h\big(\x_\omega(T)\big)\big]
    \;=\;\int h\,d\pi_\lambda .
\end{equation}
The order of the limits cannot be interchanged on the basis of
Theorem~\ref{thm:main}, whose threshold $\omega^*$ depends on $T$.
\end{corollary}

\begin{proof}
For fixed $T$, split
$|\Expec h(\x_\omega(T))-\pi_\lambda(h)|
 \le|\Expec h(\x_\omega(T))-\Expec h(\bar\x(T))|
 +|\Expec h(\bar\x(T))-\pi_\lambda(h)|$.
For any $\delta>0$ the first term is at most
$\mathrm{Lip}(h)\,\delta+2\|h\|_\infty
 \Prob(\|\x_\omega(T)-\bar\x(T)\|>\delta)\to\mathrm{Lip}(h)\,\delta$ as
$\omega\to\infty$ by Theorem~\ref{thm:main}, and $\delta$ is arbitrary.
The second term decays as $C_he^{-\rho T}$ by exponential ergodicity of
the Langevin diffusion under dissipativity
\citep{roberts1996langevin}. Let $\omega\to\infty$, then $T\to\infty$.
\end{proof}

Under the additional strong log-concavity hypothesis $\nabla^2\E\succeq mI$, Appendix~\ref{supp:uniform} upgrades this iterated limit to $\lim_{\omega\to\infty}\limsup_{T\to\infty}\mathcal W_2(\mathrm{Law}(\x_\omega(T)),\pi_\lambda)=0$, i.e., the frequency limit may be taken after (or jointly with) the time limit---conditionally on the start-uniform window estimate of Assumption~\ref{ass:window} (unconditionally at the level of the transfer principle, Proposition~\ref{prop:uniform}).

\begin{corollary}[$\lambda$-tempering]
\label{cor:tempering}
In the setting of Corollary~\ref{cor:sampling}, the stationary density of
the averaged dynamics $d\bar\x=-\lambda\nabla\E\,dt+\sqrt2\,d\w$ is
$\rho_\infty\propto e^{-\lambda\E}\propto p^\lambda$.
\end{corollary}

\begin{proof}
The stationary Fokker--Planck equation is
$0=\nabla\!\cdot\!\big(\lambda\nabla\E\,\rho+\nabla\rho\big)$, solved by
any $\rho$ with $\nabla\rho=-\lambda\rho\nabla\E$, i.e.,
$\rho\propto e^{-\lambda\E}$; uniqueness under dissipativity is
classical \citep{roberts1996langevin}.
\end{proof}

\subsection{Uniform-in-time accuracy under strong log-concavity}
\label{supp:uniform}

Throughout this subsection $\f,\bsig,J$ are independent of $t$, and Assumption~\ref{ass:growth} is in force, so that for each fixed $\omega$ both \eqref{eq:supp-sde} and \eqref{eq:supp-averaged} have unique global strong solutions with finite second moments at every time. We write $\mathcal P_2(\R^N)$ for the probability measures with finite second moment, $\mathcal W_2$ for the Wasserstein-2 distance, under which $\mathcal P_2(\R^N)$ is a complete metric space \citep{villani2009optimal}, and $(\bar P_t)_{t\ge0}$ for the Markov semigroup of the averaged SDE acting on laws, $\mu\bar P_t:=\mathrm{Law}(\bar\x(t))$ for $\bar\x(0)\sim\mu$; under Assumption~\ref{ass:growth} with locally Lipschitz coefficients, $\bar P_t$ maps $\mathcal P_2$ into itself \citep[Ch.~5]{karatzas1991brownian}.

\begin{lemma}[Start invariance of the window estimate]\label{lem:shift}
Let Assumption~\ref{ass:glob} hold with $\f,\bsig,J$ independent of $t$, fix $T_0>0$ and $t_0\ge0$, and let $\x^*$ be $\mathcal F_{t_0}$-measurable with $\Expec\|\x^*\|^2<\infty$. Let $\x_\omega$ solve the ES SDE \eqref{eq:supp-sde} and $\bar\x$ the averaged SDE \eqref{eq:supp-averaged} on $[t_0,t_0+T_0]$, driven by the same Brownian motion, from the common initial condition $\x_\omega(t_0)=\bar\x(t_0)=\x^*$. Then
\begin{equation*}
    \Expec\Big[\sup_{t_0\le t\le t_0+T_0}\big\|\x_\omega(t)-\bar\x(t)\big\|_2^2\Big]\;\le\;\frac{C(T_0)}{\omega},
\end{equation*}
with $C(T_0)=C_1e^{C_2T_0}$ the constant of Theorem~\ref{thm:quant}, independent of $t_0$ and of the law of $\x^*$.
\end{lemma}

\begin{proof}
In the shifted time $u=t-t_0$ the pair solves the same two equations, driven by the Brownian motion $\tilde\w_u:=\w_{t_0+u}-\w_{t_0}$ from the (time-$0$-measurable, for the shifted filtration) start $\x^*$, except that the dither phases carry the constant offsets $\theta_n:=\omega_nt_0$: the dither drift reads $g_n\sqrt{\alpha\omega_n}\cos(\omega_nu+\theta_n+kJ)$. Every step of the proof of Theorem~\ref{thm:quant} is invariant under such offsets: the dither antiderivatives become $v_c(u)=\sqrt{\alpha/\omega_n}\,\sin(\omega_nu+\theta_n)$ and $v_s(u)=-\sqrt{\alpha/\omega_n}\,\cos(\omega_nu+\theta_n)$ and obey the same bounds; the product-to-sum identities of Step~2 produce zero-mean trigonometric factors at the same frequencies $\omega|r_n\pm r_i|$ and $2\omega r_n$, now with constant phases, to which Lemma~\ref{lem:antider} applies verbatim; and the resonant extraction $\sin^2(\omega_nu+\theta_n)=\tfrac12-\tfrac12\cos(2\omega_nu+2\theta_n)$ is unchanged. No constant in that proof depends on the initial condition (beyond $\Expec\|\x^*\|^2<\infty$, used only for a priori moment finiteness) or on $t_0$.
\end{proof}

\begin{assumption}[(W): start-uniform window averaging]\label{ass:window}
There exist $T_0>0$ and $\varepsilon:[1,\infty)\to(0,\infty)$ with $\varepsilon(\omega)\to0$ as $\omega\to\infty$ such that the coupled pair of Lemma~\ref{lem:shift} satisfies, for every $t_0\ge0$ and every $\mathcal F_{t_0}$-measurable $\x^*$ with $\Expec\|\x^*\|^2<\infty$,
\begin{equation}
    \Expec\Big[\sup_{t_0\le t\le t_0+T_0}\big\|\x_\omega(t)-\bar\x(t)\big\|_2^2\Big]\;\le\;\varepsilon(\omega)^2\,\big(1+\Expec\|\x^*\|^2\big).
    \label{eq:supp-window}
\end{equation}
\end{assumption}

By Lemma~\ref{lem:shift}, (W) holds with $\varepsilon(\omega)=(C(T_0)/\omega)^{1/2}$ (for any fixed $T_0$, say $T_0=1$), and even without the factor $(1+\Expec\|\x^*\|^2)$, whenever Assumption~\ref{ass:glob} does---in particular whenever $\nabla J$ is bounded. The factor is what makes (W) a meaningful hypothesis for coercive $J$: there the averaged drift grows linearly while the dither drift of \eqref{eq:supp-sde} is bounded by $C\sqrt\omega$, so from a start with $\|\x^*\|\gg\sqrt\omega$ the two processes separate at a rate proportional to $\|\x^*\|$, and no bound uniform in $\x^*$ can hold. For coercive $J$, (W) is precisely the quantitative, start-uniform content that the localized Theorem~\ref{thm:main} does not provide; see Remark~\ref{rem:T2}.

\begin{proposition}[Uniform-in-time sampling accuracy]\label{prop:uniform}
Let $\f,\bsig,J$ be independent of $t$, let Assumption~\ref{ass:growth} hold with locally Lipschitz coefficients, and assume: (i)~Assumption~\ref{ass:window}; and (ii)~there is $\rho>0$ such that $\mathcal W_2(\mu\bar P_t,\nu\bar P_t)\le e^{-\rho t}\,\mathcal W_2(\mu,\nu)$ for all $\mu,\nu\in\mathcal P_2(\R^N)$ and all $t\ge0$. Then $\bar P$ has a unique invariant law $\bar\pi\in\mathcal P_2(\R^N)$, with second moment $\bar m:=(\int\|\x\|^2\,d\bar\pi)^{1/2}$, and, writing $\eta:=1-e^{-\rho T_0}\in(0,1)$ and $\rho':=\eta/(2T_0)$, for every $\omega$ with $\varepsilon(\omega)\le\eta/2$ and every $t\ge0$,
\begin{equation}
    \mathcal W_2\big(\mathrm{Law}(\x_\omega(t)),\,\bar\pi\big)
    \;\le\;
    2e^{\eta/2}\,e^{-\rho' t}\,\mathcal W_2\big(\mathrm{Law}(\x_0),\,\bar\pi\big)
    \;+\;\Big(1+\frac{4}{\eta}\Big)(1+\bar m)\,\varepsilon(\omega).
    \label{eq:supp-uniform-bound}
\end{equation}
If (W) holds without the factor $(1+\Expec\|\x^*\|^2)$, as it does under Assumption~\ref{ass:glob}, the same argument gives the sharper bound $e^{-\rho t}\mathcal W_2(\mathrm{Law}(\x_0),\bar\pi)+(1+1/\eta)\,\varepsilon(\omega)$ for all $\omega\ge1$. In either case $\lim_{\omega\to\infty}\limsup_{t\to\infty}\mathcal W_2(\mathrm{Law}(\x_\omega(t)),\bar\pi)=0$, and after a burn-in of order $\rho'^{-1}\log\big(1/\varepsilon(\omega)\big)$ the law of the ES iterate remains within $C\varepsilon(\omega)$ of $\bar\pi$ in $\mathcal W_2$ for \emph{all} subsequent times simultaneously; the time limit is a $\limsup$ because at fixed $\omega$ the law of the dithered process is asymptotically periodic in $t$, with oscillations of amplitude $O(\varepsilon(\omega))$, rather than convergent. Moreover, hypothesis (ii) holds whenever $\bsig\equiv\bsig_0$ is constant and the averaged drift is $\rho$-strongly monotone, $\langle\bar\bmu(\x)-\bar\bmu(\y),\,\x-\y\rangle\le-\rho\|\x-\y\|^2$ for all $\x,\y$.
\end{proposition}

\begin{proof}
\emph{Sufficiency of strong monotonicity for (ii).} Let $\bar\x,\bar\y$ solve the averaged SDE from initial conditions $(\bar\x_0,\bar\y_0)$ distributed according to an optimal $\mathcal W_2$-coupling of $(\mu,\nu)$, driven by the \emph{same} Brownian motion. With $\bsig$ constant the noise cancels in the difference, and $\frac{d}{dt}\|\bar\x-\bar\y\|^2=2\langle\bar\bmu(\bar\x)-\bar\bmu(\bar\y),\bar\x-\bar\y\rangle\le-2\rho\|\bar\x-\bar\y\|^2$, so $\|\bar\x(t)-\bar\y(t)\|\le e^{-\rho t}\|\bar\x_0-\bar\y_0\|$ pathwise; taking $L^2(\Prob)$-norms yields $\mathcal W_2(\mu\bar P_t,\nu\bar P_t)\le e^{-\rho t}\mathcal W_2(\mu,\nu)$.

\emph{Invariant law.} By (ii), $\mu\mapsto\mu\bar P_{T_0}$ is a strict contraction on the complete metric space $(\mathcal P_2,\mathcal W_2)$ \citep{villani2009optimal}, hence has a unique fixed point $\bar\pi$. For $s>0$, $(\bar\pi\bar P_s)\bar P_{T_0}=(\bar\pi\bar P_{T_0})\bar P_s=\bar\pi\bar P_s$, so $\bar\pi\bar P_s$ is also a fixed point of $\bar P_{T_0}$ and therefore equals $\bar\pi$: $\bar\pi$ is invariant for all times, and unique in $\mathcal P_2$ by (ii).

\emph{Window recursion.} Fix $\omega\ge1$, set $t_j:=jT_0$, write $\mu_t:=\mathrm{Law}(\x_\omega(t))\in\mathcal P_2$ (second moments are finite at every time for fixed $\omega$ by Assumption~\ref{ass:growth}), and $a_j:=\mathcal W_2(\mu_{t_j},\bar\pi)<\infty$. On $[t_j,t_{j+1}]$ let $\bar\x^{(j)}$ solve the averaged SDE from $\bar\x^{(j)}(t_j)=\x_\omega(t_j)$, driven by the same Brownian motion as $\x_\omega$. Since the averaged coefficients are time-homogeneous and the increments $(\w_{t_j+s}-\w_{t_j})_{s\ge0}$ are independent of $\mathcal F_{t_j}$, $\mathrm{Law}(\bar\x^{(j)}(t_j+s))=\mu_{t_j}\bar P_s$ for $s\in[0,T_0]$; and $\big(\x_\omega(t_j+s),\bar\x^{(j)}(t_j+s)\big)$ is a coupling of $(\mu_{t_j+s},\,\mu_{t_j}\bar P_s)$, so by \eqref{eq:supp-window} with $\x^*=\x_\omega(t_j)$,
\begin{align*}
    \mathcal W_2\big(\mu_{t_j+s},\,\mu_{t_j}\bar P_s\big)
    \;&\le\;\Big(\Expec\big\|\x_\omega(t_j+s)-\bar\x^{(j)}(t_j+s)\big\|_2^2\Big)^{1/2}\\
    &\le\;\varepsilon(\omega)\,\big(1+m_2(\mu_{t_j})\big)
    \;\le\;\varepsilon(\omega)\,\big(1+\bar m+a_j\big),
    \qquad s\in[0,T_0],
\end{align*}
where $m_2(\mu):=(\int\|\x\|^2d\mu)^{1/2}=\mathcal W_2(\mu,\delta_{\mathbf 0})\le\mathcal W_2(\mu,\bar\pi)+\bar m$. Hence, by the triangle inequality, invariance of $\bar\pi$, and (ii),
\begin{align*}
    a_{j+1}\;&\le\;\mathcal W_2\big(\mu_{t_{j+1}},\mu_{t_j}\bar P_{T_0}\big)+\mathcal W_2\big(\mu_{t_j}\bar P_{T_0},\bar\pi\bar P_{T_0}\big)\\
    &\le\;\varepsilon(\omega)(1+\bar m)+\big(e^{-\rho T_0}+\varepsilon(\omega)\big)a_j
    \;\le\;\varepsilon(\omega)(1+\bar m)+\big(1-\tfrac\eta2\big)a_j
\end{align*}
for $\varepsilon(\omega)\le\eta/2$, so that $a_j\le(1-\tfrac\eta2)^ja_0+\tfrac2\eta(1+\bar m)\varepsilon(\omega)$: the second moment of $\mu_{t_j}$, which the weighted form of (W) feeds back into the recursion, is controlled by $a_j$ itself, so no separate moment bound is needed. For arbitrary $t=t_j+s$ with $s\in[0,T_0)$,
\begin{align*}
    \mathcal W_2(\mu_t,\bar\pi)
    \;&\le\;\mathcal W_2\big(\mu_t,\mu_{t_j}\bar P_s\big)+\mathcal W_2\big(\mu_{t_j}\bar P_s,\bar\pi\bar P_s\big)
    \;\le\;\varepsilon(\omega)(1+\bar m+a_j)+e^{-\rho s}a_j\\
    &\le\;\varepsilon(\omega)(1+\bar m)+2a_j
    \;\le\;2e^{\eta/2}e^{-\rho't}a_0+\Big(1+\tfrac4\eta\Big)(1+\bar m)\,\varepsilon(\omega),
\end{align*}
using $(1-\eta/2)^j\le e^{-\eta j/2}=e^{-\eta t_j/(2T_0)}\le e^{\eta/2}e^{-\rho't}$, which is \eqref{eq:supp-uniform-bound}. Without the factor $(1+\Expec\|\x^*\|^2)$ in (W) the first display gives $\mathcal W_2(\mu_{t_j+s},\mu_{t_j}\bar P_s)\le\varepsilon(\omega)$, the recursion reads $a_{j+1}\le\varepsilon(\omega)+e^{-\rho T_0}a_j$, and the same steps give the sharper bound quoted after \eqref{eq:supp-uniform-bound}.
\end{proof}

\begin{corollary}[Strongly log-concave targets]\label{cor:loguniform}
Take instantiation (I$_\lambda$): $\f\equiv\mathbf 0$, $\bsig=\sqrt2\,I$, $k\alpha=\lambda>0$, and $J=\E\in C^2$ independent of $t$ with $\nabla\E$ globally Lipschitz and $\nabla^2\E(\x)\succeq mI$ for all $\x$ and some $m>0$ (a strongly log-concave target). Then hypothesis (ii) of Proposition~\ref{prop:uniform} holds with $\rho=\lambda m$, the unique invariant law is $\pi_\lambda\propto e^{-\lambda\E}\in\mathcal P_2$, and consequently, under Assumption~\ref{ass:window}, with $\eta=1-e^{-\lambda mT_0}$, $\rho'=\eta/(2T_0)$ and $\bar m$ the second moment of $\pi_\lambda$,
\begin{equation*}
    \mathcal W_2\big(\mathrm{Law}(\x_\omega(t)),\,\pi_\lambda\big)
    \;\le\;2e^{\eta/2}e^{-\rho't}\,\mathcal W_2\big(\mathrm{Law}(\x_0),\,\pi_\lambda\big)
    +\Big(1+\tfrac{4}{\eta}\Big)(1+\bar m)\,\varepsilon(\omega)
\end{equation*}
for all $t\ge0$ and all $\omega$ with $\varepsilon(\omega)\le\eta/2$,
so the iterated limit of Corollary~\ref{cor:sampling}---whose dissipativity hypothesis is implied by strong convexity---upgrades to $\lim_{\omega\to\infty}\limsup_{t\to\infty}\mathcal W_2(\mathrm{Law}(\x_\omega(t)),\pi_\lambda)=0$, with the frequency limit taken after, or jointly with, the time limit.
\end{corollary}

\begin{proof}
The averaged drift is $\bar\bmu=-\lambda\nabla\E$, and strong convexity gives $\langle\bar\bmu(\x)-\bar\bmu(\y),\x-\y\rangle=-\lambda\langle\nabla\E(\x)-\nabla\E(\y),\x-\y\rangle\le-\lambda m\|\x-\y\|^2$, so the sufficient condition of Proposition~\ref{prop:uniform} applies with $\rho=\lambda m$ and constant $\bsig$. The law $\pi_\lambda$ is invariant by Corollary~\ref{cor:tempering}, lies in $\mathcal P_2$ because $\E(\x)\ge\E(\x_*)+\tfrac m2\|\x-\x_*\|^2$ at the minimizer $\x_*$ makes $e^{-\lambda\E}$ sub-Gaussian, and is the unique invariant law in $\mathcal P_2$ by the contraction. The well-posedness and moment bounds required by Assumption~\ref{ass:growth} hold because $\nabla\E$ is Lipschitz, hence of linear growth.
\end{proof}

\begin{remark}[Status of the uniform-in-time program]\label{rem:T2}
The two ingredients of Proposition~\ref{prop:uniform} are individually proved but, at present, never simultaneously. Under Assumption~\ref{ass:glob}, (W) is a theorem (Lemma~\ref{lem:shift}), but no bounded drift satisfies (ii): if $\|\bar\bmu\|_\infty=B<\infty$, the mean functions $\mathbf m_\x(t):=\Expec[\bar\x(t)\,|\,\bar\x(0)=\x]$ obey $\|\mathbf m_\x(t)-\mathbf m_\y(t)\|\ge\|\x-\y\|-2Bt$, and since every coupling satisfies $\Expec\|X-Y\|\ge\|\Expec X-\Expec Y\|$, also $\mathcal W_2(\delta_\x\bar P_t,\delta_\y\bar P_t)\ge\|\mathbf m_\x(t)-\mathbf m_\y(t)\|$; the contraction $e^{-\rho t}\|\x-\y\|\ge\|\x-\y\|-2Bt$ then fails whenever $\|\x-\y\|>2Bt/(1-e^{-\rho t})$. Conversely, under strong log-concavity (ii) holds, but $\nabla^2\E\succeq mI$ forces $\|\nabla\E(\x)\|\ge m\|\x\|-\|\nabla\E(\mathbf 0)\|$, so $\nabla\E$ is unbounded and Assumption~\ref{ass:glob}---hence the proof of (W)---is unavailable; the localization behind Theorem~\ref{thm:main} covers such coercive energies, but yields an in-probability window statement without the start-uniform mean-square form that the recursion consumes. (The growth weight $(1+\Expec\|\x^*\|^2)$ in \eqref{eq:supp-window} is not a convenience but a necessity here: the dither drift is bounded by $C\sqrt\omega$ while $\|\bar\bmu(\x)\|\ge\lambda m\|\x\|-C$, so the window error from a start $\x^*$ is at least of order $(\lambda m\|\x^*\|-C\sqrt\omega)_+T_0$, which no $\x^*$-free $\varepsilon(\omega)$ can bound.) What a proof of (W) for $m$-strongly convex $\E$ with Lipschitz gradient requires is a linear-growth-weighted extension of the quantitative machinery: versions of Lemmas~\ref{lem:antider}--\ref{lem:osc-moll} for slow factors of linear growth along $\x_\omega$, together with $\omega$-uniform second- and fourth-moment bounds for the dithered process, in the spirit of the perturbed-Lyapunov-function method of stochastic averaging \citep{kushner1984approximation}. We expect (W) to hold there with $\varepsilon(\omega)=C\omega^{-1/2}$ and record it as an assumption rather than assert it. In summary: the transfer principle and the contraction are proved (Proposition~\ref{prop:uniform}, Corollary~\ref{cor:loguniform}), and the start-uniform window estimate is the single isolated remaining ingredient in the coercive case---in the same spirit as Conjecture~\ref{conj:occ} on the Lipschitz rung, which is proved in Theorem \ref{thm:occ} under Lipschitz time dependence (Theorem~\ref{thm:occ}); the pathwise near-identity transformation of Appendix~\ref{supp:rough-occ}, which is a perturbed-Lyapunov device in all but name, is the natural starting point for the moment bounds that (W) requires.
\end{remark}

\section{Relaxed regularity: sampling non-differentiable energies}
\label{supp:nonsmooth}

(Throughout Appendices~\ref{supp:nonsmooth}--\ref{supp:rough-boundary}, as in Appendix~\ref{supp:proof-main}, $N=d$ denotes the state dimension.) Theorem~\ref{thm:main} assumes $J(\cdot,t)\in C^2$, but this is a convenience of the proof, not a requirement of the method: the algorithm \eqref{eq:supp-sde}--\eqref{eq:supp-drift} only ever evaluates $\cos(\omega_n t+kJ)$, which is well defined for any measurable $J$, and the averaged drift only needs $\nabla J$ to exist in a suitably weak sense. Two mechanisms make weak derivatives admissible. The first is purely analytic and needs no noise at all: if $\nabla J$ is Lipschitz, mollification costs only $\varepsilon$-powers that the dither frequency can absorb (Appendix~\ref{supp:rough-c11}). The second is the noise: for a uniformly non-degenerate diffusion, the law of the process is absolutely continuous, so quantities defined only almost everywhere are seen unambiguously along the trajectory (Appendix~\ref{supp:rough-noise}). We use the first to prove the $C^{1,1}$ case with a rate, and both together for the Sobolev case.

Throughout this appendix, mollification is in $\x$ only: $J_\varepsilon:=J*_\x\rho_\varepsilon$ for a standard mollifier $\rho_\varepsilon$, and Assumption~\ref{ass:glob} below is retained for $\f,\bsig$ and for the $t$-regularity of $J$ (bounded variation, uniformly in $\x$; Theorem~\ref{thm:sobolev} strengthens this to $\partial_tJ$ bounded, which its It\^o--Krylov step consumes). ``Uniformly locally'' refers to the norms $\|h\|_{L^p_{\mathrm{uloc}}}:=\sup_{\x_0}\|h\|_{L^p(B_1(\x_0))}$. On every rung below, $\nabla J$ remains (essentially) bounded, so the dither drift $\x\mapsto g_n(\x,t)\sqrt{\alpha\omega_n}\cos(\omega_nt+kJ(\x,t))$ stays globally Lipschitz in $\x$ and \eqref{eq:supp-sde} retains unique global strong solutions for each fixed $\omega$; the well-posedness questions below concern only the averaged limit.

\subsection{Global hypotheses and the quantitative baseline}
\label{supp:rough-quant}

The rough-energy results are cleanest under the following global form of the hypotheses, which removes the localization of Steps~0 and~4 of Appendix~\ref{supp:proof-main} and upgrades the convergence to a mean-square estimate with an explicit rate.

\begin{assumption}[Global regularity]\label{ass:glob}
There exist finite constants such that, uniformly on $\R^N\times[0,T]$:
(i)~each $f_n,g_n$ is twice continuously differentiable in $\x$, with the function and its first two $\x$-derivatives bounded and globally Lipschitz in $\x$; $J$ is twice continuously differentiable in $\x$ with $\nabla_\x J$ and $\nabla^2_\x J$ bounded and globally Lipschitz in $\x$ ($J$ itself may be unbounded: it enters \eqref{eq:supp-drift} only through the $2\pi$-periodic phase of the cosine, and no bound on $|J|$ is used anywhere below, so linear-growth energies with bounded gradient---Laplace-type tails---are covered as stated);
(ii)~$f_n,g_n,J$ and their first two $\x$-derivatives are of bounded variation in $t$, uniformly in $\x$ (continuous differentiability in $t$ with bounded derivative, as in Assumption~\ref{ass:reg}, is a special case);
(iii)~the dither ratios $r_1,\dots,r_N>0$ are pairwise distinct (Assumption~\ref{ass:freq}) and normalized so that $\min_nr_n=1$ (a common scale is absorbed into $\omega$; without this normalization the ratio-only dependence of the constants below would not be scale-invariant).
\end{assumption}

We write $C_0$ for a common bound on $|f_n|,|g_n|$, on their first two $\x$-derivatives, and on $|\nabla_\x J|,|\nabla^2_\x J|$; $L$ for a common Lipschitz constant (in $\x$) of these quantities; and $V$ for a common bound on their total variation in $t$, all uniform on $\R^N\times[0,T]$. We abbreviate $r_{\min}=\min_nr_n=1$, $r_{\max}=\max_nr_n$, and
\begin{equation}
    c_r \;:=\; \max\Big(\max_{i\neq j}|r_i-r_j|^{-1},\ (2r_{\min})^{-1}\Big),
    \label{eq:supp-cr}
\end{equation}
the reciprocal frequency spacing that carries the dimension dependence of the constants (cf.\ Sec.~\ref{sec:practice} and Appendix~\ref{supp:dimension}). All constants $C$ in this appendix depend only on $(T,N,\alpha,k,C_0,L,V,r_{\max},c_r)$ and may change from line to line; we take $\omega\ge1$ so that $\omega^{-1}\le\omega^{-1/2}$.

\begin{lemma}[Dither antiderivative]\label{lem:antider}
Let $\zeta_\omega(s)=\sum_mc_m\,\mathrm{trig}(\omega\nu_ms+\phi_m)$ be a finite sum with $\mathrm{trig}\in\{\sin,\cos\}$, arbitrary fixed phases $\phi_m\in\R$, $\nu_m\neq0$, and $\sum_m|c_m|\le1$. Then $A_\omega(t):=\int_0^t\zeta_\omega(s)\,ds$ satisfies $\|A_\omega\|_{L^\infty[0,T]}\le2c_\zeta/\omega$ with $c_\zeta:=\max_m|\nu_m|^{-1}$. For the zero-mean trigonometric factors produced by the dither products \eqref{eq:supp-cross} of Step~2, $\nu_m\in\{r_n\pm r_i,\,2r_n\}$, and hence $c_\zeta\le c_r$ of \eqref{eq:supp-cr}.
\end{lemma}

\begin{proof}
Integrate each summand explicitly: $\big|\int_0^t\mathrm{trig}(\omega\nu s+\phi)\,ds\big|=\big|\mathrm{trig}^*(\omega\nu t+\phi)-\mathrm{trig}^*(\phi)\big|/(\omega|\nu|)\le2(\omega|\nu|)^{-1}$, where $\mathrm{trig}^*$ denotes the corresponding antiderivative ($\sin$ for $\cos$, and $-\cos$ for $\sin$).
\end{proof}

\begin{lemma}[Oscillatory integrals along the ES process, quantitative form]\label{lem:osc-moll}
Let Assumption~\ref{ass:glob} hold for $\f,\bsig$ (for $J$, only boundedness and Lipschitz continuity of $\nabla J$ enter, through the drift bound), let $\x_\omega$ solve \eqref{eq:supp-sde}, and let $\zeta_\omega,c_\zeta$ be as in Lemma~\ref{lem:antider}. Let $\Phi:\R^N\times[0,T]\to\R$ be bounded, with $\nabla_\x\Phi$ bounded and Lipschitz in $\x$ uniformly in $t$, and with $\Phi(\x,\cdot)$ of bounded variation uniformly in $\x$. Then, for $\omega\ge1$,
\begin{equation}
    \Expec\Big[\sup_{t\le T}\Big|\int_0^t\Phi\big(\x_\omega(s),s\big)\,\zeta_\omega(s)\,ds\Big|^2\Big]
    \;\le\;\frac{C_\Phi\,(1+c_\zeta)^2}{\omega},
    \label{eq:supp-osc-moll}
\end{equation}
with $C_\Phi$ depending only on $T,N,\alpha,r_{\max}$, the coefficient bounds, and $\|\Phi\|_\infty,\|\nabla_\x\Phi\|_\infty,\mathrm{Lip}(\nabla_\x\Phi),\mathrm{TV}_t\Phi$.
\end{lemma}

\begin{proof}
Write $A=A_\omega$ as in Lemma~\ref{lem:antider}, so $\|A\|_\infty\le2c_\zeta/\omega$; for readability the bounds below are written with $c_\zeta/\omega$, the factor $2$ being absorbed into the generic constants. Mollify $\Phi$ in $\x$ only: $\Phi_\delta:=\Phi*_\x\rho_\delta$ with a standard mollifier, $\delta:=\omega^{-1/2}$. Then $\|\Phi-\Phi_\delta\|_\infty\le\|\nabla_\x\Phi\|_\infty\delta$, $\|\nabla_\x\Phi_\delta\|_\infty\le\|\nabla_\x\Phi\|_\infty$, $\|\nabla^2_\x\Phi_\delta\|_\infty\le\mathrm{Lip}(\nabla_\x\Phi)$ (mollification of a Lipschitz gradient does not inflate the bound), and $\mathrm{TV}_t\Phi_\delta\le\mathrm{TV}_t\Phi$. Split
\begin{equation*}
    \int_0^t\Phi\,\zeta_\omega\,ds
    \;=\;\int_0^t(\Phi-\Phi_\delta)\,\zeta_\omega\,ds
    \;+\;\int_0^t\Phi_\delta\,\zeta_\omega\,ds;
\end{equation*}
the first term is bounded pathwise by $\|\nabla_\x\Phi\|_\infty T\,\omega^{-1/2}$. For the second, integrate by parts against the finite-variation function $A$ (the covariation $[\Phi_\delta(\x_\cdot,\cdot),A]$ vanishes):
\begin{equation*}
    \int_0^t\Phi_\delta\big(\x_\omega(s),s\big)\,dA(s)
    \;=\;\Phi_\delta\big(\x_\omega(t),t\big)A(t)-\int_0^tA(s)\,d\big[\Phi_\delta\big(\x_\omega(s),s\big)\big],
\end{equation*}
and expand $d[\Phi_\delta(\x_\omega(s),s)]$ by It\^o's formula ($\Phi_\delta\in C^2$ in $\x$; the $t$-dependence contributes a Stieltjes term):
\begin{equation*}
    d\Phi_\delta
    \;=\;\Phi_\delta(\x,ds)
    +\nabla_\x\Phi_\delta\!\cdot\!\bmu_\omega\,ds
    +\tfrac12\operatorname{tr}\!\big(\bsig\bsig^{\!\top}\nabla^2_\x\Phi_\delta\big)\,ds
    +(\nabla_\x\Phi_\delta)^{\!\top}\bsig\,d\w .
\end{equation*}
The resulting terms are bounded as follows, using $\|A\|_\infty\le c_\zeta/\omega$: the boundary term by $\|\Phi\|_\infty c_\zeta/\omega$; the Stieltjes term by $(c_\zeta/\omega)\,\mathrm{TV}_t\Phi$; the drift term---the dominant one, since $\|\bmu_\omega\|\le C_0\sqrt N\,(1+\sqrt{\alpha r_{\max}\omega})$---by $(c_\zeta/\omega)\,\|\nabla_\x\Phi\|_\infty\,C_0\sqrt N\,(1+\sqrt{\alpha r_{\max}\omega})\,T\le C\,c_\zeta\,\omega^{-1/2}$; the It\^o-correction term by $(c_\zeta/\omega)\tfrac N2C_0^2\,\mathrm{Lip}(\nabla_\x\Phi)\,T$; and the martingale term $M_t=\int_0^tA\,(\nabla_\x\Phi_\delta)^{\!\top}\bsig\,d\w$ by Doob's maximal inequality and the It\^o isometry, $\Expec\sup_{t\le T}|M_t|^2\le4(c_\zeta/\omega)^2\|\nabla_\x\Phi\|_\infty^2NC_0^2T$. Collecting: the pathwise-bounded parts total at most $C(1+c_\zeta)\,\omega^{-1/2}$, and the martingale part has mean-square supremum $O(c_\zeta^2\omega^{-2})$; squaring and using $(x+y)^2\le2x^2+2y^2$ gives \eqref{eq:supp-osc-moll}.
\end{proof}

The mollification step is why Lemma~\ref{lem:osc-moll} asks only that $\nabla_\x\Phi$ be Lipschitz rather than that $\Phi\in C^2$ with Lipschitz second derivatives (compare the genuine integration by parts of Remark~\ref{rem:rate-supp}, which asked for one more derivative): in the application, $\Phi$ is built from first derivatives of the correctors $u^c,u^s$ of Step~1, so $\nabla_\x\Phi$ involves second derivatives of $J$ and $g_n$, whose Lipschitz continuity is exactly Assumption~\ref{ass:glob}(i); no third derivatives are ever required.

\begin{theorem}[Quantitative averaging under global bounds]\label{thm:quant}
Under Assumption~\ref{ass:glob} there is a constant $C=C_1e^{C_2T}$, with $C_1,C_2$ depending polynomially on $N,\alpha,|k|,C_0,L,V,r_{\max}$ and $c_r$, such that for all $\omega\ge1$,
\begin{equation}
    \Expec\Big[\sup_{0\le t\le T}\big\|\x_\omega(t)-\bar\x(t)\big\|_2^2\Big]
    \;\le\;\frac{C}{\omega}.
    \label{eq:supp-quant}
\end{equation}
In particular $\Prob\big(\sup_{t\le T}\|\x_\omega-\bar\x\|_2>\delta\big)\le C/(\delta^2\omega)$ for every $\delta>0$, i.e., $\sup_{t\le T}\|\x_\omega-\bar\x\|=O_\Prob(\omega^{-1/2})$: the localized $\omega^{-1/4}$ is sharpened without the $C^3$ smoothness of Remark~\ref{rem:rate-supp}.
\end{theorem}

\begin{proof}
Repeat Steps~1--3 of Appendix~\ref{supp:proof-main} with $\tau_R\equiv T$: every bound of Assumption~\ref{ass:glob} is global, so no localization is needed and every constant $C_R$ becomes a global constant. (Under the bounded-variation form of Assumption~\ref{ass:glob}(ii), the $\partial_tu$ contribution of Step~1 is read as a Stieltjes differential and bounded by the total variation $V$, exactly as the Stieltjes term in the proof of Lemma~\ref{lem:osc-moll}.) The one upgrade is in Step~2: the non-resonant family $E^{(2)}_n$ is estimated by Lemma~\ref{lem:osc-moll} in place of Lemma~\ref{lem:osc-main-ref}. Each nonzero-frequency term of \eqref{eq:supp-cross} has the form $\alpha\sqrt{r_i/r_n}\int_0^t\Phi(\x_\omega(s),s)\,\zeta_\omega(s)\,ds$ with prefactor $\alpha\sqrt{r_i/r_n}\le\alpha\sqrt{r_{\max}}$, slow factor $\Phi\in\{\partial_iu^{c/s}g_i\cos kJ,\ \partial_iu^{c/s}g_i\sin kJ\}$, and $\zeta_\omega$ a zero-mean trigonometric factor at frequency $\omega|r_n\pm r_i|$ or $2\omega r_n$, so that $c_\zeta\le c_r$ by Lemma~\ref{lem:antider}. Under Assumption~\ref{ass:glob}(i)--(ii) each such $\Phi$ is bounded by $C_0^2(1+|k|C_0)$, with $\nabla_\x\Phi$ bounded and Lipschitz and $\mathrm{TV}_t\Phi\le CV$; Lemma~\ref{lem:osc-moll} therefore bounds each term's mean-square supremum by $C(1+c_r)^2/\omega$, and summing the at most $8N$ terms per coordinate, with $(\sum_{j\le m}y_j)^2\le m\sum_jy_j^2$, gives
\begin{equation*}
    \Expec\sup_{t\le T}\big\|\mathbf E(t)\big\|^2\;\le\;\frac C\omega
\end{equation*}
in place of \eqref{eq:supp-mild}, the family $E^{(1)}$ being $O(\omega^{-1})$ in mean-square supremum with global constants exactly as in Step~1. The Gr\"onwall comparison of Step~3 closes the estimate: the averaged drift $\bar\bmu=\f-\tfrac{k\alpha}2\bsig\bsig^{\!\top}\nabla_\x J$ is bounded and globally Lipschitz (with constant $L_b\le L(1+|k|\alpha C_0^2)$), as is $\bsig$, so \eqref{eq:supp-gronwall} holds with $\tau_R\equiv T$, a global constant, and $\omega^{-1}$ in place of $\omega^{-1/2}$. (Both suprema have finite second moments at each fixed $\omega$: the dither drift is bounded by $C_0\sqrt{\alpha r_{\max}\omega}$, so the standard moment bounds for bounded-coefficient It\^o SDEs apply.)
\end{proof}

\begin{remark}[Where the dither ratios enter]\label{rem:quant-const}
The dependence of $C$ in \eqref{eq:supp-quant} on the dither ratios enters in exactly two places: the prefactors $\alpha\sqrt{r_i/r_n}\le\alpha\sqrt{r_{\max}}$ of the dither products, and the antiderivative bound $c_\zeta\le c_r$ of Lemma~\ref{lem:antider}, which is dominated by the reciprocal spacing $\max_{i\neq j}|r_i-r_j|^{-1}$. Packing $N$ pairwise-distinct ratios into a fixed band therefore costs a factor $O(N)$ in $c_r$ and hence $O(N^2)$ in \eqref{eq:supp-quant}; spreading the ratios out instead inflates $\omega r_{\max}$ and the evaluation budget at fixed step resolution. This is the quantitative counterpart of the frequency-gap discussion of Sec.~\ref{sec:practice}, whose sharp finite-$\omega$ version is the subject of Appendix~\ref{supp:dimension}.
\end{remark}

\subsection{$C^{1,1}$ energies: rough curvature without ellipticity}
\label{supp:rough-c11}

\medskip\noindent\textbf{Proposition~\ref{prop:c11} ($C^{1,1}$ energies; restated).}\ \emph{In Assumption~\ref{ass:glob}, replace the hypothesis on $J$ by: $\nabla J$ bounded and globally Lipschitz with constant $M$ (so $\nabla^2J$ exists a.e.\ with $\|\nabla^2J\|_{L^\infty}\le M$ but may be discontinuous; $J$ itself may be unbounded). Then, with no non-degeneracy assumption on $\bsig$, the conclusion of Theorem~\ref{thm:quant} holds unchanged: $\Expec\sup_{t\le T}\|\x_\omega-\bar\x\|_2^2\le C/\omega$.}

\begin{proof}
Set $\varepsilon=\omega^{-1/2}$ and let $J_\varepsilon=J*\rho_\varepsilon$. Because $\nabla J$ is $M$-Lipschitz,
\begin{equation}
    \|J-J_\varepsilon\|_\infty\le\|\nabla J\|_\infty\varepsilon,
    \quad
    \|\nabla J-\nabla J_\varepsilon\|_\infty\le M\varepsilon,
    \quad
    \|\nabla^2J_\varepsilon\|_\infty\le M,
    \quad
    \|\nabla^3J_\varepsilon\|_\infty\le CM/\varepsilon.
    \label{eq:supp-moll-bounds}
\end{equation}
Split each dither integral $I_{2,n}$ of Step~1 as $I^\varepsilon_{2,n}+M_{2,n}$, where $I^\varepsilon_{2,n}$ carries $\cos(\omega_ns+kJ_\varepsilon)$ and
\begin{equation}
\begin{split}
    M_{2,n}
    &=\int_0^tg_n\sqrt{\alpha\omega_n}\,\big[\cos(\omega_ns{+}kJ)-\cos(\omega_ns{+}kJ_\varepsilon)\big]\,ds\\
    &=-2\int_0^tg_n\sqrt{\alpha\omega_n}\,\sin\!\big(\omega_ns+k\Theta_\varepsilon\big)\,\psi_\varepsilon\,ds,
\end{split}
    \label{eq:supp-mismatch}
\end{equation}
with $\Theta_\varepsilon:=\tfrac{J+J_\varepsilon}2$ and $\psi_\varepsilon:=\sin\!\big(\tfrac k2(J-J_\varepsilon)\big)$, so that $\|\psi_\varepsilon\|_\infty\le\tfrac{|k|}2\|\nabla J\|_\infty\varepsilon$ and, crucially, $\|\nabla\psi_\varepsilon\|_\infty\le\tfrac{|k|}2M\varepsilon$ by \eqref{eq:supp-moll-bounds}. (Throughout this proof, $O(\cdot)$ refers to the pathwise supremum on $[0,T]$ for finite-variation contributions and to the mean-square supremum, via Doob, for martingale contributions, as in Appendix~\ref{supp:proof-main}.)

\emph{The $\varepsilon$-corrector stream.} Run Steps~1--2 of Appendix~\ref{supp:proof-main} on $I^\varepsilon_{2,n}$ with correctors $u^c_\varepsilon=g_n\cos(kJ_\varepsilon)$, $u^s_\varepsilon=g_n\sin(kJ_\varepsilon)$, estimating the non-resonant dither products by Lemma~\ref{lem:osc-moll} as in the proof of Theorem~\ref{thm:quant}. Every contribution to the family $E^{(1)}_n$ is bounded exactly as in Step~1, with constants now involving $\|\nabla^2J_\varepsilon\|_\infty\le M$ (uniform in $\varepsilon$), so the family remains $O(\omega^{-1})$ in mean-square supremum. In the non-resonant family the slow factors are mixed, e.g.\ $\Phi=\partial_iu^c_\varepsilon\,g_i\cos(kJ)$: decompose $\cos(kJ)=\cos(kJ_\varepsilon)+[\cos kJ-\cos kJ_\varepsilon]$. The all-$\varepsilon$ part is handled by Lemma~\ref{lem:osc-moll}, whose It\^o-correction contribution carries $\mathrm{Lip}(\nabla\Phi^\varepsilon)\le C(M)(1+\|\nabla^3J_\varepsilon\|_\infty)\le CM/\varepsilon$ and is therefore $O\big(c_r/(\varepsilon\omega)\big)=O(\omega^{-1/2})$, all other contributions being $O(\omega^{-1/2})$ with $\varepsilon$-uniform constants; the difference part has amplitude $\le|k|\,\|\nabla J\|_\infty\varepsilon$ against $O(1)$ integrands, contributing $O(\varepsilon)=O(\omega^{-1/2})$ pathwise. The surviving resonant means are computed as in Step~2 but with mixed phases: the same trigonometric algebra gives
\begin{equation}
    -\tfrac{k\alpha}2\,g_n^2\,\partial_nJ_\varepsilon(\x_\omega)\,\cos\!\big(k(J{-}J_\varepsilon)\big)
    +\tfrac{\alpha}2\,g_n\,\partial_ng_n\,\sin\!\big(k(J{-}J_\varepsilon)\big)
    \;=\;-\tfrac{k\alpha}2\,g_n^2\,\partial_nJ(\x_\omega)+O(\varepsilon)
    \label{eq:supp-mixed-resonant}
\end{equation}
pathwise, using $|\partial_nJ_\varepsilon-\partial_nJ|\le M\varepsilon$, $\cos(k(J-J_\varepsilon))=1+O(k^2\varepsilon^2)$, and $\sin(k(J-J_\varepsilon))=O(|k|\varepsilon)$; the $\partial_ng_n$ cross terms, which cancelled identically in Step~2, are now merely $O(\varepsilon)$.

\emph{The mismatch stream.} Expand $\sin(\omega_ns+k\Theta_\varepsilon)=\sin(\omega_ns)\cos(k\Theta_\varepsilon)+\cos(\omega_ns)\sin(k\Theta_\varepsilon)$ in \eqref{eq:supp-mismatch} and integrate each branch by parts against the antiderivative of the fast trigonometric factor, exactly as in the proof of Lemma~\ref{lem:osc-moll}, with slow factor $\tilde u=g_n\psi_\varepsilon\,\mathrm{trig}(k\Theta_\varepsilon)$. The key point of $C^{1,1}$ regularity is that $\nabla\tilde u$ is a sum of products of \emph{bounded Lipschitz} functions (each of $\psi_\varepsilon,\nabla\psi_\varepsilon,\nabla\Theta_\varepsilon,g_n,\nabla g_n$ and the trig factors is bounded with a Lipschitz constant $\le C(M,|k|,C_0,L)$ \emph{independent of $\varepsilon$}), so $\mathrm{Lip}(\nabla\tilde u)\le C$ and the mollification device inside Lemma~\ref{lem:osc-moll} applies with $\varepsilon$-uniform constants. The resulting terms are: boundary $\le C\varepsilon\,\omega^{-1/2}$; Stieltjes and It\^o-correction $\le C\omega^{-1/2}$; the drift term, which is the dominant one,
\begin{equation*}
    \le\;\sqrt{\alpha\omega_n}\cdot\frac{C}{\omega_n}\cdot\|\nabla\tilde u\|_\infty\cdot C_0\sqrt N\big(1+\sqrt{\alpha r_{\max}\omega}\big)\,T
    \;\le\;C\,\|\nabla\tilde u\|_\infty\;\le\;C\varepsilon,
\end{equation*}
using $\|\nabla\tilde u\|_\infty\le C\big(\|\nabla\psi_\varepsilon\|_\infty+|k|\,\|\psi_\varepsilon\|_\infty\|\nabla\Theta_\varepsilon\|_\infty+\|\psi_\varepsilon\|_\infty\big)\le C\varepsilon$; and a martingale with mean-square supremum $\le C\varepsilon^2\omega^{-1}$.

Collecting both streams, the remainder $\mathbf E$ in the exact representation \eqref{eq:supp-mild} (now derived with $\partial_nJ$ via \eqref{eq:supp-mixed-resonant}, and with $\tau_R\equiv T$) obeys $\Expec\sup_{t\le T}\|\mathbf E\|^2\le C(\omega^{-1}+\varepsilon^2)=C\omega^{-1}$. The averaged drift $\f-\tfrac{k\alpha}2\bsig\bsig^{\!\top}\nabla J$ is Lipschitz ($\nabla J$ is), so the Gr\"onwall comparison of Step~3 applies verbatim. No step used non-degeneracy of $\bsig$.
\end{proof}

This rung already covers the marquee examples of ``twice differentiable off a null set'': piecewise-quadratic energies glued in $C^1$, max/soft-max compositions with matched gradients, and any energy whose curvature jumps across hypersurfaces while staying bounded---with the full rate of Theorem~\ref{thm:quant}, and for arbitrary (possibly vanishing) diffusion amplitude $\bsig$.

\subsection{The regularizing role of the noise}
\label{supp:rough-noise}

When the Hessian is not merely discontinuous but unbounded ($\nabla^2J\in L^p$ only), pathwise bounds on the mollified quantities are no longer enough and terms must be integrated along the trajectory.

\begin{lemma}[Occupation density and the Krylov estimate]\label{lem:krylov}
Let $\x$ solve $d\x=\mathbf b(\x,t)\,dt+\bsig(\x,t)\,d\w$ on $[0,T]$ with $\mathbf b,\bsig$ bounded and $\bsig\bsig^{\!\top}\succeq\lambda_0I$ uniformly ($\lambda_0>0$). Then there is a constant $N_*=N_*(\lambda_0,N,p,q,T,\|\mathbf b\|_\infty)$ such that for every $h\in L^q\big([0,T];L^p(\R^N)\big)$ with $\frac Np+\frac2q<2$,
\begin{equation}
    \Expec\int_0^T\big|h(t,\x(t))\big|\,dt\;\le\;N_*\,\|h\|_{L^q_tL^p_\x}.
    \label{eq:supp-krylov}
\end{equation}
In particular the expected occupation measure $\mu(A)=\Expec\int_0^T\mathbf 1_A(t,\x(t))\,dt$ is absolutely continuous with respect to Lebesgue measure on $[0,T]\times\R^N$. Consequently $\x$ charges no Lebesgue-null set, and two drifts that agree Lebesgue-a.e.\ produce indistinguishable dynamics.
\end{lemma}

This is the Krylov estimate for uniformly non-degenerate diffusions \citep{krylov1980controlled,krylov2005strong}. Along the averaged process, a drift or Hessian defined only almost everywhere is as good as a continuous one, because the elliptic noise smooths the law of $\x$ into an absolutely continuous measure that never sees the exceptional set. This is exactly why Langevin and SGLD samplers operate on rough energies.

Lemma~\ref{lem:krylov} cannot, however, be applied as is along the ES process $\x_\omega$: its constant depends on the drift bound, and the dither drift grows like $\sqrt\omega$. That gap is closed by the following estimate, whose proof exploits the parabolic scaling of the dither: over a window of length $1/\omega$ the dither displaces the process by only $O(\omega^{-1/2})$, which is exactly the diffusive spread, so Girsanov applies on each window.

\begin{lemma}[Microinterval occupation estimate]\label{lem:micro}
Let $\x_\omega$ solve \eqref{eq:supp-sde} with $\f,\bsig$ bounded and Lipschitz, $\nabla J$ bounded, and, additionally, $\bsig\bsig^{\!\top}\succeq\lambda_0I$ uniformly. Let $p>\max(N,2)$ and let $h\ge0$ be measurable with $M_h:=\sup_{t\le T}\|h(t,\cdot)\|_{L^p_{\mathrm{uloc}}}<\infty$. Then for all $\omega\ge1$,
\begin{equation}
    \Expec\int_0^Th\big(t,\x_\omega(t)\big)\,dt\;\le\;C\big(1+\omega^{N/(2p)}\big)\,M_h,
    \label{eq:supp-micro}
\end{equation}
with $C=C(N,p,\lambda_0,T,C_0,L,\alpha,r_{\max})$ independent of $\omega$. Note $N/(2p)<\tfrac12$ if and only if $p>N$: the dither prefactor $\omega^{-1/2}$ of the averaging estimates beats the occupation growth exactly at the Morrey threshold.
\end{lemma}

\begin{proof}
Partition $[0,T]$ into $\lceil T\omega\rceil$ intervals $I_j=[t_j,t_{j+1}]$ of length $h_0\le1/\omega$ and condition on $\mathcal F_{t_j}$. Write $\boldsymbol\theta:=\bsig^{-1}\bmu_\omega$, so $\|\boldsymbol\theta\|^2\le\lambda_0^{-1}NC_0^2\big(1+\sqrt{\alpha r_{\max}\omega}\big)^2$ and hence
\begin{equation*}
    \tfrac12\int_{I_j}\|\boldsymbol\theta\|^2\,ds
    \;\le\;\tfrac12\lambda_0^{-1}NC_0^2\big(1+\sqrt{\alpha r_{\max}\omega}\big)^2h_0
    \;\le\;\Lambda_0:=C\lambda_0^{-1}NC_0^2\alpha r_{\max}
    \qquad\text{for }\omega\ge1,
\end{equation*}
uniformly in $\omega$: this is the parabolic scaling $K^2h_0=O(1)$ with $K\asymp\sqrt\omega$, $h_0=1/\omega$. Novikov's condition therefore holds conditionally on each $I_j$, and Girsanov's theorem defines $\mathbb Q_j\sim\Prob$ on $\mathcal F_{t_{j+1}}$ under which $\x$ solves the driftless non-degenerate SDE $d\x=\bsig(\x,s)\,d\tilde\w$ on $I_j$ from $\x_{t_j}$. With $Z_j:=d\Prob/d\mathbb Q_j$ one has $\Expec^{\mathbb Q_j}[Z_j^2\,|\,\mathcal F_{t_j}]\le e^{2\Lambda_0}$ (bounded $\boldsymbol\theta$ makes the Dol\'eans exponential a true martingale, and the standard second-moment bound applies). By conditional Cauchy--Schwarz and then Cauchy--Schwarz in $s$,
\begin{align*}
    \Expec\Big[\int_{I_j}h(s,\x_s)\,ds\;\Big|\;\mathcal F_{t_j}\Big]
    &\;\le\;e^{\Lambda_0}\Big(\Expec^{\mathbb Q_j}\Big[\Big(\int_{I_j}h\,ds\Big)^2\Big|\mathcal F_{t_j}\Big]\Big)^{1/2}\\
    &\;\le\;e^{\Lambda_0}h_0^{1/2}\Big(\Expec^{\mathbb Q_j}\Big[\int_{I_j}h^2\,ds\Big|\mathcal F_{t_j}\Big]\Big)^{1/2}.
\end{align*}
Under $\mathbb Q_j$ the process is a uniformly elliptic diffusion with bounded Lipschitz coefficients, whose transition density over the elapsed time $s-t_j$ from $\x_{t_j}$ obeys the classical Gaussian upper bound $p_{s-t_j}(\x_{t_j},\y)\le C_A(s-t_j)^{-N/2}e^{-c\|\y-\x_{t_j}\|^2/(s-t_j)}$ \citep{friedman1964parabolic}. Hence, with $r=p/2>1$ and $r'$ its conjugate, $\|p_{s-t_j}(\x_{t_j},\cdot)\|_{L^{r'}(B_1(\y_0))}\le C(s-t_j)^{-N/p}e^{-c\,\mathrm{dist}(\x_{t_j},B_1(\y_0))^2/(s-t_j)}$, and summing over a cover of $\R^N$ by unit balls (the Gaussian tails dominate the polynomially many relevant balls),
\begin{equation*}
    \Expec^{\mathbb Q_j}\big[h^2(s,\x_s)\,\big|\,\mathcal F_{t_j}\big]\le C\,(s-t_j)^{-N/p}M_h^2,
    \qquad
    \Expec^{\mathbb Q_j}\Big[\int_{I_j}h^2\,ds\;\Big|\;\mathcal F_{t_j}\Big]\le C\,M_h^2\,h_0^{\,1-N/p},
\end{equation*}
the time integral converging precisely because $N/p<1$. Combining, each interval contributes at most $CM_h\,h_0^{\,1-N/(2p)}$, and summing over $\lceil T\omega\rceil$ intervals gives $\Expec\int_0^Th(s,\x_\omega(s))\,ds\le CM_hT\,h_0^{-N/(2p)}\le C(1+\omega^{N/(2p)})M_h$.
\end{proof}

\subsection{Sobolev energies: proof of Theorem~\ref{thm:sobolev}}
\label{supp:rough-sobolev}

\medskip\noindent\textbf{Theorem~\ref{thm:sobolev} (Sobolev energies; restated).}\ \emph{Assume $\bsig\bsig^{\!\top}\succeq\lambda_0I$ uniformly with $\f,\bsig$ bounded and Lipschitz, and that $J(\cdot,t)\in W^{2,p}_{\mathrm{uloc}}(\R^N)$ uniformly in $t$ for some $p>\max(N,2)$, with $\nabla J$ and $\partial_tJ$ bounded; write $M_p:=\sup_t\|\nabla^2J(\cdot,t)\|_{L^p_{\mathrm{uloc}}}$. Then $\x_\omega\to\bar\x$ in probability, uniformly on $[0,T]$, with $\Expec\sup_{t\le T}\|\x_\omega-\bar\x\|_2^2\le C\,\omega^{-(p-N)/p}$, where $\bar\x$ is the unique strong solution of the averaged SDE \eqref{eq:supp-averaged}, whose drift is continuous ($\nabla J\in W^{1,p}_{\mathrm{uloc}}\hookrightarrow C^{0,1-N/p}$ by Morrey); the a.e.-defined Hessian $\nabla^2J$ enters only the proof, along the trajectories.}

\begin{proof}
\emph{(0).} The averaged drift $\f-\tfrac{k\alpha}2\bsig\bsig^{\!\top}\nabla J$ is bounded and measurable ($\nabla J$ is bounded by hypothesis and continuous by Morrey's embedding, $p>N$), so strong existence and uniqueness for \eqref{eq:supp-averaged} follow from Veretennikov's theorem \citep{veretennikov1981strong}; by Lemma~\ref{lem:krylov}, the a.e.-defined $\nabla^2J$ never enters ambiguously along $\bar\x$.

\emph{(i).} Uniformly-local Morrey gives the global modulus
\begin{equation}
\begin{gathered}
    \|\nabla J-\nabla J_\varepsilon\|_{L^\infty}\le C(N,p)\,M_p\,\varepsilon^{1-N/p},
    \qquad
    \|\nabla^2J_\varepsilon\|_{L^\infty}\le CM_p\,\varepsilon^{-N/p},\\
    \|\nabla^3J_\varepsilon\|_{L^\infty}\le CM_p\,\varepsilon^{-1-N/p},
\end{gathered}
    \label{eq:supp-sobolev-moduli}
\end{equation}
together with $\|J-J_\varepsilon\|_\infty\le\|\nabla J\|_\infty\varepsilon$ and $\|\nabla^2J_\varepsilon\|_{L^p_{\mathrm{uloc}}}\le CM_p$. (The first bound follows from $|\nabla J(\x)-\nabla J(\y)|\le C\|\nabla^2J\|_{L^p(B(\x,2\|\x-\y\|))}\|\x-\y\|^{1-N/p}$ for $\|\x-\y\|\le1$.)

\emph{(ii).} As in Proposition~\ref{prop:c11}, split $I_{2,n}=I^\varepsilon_{2,n}+M_{2,n}$ and run Steps~1--2 of Appendix~\ref{supp:proof-main} on $I^\varepsilon_{2,n}$ with correctors built from $J_\varepsilon$, which is smooth with the bounds \eqref{eq:supp-sobolev-moduli}. All estimates are now pathwise in the mollified quantities: the boundary, Stieltjes ($\partial_tu$), $\f$-drift, and stochastic-integral constituents of $E^{(1)}_n$ are $O(\omega^{-1/2})$ with $\varepsilon$-uniform constants; the Hessian (It\^o-correction) constituent is $O(\omega^{-1/2}\varepsilon^{-N/p})$; the non-resonant family $E^{(2)}_n$, after the decomposition $\cos kJ=\cos kJ_\varepsilon+[\cos kJ-\cos kJ_\varepsilon]$ of its mixed slow factors, contributes $O\big(\omega^{-1/2}\varepsilon^{-N/p}+c_r\,\omega^{-1}\varepsilon^{-1-N/p}\big)$ from the all-$\varepsilon$ parts via Lemma~\ref{lem:osc-moll} (whose It\^o-correction term now carries $\mathrm{Lip}(\nabla\Phi^\varepsilon)\lesssim\varepsilon^{-1-N/p}$) and $O(\varepsilon)$ from the small difference parts. The surviving resonant means give, exactly as in \eqref{eq:supp-mixed-resonant},
\begin{equation*}
    -\tfrac{k\alpha}2\,g_n^2\,\partial_nJ_\varepsilon(\x_\omega)\cos\!\big(k(J-J_\varepsilon)\big)+O(\varepsilon)
    \;=\;-\tfrac{k\alpha}2\,g_n^2\,\partial_nJ(\x_\omega)+O\big(\varepsilon+M_p\varepsilon^{1-N/p}\big)
\end{equation*}
pathwise, the replacement $\partial J_\varepsilon\to\partial J$ costing $M_p\varepsilon^{1-N/p}$ globally by \eqref{eq:supp-sobolev-moduli}---no occupation estimate is needed for this step.

\emph{(iii).} Write $M_{2,n}$ as in \eqref{eq:supp-mismatch}. Now $\|\psi_\varepsilon\|_\infty\le C|k|\varepsilon$ and $\|\nabla\psi_\varepsilon\|_\infty\le C|k|M_p\varepsilon^{1-N/p}$ by \eqref{eq:supp-sobolev-moduli}, but $\nabla\tilde u$ is no longer Lipschitz: its weak gradient contains $\nabla^2\Theta_\varepsilon\ni\nabla^2J\in L^p_{\mathrm{uloc}}$ only. Integrate by parts against the fast antiderivative as before; the It\^o expansion of $\tilde u$ along $\x_\omega$ is justified, for each fixed $\omega$, by the It\^o--Krylov formula \citep[Ch.~2]{krylov1980controlled} ($\tilde u(\cdot,t)\in W^{2,p}_{\mathrm{uloc}}$ by the chain rule, $\partial_t\tilde u$ bounded, diffusion non-degenerate, drift bounded for fixed $\omega$). The resulting terms: boundary $O(\varepsilon\omega^{-1/2})$; Stieltjes $O(\omega^{-1/2})$; the drift term $O(\|\nabla\tilde u\|_\infty)=O\big(\varepsilon+M_p\varepsilon^{1-N/p}\big)$; the martingale term $O(\varepsilon\omega^{-1/2})$ in mean-square supremum; and the second-order term, the only one requiring the occupation estimate,
\begin{equation}
\begin{split}
    \sqrt{\alpha\omega_n}&\cdot\frac{C}{\omega_n}\,
    \Expec\int_0^T\big|\operatorname{tr}\big(\bsig\bsig^{\!\top}\nabla^2\tilde u\big)\big|\big(\x_\omega(s),s\big)\,ds\\
    &\;\le\;C\,\omega^{-1/2}\,\Expec\int_0^T\big(1+|\nabla^2J|+|\nabla^2J_\varepsilon|\big)\big(\x_\omega(s),s\big)\,ds
    \;\le\;C\,M_p\,\omega^{-1/2+N/(2p)}
\end{split}
    \label{eq:supp-sobolev-occ}
\end{equation}
by Lemma~\ref{lem:micro} (mollification does not increase $\|\cdot\|_{L^p_{\mathrm{uloc}}}$), vanishing precisely because $p>N$.

\emph{(iv).} Choosing $\varepsilon=\omega^{-1/2}$, every pathwise contribution above is $O\big(\omega^{-(p-N)/(2p)}\big)$ (the slowest terms being $\omega^{-1/2}\varepsilon^{-N/p}$, $\omega^{-1}\varepsilon^{-1-N/p}$, $M_p\varepsilon^{1-N/p}$, and $\omega^{-1/2+N/(2p)}$, all equal to $\omega^{-(p-N)/(2p)}$ at this choice), and every martingale contribution is at least as small in mean-square supremum. Hence the representation \eqref{eq:supp-mild} holds, with $\tau_R\equiv T$, with $\bar\bmu$ containing the a.e.-defined $\nabla J$, and with $\Expec\sup_{t\le T}\|\mathbf E\|^2\le C\omega^{-(p-N)/p}$. The averaged drift is bounded and, by uniformly-local Morrey ($W^{1,p}_{\mathrm{uloc}}\hookrightarrow C^{0,1-N/p}$, $p>N$, applied to $\nabla J$), globally $(1{-}N/p)$-H\"older in $\x$ uniformly in $t$---but not necessarily Lipschitz---so the Gr\"onwall step of Step~3 is replaced by a Zvonkin-transformed comparison \citep{zvonkin1974transformation,veretennikov1981strong}. The H\"older modulus is what makes the transformation quantitatively usable here. Because the coefficients depend on $t$, the transformation is built from the \emph{backward parabolic} Cauchy problem rather than an elliptic resolvent \citep{zvonkin1974transformation,veretennikov1981strong}: for $\varrho$ large, let $\mathbf u$ solve $\partial_t\mathbf u+\bar{\mathcal L}_t\mathbf u-\varrho\,\mathbf u=-\bar\bmu$ on $[0,T]\times\R^N$ with $\mathbf u(\cdot,T)=\mathbf 0$, where $\bar{\mathcal L}_t$ is the averaged generator at time $t$. Its drift $\bar\bmu$ is jointly H\"older: $(1{-}N/p)$-H\"older in $\x$ by Morrey, and H\"older in $t$ with exponent at least $(1-N/p)/2$ by interpolation between the boundedness of $\partial_tJ$ and the spatial modulus (mollify $J$ at scale $\varepsilon$ and optimize: $|\nabla J(\x,t)-\nabla J(\x,s)|\le C\varepsilon^{1-N/p}+C|t-s|/\varepsilon$). Parabolic Schauder theory \citep{friedman1964parabolic} therefore gives $\mathbf u\in C^{2+\beta,1+\beta/2}$ with $\|\nabla\mathbf u\|_\infty\le\tfrac12$ once $\varrho$ is large, so $\Phi(\cdot,t):=\mathrm{id}+\mathbf u(\cdot,t)$ is, for each $t$, a global bi-Lipschitz diffeomorphism with $\nabla\Phi,\nabla^2\Phi$ bounded---a conclusion that would fail for merely bounded measurable drift, where Calder\'on--Zygmund theory yields only $\mathbf u\in W^{2,r}_{\mathrm{uloc}}$---and by It\^o's formula the time derivative $\partial_t\mathbf u$ is exactly compensated: $d\Phi(\bar\x,t)=\varrho\,\mathbf u(\bar\x,t)\,dt+(I+\nabla\mathbf u)\bsig(\bar\x,t)\,d\w$, so that in the coordinates $\y=\Phi(\x,t)$ the averaged equation has Lipschitz drift $\varrho\,\mathbf u\circ\Phi^{-1}$ and Lipschitz diffusion $(I+\nabla\mathbf u)\bsig\circ\Phi^{-1}$. (An elliptic resolvent $\varrho\mathbf u-\bar{\mathcal L}_t\mathbf u=\bar\bmu$ solved separately for each $t$ would leave an uncompensated $\partial_t\mathbf u\,dt$ in this expansion, controlled only if $\nabla\partial_tJ$ were bounded, which is not assumed; for autonomous $J$ the two formulations coincide.) Both processes are continuous semimartingales, $\x_\omega$ by \eqref{eq:supp-sde} and, equivalently through the representation \eqref{eq:supp-mild}, as the averaged dynamics along its own trajectory plus $\mathbf E$, where $\mathbf E$ collects the finite-variation families and the stochastic-integral family with integrands of size $O(\omega^{-1/2})$. It\^o's formula applied to $\Phi(\x_\omega)$ and $\Phi(\bar\x)$ therefore gives, for the difference, the transformed Lipschitz comparison plus remainders of two kinds: $\nabla\Phi$ acting on $d\mathbf E$, controlled by $\|\nabla\Phi\|_\infty^2\,\Expec\sup\|\mathbf E\|^2$; and second-order cross-variation corrections, in which only the martingale part of $\mathbf E$ participates: its quadratic variation has integrands of size $O(\omega^{-1})$, and its cross-variation with $\bsig\,d\w$ has integrands of size $O(\omega^{-1/2})$ which oscillate at the dither frequencies, so that their time integrals are $O(\omega^{-1})$ by Lemma~\ref{lem:osc-moll}; even without exploiting the oscillation these corrections are bounded by $\|\nabla^2\Phi\|_\infty\,O(\omega^{-1/2})$, which is already smaller than the target rate $\omega^{-(p-N)/(2p)}$. Gr\"onwall--Burkholder--Davis--Gundy in the $\y$-coordinates, followed by the bi-Lipschitz transfer back through $\Phi^{-1}$, yields $\Expec\sup_{t\le T}\|\x_\omega-\bar\x\|^2\le C\,\omega^{-(p-N)/p}$. Convergence in probability follows by Chebyshev.
\end{proof}

Theorem~\ref{thm:sobolev} covers energies whose Hessian is unbounded but $p$-integrable, $p>N$: energies with isolated curvature blow-ups of order $\|\x-\x_0\|^{-\gamma}$, $\gamma<N/p$, and, more generally, any $J$ whose only failure of $C^2$ occurs on a null set across which $\nabla J$ stays continuous. The rate degrades continuously to zero as $p\downarrow N$, the Morrey threshold at which $\nabla J$ can first fail to be continuous.

\subsection{Lipschitz energies and gradient kinks}
\label{supp:rough-lip}

When the gradient itself jumps---kinks, as in $\|\cdot\|_1$ or ReLU-type energies, and the mean-absolute-difference measurement terms of the tracking energy \eqref{eq:supp-track-J}---$J$ is no longer in $W^{2,p}$ ($\nabla^2J$ acquires a singular surface part), and neither proof above applies: $\|\nabla J-\nabla J_\varepsilon\|_\infty$ no longer vanishes, and $\|\nabla^2J_\varepsilon\|_{L^p_{\mathrm{uloc}}}\asymp\varepsilon^{1/p-1}$ blows up. Well-posedness of the limit nevertheless survives in full, and we isolate exactly what the convergence proof still needs---an occupation estimate for the dithered process that is uniform in $\omega$---which Appendix~\ref{supp:rough-occ} then proves.

\begin{proposition}[Lipschitz energies]\label{prop:lip}
Assume $\bsig\bsig^{\!\top}\succeq\lambda_0I$ uniformly with $\f,\bsig$ bounded and Lipschitz, and that $J(\cdot,t)$ is globally Lipschitz uniformly in $t$ ($J$ need not be bounded). By Rademacher's theorem \citep{evans1992measure}, $\nabla J$ exists Lebesgue-a.e.\ with $\|\nabla J\|_{L^\infty}\le L_J$. Then: (a)~the averaged SDE \eqref{eq:supp-averaged} has a unique strong solution \citep{veretennikov1981strong,zvonkin1974transformation}, and by Lemma~\ref{lem:krylov} it does not depend on the values assigned to $\nabla J$ on its null exceptional set; (b)~if, in addition, Conjecture~\ref{conj:occ} below holds---by Theorem~\ref{thm:occ} this is the case whenever $J$ is also Lipschitz in $t$ uniformly in $\x$ and $\f,\bsig$ satisfy hypotheses (H1)--(H2) of Appendix~\ref{supp:rough-occ}---then $\x_\omega\Rightarrow\bar\x$ as $\omega\to\infty$.
\end{proposition}

\begin{conjecture}[$\omega$-uniform occupation estimate]\label{conj:occ}
Under the hypotheses of Proposition~\ref{prop:lip}, the estimate \eqref{eq:supp-micro} holds with exponent zero: $\Expec\int_0^Th(t,\x_\omega(t))\,dt\le C\,M_h$ uniformly in $\omega\ge1$, for $h\in L^\infty_tL^p_{\mathrm{uloc}}$, $p>\max(N,2)$.
\end{conjecture}

Conjecture~\ref{conj:occ} asserts that the occupation measure of the dithered process is no rougher, uniformly in $\omega$, than that of a bounded-drift diffusion. It is what the averaging limit itself predicts (the occupation measure of $\x_\omega$ converges to that of $\bar\x$, which satisfies Lemma~\ref{lem:krylov}), and it holds with the harmless $\omega^{N/(2p)}$ correction by Lemma~\ref{lem:micro}. Removing the correction without presupposing the averaging one is trying to prove requires taking the dither out of the drift \emph{before} Girsanov's theorem is applied, which Appendix~\ref{supp:rough-occ} does by a pathwise near-identity change of variables: Theorem~\ref{thm:occ} proves Conjecture~\ref{conj:occ} whenever $J$ is Lipschitz in $t$ (the time regularity that Theorem~\ref{thm:sobolev} also carries) and $\partial_tg_n$ is bounded, and Corollary~\ref{cor:tight} supplies the $\omega$-uniform tightness that the localization below uses. Only the case of merely bounded-variation time dependence remains open, and only because of the time regularity, not because of the kinks (Remark~\ref{rem:occ-scope}(i)); we keep the statement here, in the form in which the two-scale argument consumes it.

\begin{proof}[Proof sketch of Proposition~\ref{prop:lip}(b)]
Mollify: $J_\varepsilon=J*\rho_\varepsilon$ with $\|\nabla J_\varepsilon\|_\infty\le L_J$, $\nabla J_\varepsilon\to\nabla J$ Lebesgue-a.e., $\|\nabla^2J_\varepsilon\|_\infty\le CL_J/\varepsilon$, and run the two-stream architecture of Proposition~\ref{prop:c11} / Theorem~\ref{thm:sobolev}. At no point is the crude bound $\sqrt\omega\,\|J-J_\varepsilon\|_\infty$ used: the dither mismatch \eqref{eq:supp-mismatch} is itself a zero-mean oscillation and is integrated by parts, so the mismatch enters only through (i)~remainders carrying $\varepsilon$-powers of the mollified norms, all of the form $\omega^{-1/2}\varepsilon^{-\gamma}$ ($\gamma\le2$) and vanishing along, e.g., $\varepsilon=\omega^{-1/8}$; and (ii)~drift-replacement and slow-factor terms proportional to $|\nabla J-\nabla J_\varepsilon|(\x_\omega(s),s)$ integrated in time. For (ii), $h_\varepsilon:=|\nabla J-\nabla J_\varepsilon|$ satisfies $h_\varepsilon\le2L_J$ and $h_\varepsilon\to0$ a.e.\ in $(\x,t)$; by Theorem~\ref{thm:occ}, whose constant does not depend on $\omega$, $\Expec\int_0^Th_\varepsilon(\x_\omega)\,ds\to0$ as $\varepsilon\to0$ uniformly in $\omega$ (localize to a ball around $\x_0$, where dominated convergence applies to the mixed norm $\|h_\varepsilon\mathbf 1_{B_R}\|_{p,q}$ with $q<\infty$, and control the complement by the $\omega$-uniform tightness of Corollary~\ref{cor:tight}; the details are in Remark~\ref{rem:occ-uniform}), so both limits $\varepsilon\to0$, $\omega\to\infty$ may be taken along the single sequence $\varepsilon=\omega^{-1/8}$. The identified limit drift is $-\tfrac{k\alpha}2\bsig\bsig^{\!\top}\nabla J$ with the a.e.\ gradient. The argument then closes at the level of laws rather than paths, because the limit equation has only bounded measurable drift, so that no pathwise Lipschitz comparison of the kind used in step~(iv) of Theorem~\ref{thm:sobolev} is available (for such drift the Zvonkin map satisfies only $\mathbf u\in W^{2,r}_{\mathrm{uloc}}$ for every $r<\infty$ \citep{veretennikov1981strong,krylov2005strong}, so $\nabla\Phi$ is merely H\"older and a Gr\"onwall--BDG comparison under a shared Brownian motion does not close in $N\ge2$ dimensions). Instead: (1)~\emph{tightness}: under the hypotheses (H1)--(H4) of Appendix~\ref{supp:rough-occ}, which are those under which Conjecture~\ref{conj:occ} is established, $\y=\Psi(\x_\omega,t)$ of \eqref{eq:supp-occ-y} is an It\^o process whose drift and diffusion are bounded uniformly in $\omega$, so its laws are tight on $C([0,T];\R^N)$ by the Kolmogorov--Chentsov moment criterion \citep{billingsley1999convergence}, and $\sup_t\|\x_\omega-\y\|\le C\omega^{-1/2}$ transfers tightness to $\{\mathrm{Law}(\x_\omega)\}_{\omega\ge\omega_0}$ (Corollary~\ref{cor:tight} is the quantitative form); (2)~\emph{identification}: for $\varphi\in C^\infty_c(\R^N)$, the two-stream expansion applied to $\varphi(\x_\omega)$ in place of $\x_\omega$---with slow factors $\partial_i\varphi\,g_i\cos kJ_\varepsilon$ and $\partial_i\varphi\,g_i\sin kJ_\varepsilon$---shows that $\varphi(\x_\omega(t))-\int_0^t\bar{\mathcal L}_s\varphi(\x_\omega(s))\,ds$ is a martingale up to an error vanishing in $L^1$ as $\omega\to\infty$, the oscillatory terms by Lemma~\ref{lem:osc-moll} and (i), the drift-replacement terms by (ii); hence every subsequential weak limit solves the martingale problem for the averaged generator $\bar{\mathcal L}$ with drift $-\tfrac{k\alpha}2\bsig\bsig^{\!\top}\nabla J$; (3)~\emph{uniqueness}: the martingale problem for a uniformly elliptic diffusion with bounded measurable drift and Lipschitz diffusion coefficient is well posed \citep{stroock1979multidimensional,veretennikov1981strong}, so every subsequential limit is $\mathrm{Law}(\bar\x)$, which proves $\x_\omega\Rightarrow\bar\x$. The occupation estimate thus enters the Lipschitz rung through step~(2) alone, in the drift-replacement terms and in the mollified-corrector remainders; Theorem~\ref{thm:occ} supplies it with any $p>\max(N,2)$ and any finite $q>2p/(p-N)$. (With Lemma~\ref{lem:micro} in place of the conjecture, the argument delivers everything except the uniformity of step~(ii), whose bound degrades to $C\omega^{N/(2p)}\|h_\varepsilon\|$; since $\|h_\varepsilon\|$ carries no rate for merely Lipschitz $J$, the two limits cannot then be coupled.) Equivalently, one may read $\bar\x$ as a Filippov solution of the differential inclusion $d\bar\x\in\big[\f-\tfrac{k\alpha}2\bsig\bsig^{\!\top}\partial_CJ\big]dt+\bsig\,d\w$ with $\partial_CJ$ the Clarke gradient \citep{clarke1983nonsmooth,filippov1988discontinuous}; the non-degenerate noise collapses the inclusion to the a.e.-gradient SDE.
\end{proof}

\begin{remark}[Why Lemma~\ref{lem:micro} does not suffice, even for structured kinks]\label{rem:whynot}
One might hope to bypass Conjecture~\ref{conj:occ} for the practically relevant case in which the kink set of $J$ is a fixed $(N{-}1)$-rectifiable hypersurface of finite surface measure: there $h_\varepsilon=|\nabla J-\nabla J_\varepsilon|$ is supported in an $\varepsilon$-collar of the surface, so $\|h_\varepsilon\|_{L^\infty_tL^p_{\mathrm{uloc}}}\le CL_J\varepsilon^{1/p}$ carries a genuine rate. Feeding this into Lemma~\ref{lem:micro} bounds the step-(ii) terms by $C\omega^{N/(2p)}\varepsilon^{1/p}$, which vanishes only if $\varepsilon\ll\omega^{-N/2}$. But the mollified-corrector remainders of step~(i) contain $\omega^{-1/2}\varepsilon^{-1}$ and $\omega^{-1}\varepsilon^{-2}$ and hence require $\varepsilon\gg\omega^{-1/2}$. The two windows $\varepsilon\ll\omega^{-N/2}$ and $\varepsilon\gg\omega^{-1/2}$ are disjoint for every $N\ge1$. The $\omega^{N/(2p)}$ loss in Lemma~\ref{lem:micro} is therefore not a slack constant one can trade against geometric structure in $J$; removing it---Conjecture~\ref{conj:occ}---is the genuine obstruction. Appendix~\ref{supp:rough-occ} removes it by changing the order of operations rather than the scale: the dither is taken out of the drift by a pathwise change of variables first, at its own mollification scale $\varepsilon=\omega^{-1/2}$, and Girsanov's theorem is applied to what remains, whose drift is bounded uniformly in $\omega$.
\end{remark}

\subsection{The $\omega$-uniform occupation estimate: proof of Conjecture~\ref{conj:occ}}
\label{supp:rough-occ}

Lemma~\ref{lem:micro} loses the factor $\omega^{N/(2p)}$ because it applies Girsanov's theorem to the dither drift itself: that drift has size $\sqrt\omega$, so Novikov's exponent stays bounded only on windows of length $1/\omega$, and $\lceil T\omega\rceil$ windows have to be summed. This subsection removes the loss by removing the dither first. The device is the near-identity change of variables of the deterministic warm-up, \eqref{eq:supp-warmup-ydef}---the classical averaging transformation \citep{bogoliubov1961asymptotic,sanders2007averaging}, in its stochastic form \citep{freidlin2012random,pavliotis2008multiscale}---applied pathwise to the It\^o process: subtracting the phase-locked ripple $g_n\sqrt{\alpha/\omega_n}\,\sin(\omega_nt+kJ)$ from each coordinate produces a process whose drift is bounded uniformly in $\omega$ and whose diffusion coefficient is an $O(\omega^{-1/2})$ perturbation of $\bsig$. The classical occupation estimate then applies to the transformed process with $\omega$-uniform constants, and the transformation is bi-Lipschitz with $\omega$-uniform constants, so the estimate transfers back to $\x_\omega$. What makes this work where Remark~\ref{rem:whynot} found an obstruction is a difference of purpose. An occupation estimate needs the transformed coefficients to be \emph{bounded}, not small, so the mollification scale inside the transformation can be tied to the dither scale, $\varepsilon=\omega^{-1/2}$, exactly where the dither mismatch and the It\^o-correction term are both $O(1)$. A convergence proof needs those same terms to \emph{vanish}, which forces $\varepsilon\gg\omega^{-1/2}$. The two requirements are incompatible for a single scale, which is the disjoint-windows obstruction of Remark~\ref{rem:whynot}; the resolution is that the occupation estimate carries its own scale.

Throughout this subsection we assume:
\begin{itemize}
    \item[(H1)] $\f$ is bounded and Lipschitz in $\x$; each $g_n$ is twice continuously differentiable in $\x$, with $g_n$ and its first two $\x$-derivatives bounded by $C_0$ and Lipschitz in $\x$ with constant $L$, and continuously differentiable in $t$ with $|\partial_tg_n|\le C_0$ (this is Assumption~\ref{ass:glob}(i)--(ii) for $\f,\bsig$, with (ii) in its continuously-differentiable form);
    \item[(H2)] $\bsig\bsig^{\!\top}\succeq\lambda_0I$ for some $\lambda_0>0$;
    \item[(H3)] $J$ is globally Lipschitz in $\x$ uniformly in $t$, and in $t$ uniformly in $\x$: $\|\nabla J\|_{L^\infty}\le L_J$ and $\|\partial_tJ\|_{L^\infty}\le L_t$ (the derivatives exist a.e.\ by Rademacher's theorem and are the weak derivatives);
    \item[(H4)] Assumption~\ref{ass:freq}, normalized so that $\min_nr_n=1$ (Assumption~\ref{ass:glob}(iii)).
\end{itemize}
(H3) is the hypothesis of Proposition~\ref{prop:lip} together with the Lipschitz-in-time regularity that Theorem~\ref{thm:sobolev} also carries; Remark~\ref{rem:occ-scope} explains why merely bounded variation in $t$ does not suffice for this argument. Constants $C$ below depend only on $(N,p,q,\lambda_0,T,C_0,L,\alpha,|k|,r_{\max},L_J,L_t)$ and never on $\omega$ or on $h$.

\begin{theorem}[$\omega$-uniform occupation estimate]\label{thm:occ}
Let (H1)--(H4) hold and let $\x_\omega$ solve \eqref{eq:supp-sde} from an $\mathcal F_0$-measurable $\x_0$. Let $p>\max(N,2)$ and $q\in\big(\tfrac{2p}{p-N},\infty\big]$, and for measurable $h\ge0$ on $[0,T]\times\R^N$ write
\begin{equation*}
    \|h\|_{p,q}\;:=\;\Big\|\,\|h(t,\cdot)\|_{L^p_{\mathrm{uloc}}}\Big\|_{L^q(0,T)},
    \qquad\text{so that }\|h\|_{p,\infty}=M_h\text{ in the notation of Lemma~\ref{lem:micro}}.
\end{equation*}
Then there is a constant $C$, independent of $\omega$ and of $h$, such that for all $\omega\ge1$,
\begin{equation}
    \Expec\int_0^Th\big(t,\x_\omega(t)\big)\,dt\;\le\;C\,\|h\|_{p,q}.
    \label{eq:supp-occ}
\end{equation}
In particular, Conjecture~\ref{conj:occ} holds under the hypotheses of Proposition~\ref{prop:lip} strengthened by (H1) and (H3).
\end{theorem}

\begin{proof}
\emph{Step 1: mollification at the dither scale.}
Fix $\omega\ge1$ and set $\varepsilon:=\omega^{-1/2}\le1$. Let $\rho\in C_c^\infty(B_1)$ be a nonnegative mollifier with $\int\rho=1$, $\rho_\varepsilon(\mathbf z):=\varepsilon^{-N}\rho(\mathbf z/\varepsilon)$, $c_\rho:=\max_j\|\partial_j\rho\|_{L^1}$, and $J_\varepsilon:=J*_\x\rho_\varepsilon$. Since $\nabla J_\varepsilon=(\nabla J)*_\x\rho_\varepsilon$, $\partial_{ij}J_\varepsilon=(\partial_iJ)*_\x\partial_j\rho_\varepsilon$, $\partial_tJ_\varepsilon=(\partial_tJ)*_\x\rho_\varepsilon$, and $J_\varepsilon(\x,t)-J(\x,t)=\int\big(J(\x-\mathbf z,t)-J(\x,t)\big)\rho_\varepsilon(\mathbf z)\,d\mathbf z$,
\begin{equation}
    \|J-J_\varepsilon\|_\infty\le L_J\varepsilon,\qquad
    \|\nabla J_\varepsilon\|_\infty\le L_J,\qquad
    \|\nabla^2J_\varepsilon\|_\infty\le c_\rho L_J\,\varepsilon^{-1},\qquad
    \|\partial_tJ_\varepsilon\|_\infty\le L_t.
    \label{eq:supp-occ-moll}
\end{equation}
Note that $J_\varepsilon$ is $C^\infty$ in $\x$, Lipschitz in $t$, and that $\nabla^2J_\varepsilon=J*_\x\nabla^2\rho_\varepsilon$ is continuous in $(\x,t)$.

\emph{Step 2: the transformation and its bounds.}
Write $\theta_n(\x,t):=\omega_nt+kJ(\x,t)$ for the dither phase actually driving \eqref{eq:supp-drift}, and $\theta^\varepsilon_n:=\omega_nt+kJ_\varepsilon(\x,t)$ for its mollified version. Define
\begin{equation}
    v_n(\x,t):=g_n(\x,t)\sqrt{\alpha/\omega_n}\,\sin\theta^\varepsilon_n(\x,t),
    \qquad
    \Psi(\x,t):=\x-\mathbf v(\x,t),
    \qquad n=1,\dots,N,
    \label{eq:supp-occ-transform}
\end{equation}
the stochastic counterpart of \eqref{eq:supp-warmup-ydef}, with the phase mollified. By (H1) and \eqref{eq:supp-occ-moll}, $\mathbf v$ is $C^2$ in $\x$ with $\nabla^2_\x\mathbf v$ continuous, and Lipschitz in $t$. Since $\partial_iv_n=\sqrt{\alpha/\omega_n}\,\big[\partial_ig_n\sin\theta^\varepsilon_n+kg_n\cos\theta^\varepsilon_n\,\partial_iJ_\varepsilon\big]$, using $\omega_n\ge\omega\ge1$ and $\sqrt{\alpha/\omega}\,\varepsilon^{-1}=\sqrt\alpha$ one obtains, uniformly on $\R^N\times[0,T]$,
\begin{align}
    |v_n|&\le C_0\sqrt{\alpha/\omega},
    \qquad
    |\partial_iv_n|\le\sqrt{\alpha/\omega}\,C_0(1+|k|L_J),
    \label{eq:supp-occ-v1}\\
    |\partial_{ij}v_n|&\le\sqrt{\alpha/\omega_n}\Big[|\partial_{ij}g_n|+|k|\big(|\partial_ig_n\partial_jJ_\varepsilon|+|\partial_jg_n\partial_iJ_\varepsilon|\big)
    +|k|\,|g_n\partial_{ij}J_\varepsilon|+k^2|g_n\partial_iJ_\varepsilon\partial_jJ_\varepsilon|\Big]
    \nonumber\\
    &\le\sqrt\alpha\,C_0\big[1+2|k|L_J+k^2L_J^2+|k|c_\rho L_J\big]=:c_2,
    \label{eq:supp-occ-v2}\\
    \partial_tv_n&=g_n\sqrt{\alpha\omega_n}\,\cos\theta^\varepsilon_n+\varrho_n,
    \nonumber\\
    \varrho_n&:=\sqrt{\alpha/\omega_n}\big[\partial_tg_n\sin\theta^\varepsilon_n+kg_n\,\partial_tJ_\varepsilon\cos\theta^\varepsilon_n\big],
    \qquad|\varrho_n|\le\sqrt\alpha\,C_0(1+|k|L_t),
    \label{eq:supp-occ-vt}
\end{align}
where in \eqref{eq:supp-occ-v2} the It\^o-critical term $|k|\,|g_n\partial_{ij}J_\varepsilon|\le|k|C_0c_\rho L_J\varepsilon^{-1}$ has been multiplied by $\sqrt{\alpha/\omega_n}\le\sqrt{\alpha/\omega}=\sqrt\alpha\,\varepsilon$: this is the only place where the choice $\varepsilon=\omega^{-1/2}$ is used, and it makes the second derivatives of the transformation bounded rather than small. In \eqref{eq:supp-occ-vt}, the leading term arises from $\partial_t\theta^\varepsilon_n=\omega_n+k\partial_tJ_\varepsilon$ and is the dither drift of \eqref{eq:supp-drift} with mollified phase. Let $\nabla\mathbf v:=(\partial_iv_n)_{n,i}$. By \eqref{eq:supp-occ-v1}--\eqref{eq:supp-occ-v2},
\begin{equation}
\begin{aligned}
    \|\nabla\mathbf v(\x,t)\|_{\mathrm{op}}&\le\|\nabla\mathbf v(\x,t)\|_F\le c_1\,\omega^{-1/2},
    &\qquad c_1&:=NC_0\sqrt\alpha\,(1+|k|L_J);\\
    \|\nabla\mathbf v(\x,t)-\nabla\mathbf v(\x',t)\|_{\mathrm{op}}&\le c_3\|\x-\x'\|,
    &\qquad c_3&:=N^2c_2.
\end{aligned}
    \label{eq:supp-occ-gradv}
\end{equation}

\emph{Step 3: $\Psi(\cdot,t)$ is a bi-Lipschitz diffeomorphism, uniformly in $\omega$.}
Set $\omega_0:=\max\{1,4c_1^2\}$ and let $\omega\ge\omega_0$, so that $\|\nabla\mathbf v\|_{\mathrm{op}}\le\tfrac12$ and $\mathbf v(\cdot,t)$ is $\tfrac12$-Lipschitz. Then for all $\x,\x'$,
\begin{equation}
    \tfrac12\|\x-\x'\|\;\le\;\|\Psi(\x,t)-\Psi(\x',t)\|\;\le\;\tfrac32\|\x-\x'\|,
    \label{eq:supp-occ-bilip}
\end{equation}
so $\Psi(\cdot,t)$ is injective; it is also surjective, because for every $\y\in\R^N$ the map $\x\mapsto\y+\mathbf v(\x,t)$ is a contraction of $\R^N$ whose unique fixed point $\x$ satisfies $\Psi(\x,t)=\y$. Its inverse $\Xi_t:=\Psi(\cdot,t)^{-1}$ is therefore globally defined, $2$-Lipschitz, and $C^1$ by the inverse function theorem ($\nabla_\x\Psi=I-\nabla\mathbf v$ is invertible). The singular values of $I-\nabla\mathbf v$ lie in $[\tfrac12,\tfrac32]$, whence
\begin{equation}
    2^{-N}\le|\det\nabla_\x\Psi(\x,t)|\le(\tfrac32)^N,
    \qquad
    |\det\nabla_\y\Xi_t(\y)|\le2^N.
    \label{eq:supp-occ-jac}
\end{equation}

\emph{Step 4: It\^o's formula---the dither cancels.}
Let $\y(t):=\Psi(\x_\omega(t),t)$. Since $\Psi$ is $C^2$ in $\x$ with continuous second derivatives and Lipschitz in $t$, and since for fixed $\omega$ the process $\x_\omega$ has bounded drift and non-degenerate diffusion, It\^o's formula applies with the $t$-derivative entering as a Lebesgue integral (the It\^o--Krylov formula, \citealp[Ch.~2]{krylov1980controlled}, exactly as invoked in the proof of Theorem~\ref{thm:sobolev}). With $\bsig$ diagonal,
\begin{equation*}
    dy_n=dx_{\omega,n}-\partial_tv_n\,dt-\sum_{i}\partial_iv_n\,dx_{\omega,i}-\tfrac12\sum_ig_i^2\,\partial_{ii}v_n\,dt .
\end{equation*}
Insert $dx_{\omega,i}=\big[f_i+g_i\sqrt{\alpha\omega_i}\cos\theta_i\big]dt+g_i\,dw_i$ and \eqref{eq:supp-occ-vt}: the dither drift $g_n\sqrt{\alpha\omega_n}\cos\theta_n$ of $dx_{\omega,n}$ and the leading term of $\partial_tv_n$ cancel up to the mismatch between $\theta_n$ and $\theta^\varepsilon_n$, and
\begin{equation}
\begin{aligned}
    dy_n&=b_n(t)\,dt+\sum_i\big(\delta_{ni}-\partial_iv_n\big)g_i\,dw_i,\\
    b_n&:=f_n+g_n\sqrt{\alpha\omega_n}\big[\cos\theta_n-\cos\theta^\varepsilon_n\big]-\varrho_n
    -\sum_i\partial_iv_n\big[f_i+g_i\sqrt{\alpha\omega_i}\cos\theta_i\big]-\tfrac12\sum_ig_i^2\,\partial_{ii}v_n,
\end{aligned}
    \label{eq:supp-occ-y}
\end{equation}
all coefficients evaluated at $(\x_\omega(t),t)$. Each term of $b_n$ is bounded uniformly in $\omega\ge1$: by \eqref{eq:supp-occ-moll}, $|\cos\theta_n-\cos\theta^\varepsilon_n|\le|k|\,|J-J_\varepsilon|\le|k|L_J\varepsilon$, so the mismatch term is at most $C_0\sqrt{\alpha r_{\max}\omega}\,|k|L_J\varepsilon=C_0\sqrt{\alpha r_{\max}}\,|k|L_J$ (here $\sqrt\omega\,\varepsilon=1$ is used a second time); by \eqref{eq:supp-occ-v1} the cross terms with the base drift and with the dither are at most $N\sqrt\alpha\,C_0^2(1+|k|L_J)$ and $N\sqrt{\alpha/\omega}\,C_0(1+|k|L_J)\cdot C_0\sqrt{\alpha r_{\max}\omega}=N\alpha\sqrt{r_{\max}}\,C_0^2(1+|k|L_J)$ respectively; and by \eqref{eq:supp-occ-v2} the It\^o-correction term is at most $\tfrac12NC_0^2c_2$. Hence
\begin{equation}
\begin{aligned}
    |b_n(t)|\le B&:=C_0\big[1+\sqrt{\alpha r_{\max}}\,|k|L_J+\sqrt\alpha(1+|k|L_t)\big]\\
    &\qquad+NC_0^2\big[\sqrt\alpha\,(1+\sqrt{\alpha r_{\max}})(1+|k|L_J)+\tfrac12c_2\big].
\end{aligned}
    \label{eq:supp-occ-B}
\end{equation}
The diffusion matrix of \eqref{eq:supp-occ-y} is $\Sigma_\omega(t):=(I-\nabla\mathbf v)\bsig$ at $(\x_\omega(t),t)$; for $\omega\ge\omega_0$, by (H2) and \eqref{eq:supp-occ-bilip},
\begin{equation}
    \Sigma_\omega\Sigma_\omega^{\!\top}=(I-\nabla\mathbf v)\bsig\bsig^{\!\top}(I-\nabla\mathbf v)^{\!\top}\succeq\lambda_0(I-\nabla\mathbf v)(I-\nabla\mathbf v)^{\!\top}\succeq\tfrac{\lambda_0}4I,
    \qquad
    \|\Sigma_\omega\|_{\mathrm{op}}\le\tfrac32C_0 .
    \label{eq:supp-occ-ellip}
\end{equation}
The transformed process thus has a drift bounded by $B$ and a diffusion coefficient that is bounded and uniformly elliptic, all independently of $\omega\ge\omega_0$. This is the whole point: the $O(\sqrt\omega)$ dither has been traded for $O(1)$ terms, at the price of an $O(\omega^{-1/2})$ deformation of the diffusion.

\emph{Step 5: Markov form and Girsanov.}
Since $\x_\omega(t)=\Xi_t(\y(t))$, the coefficients of \eqref{eq:supp-occ-y} are functions of $(\y(t),t)$: $b_n(t)=\tilde b_n(\y(t),t)$ and $\Sigma_\omega(t)=\tilde\Sigma(\y(t),t)$, where $\tilde{\mathbf b}(\y,t)$ is the expression in \eqref{eq:supp-occ-y} evaluated at $(\Xi_t(\y),t)$ and $\tilde\Sigma(\y,t):=\big[(I-\nabla\mathbf v)\bsig\big](\Xi_t(\y),t)$. Then $\tilde{\mathbf b}$ is measurable and bounded by $B$, and $\tilde\Sigma$ is bounded by $\tfrac32C_0$, uniformly elliptic with constant $\lambda_0/4$, and Lipschitz in $\y$ uniformly in $t$ and in $\omega\ge\omega_0$, with $\mathrm{Lip}_\y(\tilde\Sigma)\le2\big[\tfrac32L+c_3C_0\big]=:\ell$ by \eqref{eq:supp-occ-gradv} and \eqref{eq:supp-occ-bilip}. Let $\boldsymbol\theta(t):=\Sigma_\omega(t)^{-1}\mathbf b(t)$, so that $\|\boldsymbol\theta\|\le2\sqrt N\,B/\sqrt{\lambda_0}$ by \eqref{eq:supp-occ-ellip}. Novikov's condition holds, and Girsanov's theorem \citep[Ch.~3]{karatzas1991brownian} provides $\mathbb Q\sim\Prob$ on $\mathcal F_T$, with $d\mathbb Q/d\Prob=\exp\big(-\int_0^T\boldsymbol\theta\cdot d\w-\tfrac12\int_0^T\|\boldsymbol\theta\|^2ds\big)$ and $\mathbb Q=\Prob$ on $\mathcal F_0$, under which $\tilde\w:=\w+\int_0^{\cdot}\boldsymbol\theta\,ds$ is an $(\mathcal F_t)$-Brownian motion and
\begin{equation*}
    d\y=\tilde\Sigma(\y,t)\,d\tilde\w,\qquad \y(0)=\y_0:=\Psi(\x_0,0).
\end{equation*}
Moreover $Z:=d\Prob/d\mathbb Q=\exp\big(\int_0^T\boldsymbol\theta\cdot d\tilde\w-\tfrac12\int_0^T\|\boldsymbol\theta\|^2ds\big)$ satisfies $\Expec^{\mathbb Q}[Z^2\,|\,\mathcal F_0]\le e^{2\Lambda}$ with $\Lambda:=2NB^2T/\lambda_0$, because $Z^2=\exp\big(\int2\boldsymbol\theta\cdot d\tilde\w-\tfrac12\int\|2\boldsymbol\theta\|^2ds\big)\exp\big(\int\|\boldsymbol\theta\|^2ds\big)$ and the first factor, a Dol\'eans exponential with bounded integrand, is a $\mathbb Q$-martingale. The constant $\Lambda$ is independent of $\omega$: this is where the argument differs from Lemma~\ref{lem:micro}, whose Novikov exponent grows like $\omega T$ over the full horizon.

\emph{Step 6: occupation estimate for the transformed process, and transfer.}
Let $\tilde h(t,\y):=h(t,\Xi_t(\y))$. (a)~\emph{The transformation preserves uniformly local integrability.} For any unit ball $B_1(\y_0)$, $\Xi_t(B_1(\y_0))\subset B_2(\Xi_t(\y_0))$, which is covered by $K_N$ unit balls, so by the change of variables $\x=\Xi_t(\y)$ and \eqref{eq:supp-occ-jac},
\begin{equation*}
    \int_{B_1(\y_0)}h\big(t,\Xi_t(\y)\big)^p\,d\y=\int_{\Xi_t(B_1(\y_0))}h(t,\x)^p\,|\det\nabla_\x\Psi(\x,t)|\,d\x\le(\tfrac32)^NK_N\,\|h(t,\cdot)\|^p_{L^p_{\mathrm{uloc}}},
\end{equation*}
i.e.\ $\|\tilde h(t,\cdot)\|_{L^p_{\mathrm{uloc}}}\le c_\Psi\|h(t,\cdot)\|_{L^p_{\mathrm{uloc}}}$ with $c_\Psi:=\big((\tfrac32)^NK_N\big)^{1/p}$.

(b)~\emph{Gaussian bound.} Under $\mathbb Q$, conditionally on $\mathcal F_0$, $\y$ is the (pathwise unique) diffusion $d\y=\tilde\Sigma(\y,t)\,d\tilde\w$ from $\y_0$, whose diffusion coefficient is bounded, uniformly elliptic and Lipschitz in space with constants independent of $\omega$. Its transition density therefore obeys the Gaussian upper bound $p_s(\y_0,\y)\le C_Gs^{-N/2}e^{-c_G\|\y-\y_0\|^2/s}$, $0<s\le T$, with $C_G,c_G$ depending only on $(N,\lambda_0,C_0,\ell,T)$ \citep{friedman1964parabolic}.\footnote{Only the ellipticity, the sup norm and the spatial Lipschitz (indeed H\"older) seminorm of $\tilde a:=\tilde\Sigma\tilde\Sigma^{\!\top}$ enter $C_G,c_G$: the parametrix may be built by freezing the coefficient in space only, i.e.\ from the Gaussian with covariance $\int_\tau^s\tilde a(\boldsymbol\xi,u)\,du$, whose residual $\tfrac12(\tilde a(\boldsymbol\xi,s)-\tilde a(\y,s))\!:\!\nabla^2$ is controlled by the spatial modulus of $\tilde a$ alone, so that no regularity in $s$ is consumed. This matters here because the modulus of $\tilde\Sigma$ in $t$ is of order $\sqrt\omega$. For $p\ge N+1$ one can avoid transition densities altogether: Krylov's estimate for It\^o processes with bounded coefficients and non-degenerate diffusion, $\Expec\int_0^Th(s,\y_s)\,ds\le C\|h\|_{L^{N+1}([0,T]\times\R^N)}$ \citep[Ch.~2]{krylov1980controlled}, requires no continuity of $\tilde\Sigma$ at all, and the uniformly-local norm is recovered from it by localizing to unit cubes and using the tail bound of Corollary~\ref{cor:tight}.} Let $\{B_1(\mathbf z_j)\}_j$ be a cover of $\R^N$ by unit balls with centres on a lattice, so that at most $c_N(1+m)^{N-1}$ of them satisfy $\mathrm{dist}(\y_0,B_1(\mathbf z_j))\in[m,m+1)$. With $r:=p/2>1$ and $r':=r/(r-1)$, H\"older's inequality on each ball gives
\begin{align*}
    \Expec^{\mathbb Q}\big[\tilde h^2(s,\y_s)\,\big|\,\mathcal F_0\big]
    &\le\sum_j\|\tilde h(s,\cdot)\|^2_{L^p(B_1(\mathbf z_j))}\,\|p_s(\y_0,\cdot)\|_{L^{r'}(B_1(\mathbf z_j))}\\
    &\le c_\Psi^2\,\|h(s,\cdot)\|^2_{L^p_{\mathrm{uloc}}}\sum_j\|p_s(\y_0,\cdot)\|_{L^{r'}(B_1(\mathbf z_j))}.
\end{align*}
For the boundedly many (at most $c_N'$) balls with $\mathrm{dist}(\y_0,B_1(\mathbf z_j))<2$ we use the global bound
\begin{equation*}
    \|p_s(\y_0,\cdot)\|_{L^{r'}(\R^N)}\le C_Gs^{-N/2}\Big(\int e^{-r'c_G\|\y\|^2/s}\,d\y\Big)^{1/r'}=Cs^{-N/2+N/(2r')}=Cs^{-N/p};
\end{equation*}
for the others, on which $\|\y-\y_0\|\ge m:=\lfloor\mathrm{dist}(\y_0,B_1(\mathbf z_j))\rfloor\ge2$, we use
\begin{equation*}
    \|p_s(\y_0,\cdot)\|_{L^{r'}(B_1(\mathbf z_j))}\le C_G|B_1|^{1/r'}s^{-N/2}e^{-c_Gm^2/s}
    \le C\,e^{-c_Gm^2/(2T)}\,m^{-N}\sup_{u>0}u^{N/2}e^{-c_Gu/2},
\end{equation*}
and $\sum_{m\ge2}c_N(1+m)^{N-1}m^{-N}e^{-c_Gm^2/(2T)}<\infty$.
Hence, for $0<s\le T$,
\begin{equation}
    \Expec^{\mathbb Q}\big[\tilde h^2(s,\y_s)\,\big|\,\mathcal F_0\big]\le C\big(s^{-N/p}+1\big)\,\|h(s,\cdot)\|^2_{L^p_{\mathrm{uloc}}}.
    \label{eq:supp-occ-onetime}
\end{equation}

(c)~\emph{Integration in time.} Let $a:=q/2\in\big(\tfrac p{p-N},\infty\big]$ and $a'$ its conjugate exponent, so that $a'N/p<1$ precisely when $q>2p/(p-N)$ (for $q=\infty$, $a'=1$ and the condition is $N/p<1$). By H\"older's inequality in $s$ applied to \eqref{eq:supp-occ-onetime},
\begin{equation}
    \Expec^{\mathbb Q}\Big[\int_0^T\tilde h^2(s,\y_s)\,ds\;\Big|\;\mathcal F_0\Big]
    \le C\,\big\|s^{-N/p}+1\big\|_{L^{a'}(0,T)}\,\Big\|\,\|h(s,\cdot)\|^2_{L^p_{\mathrm{uloc}}}\Big\|_{L^{a}(0,T)}
    =C'\,\|h\|_{p,q}^2,
    \label{eq:supp-occ-time}
\end{equation}
the first factor being finite because $a'N/p<1$: this is the only place where the lower bound on $q$ enters.

(d)~\emph{Transfer.} By the Cauchy--Schwarz inequality under $\mathbb Q$, conditionally on $\mathcal F_0$,
\begin{align*}
    \Expec\Big[\int_0^Th\big(s,\x_\omega(s)\big)ds\;\Big|\;\mathcal F_0\Big]
    &=\Expec^{\mathbb Q}\Big[Z\int_0^T\tilde h(s,\y_s)\,ds\;\Big|\;\mathcal F_0\Big]\\
    &\le\big(\Expec^{\mathbb Q}[Z^2\,|\,\mathcal F_0]\big)^{1/2}\,T^{1/2}\Big(\Expec^{\mathbb Q}\Big[\int_0^T\tilde h^2(s,\y_s)\,ds\;\Big|\;\mathcal F_0\Big]\Big)^{1/2}\\
    &\le e^{\Lambda}\,T^{1/2}\,(C')^{1/2}\,\|h\|_{p,q},
\end{align*}
and taking expectations proves \eqref{eq:supp-occ} for $\omega\ge\omega_0$. For $1\le\omega<\omega_0$, Steps~5--6 apply verbatim to $\x_\omega$ itself, with $\Psi=\mathrm{id}$ and no transformation: the drift of \eqref{eq:supp-sde} is then bounded by $C_0\big(1+\sqrt{\alpha r_{\max}\omega_0}\big)$, and the resulting constant depends on $\omega_0$, which depends only on the admissible parameters. Taking the larger of the two constants gives \eqref{eq:supp-occ} for all $\omega\ge1$.
\end{proof}

The transformation yields a second $\omega$-uniform statement for free, which the two-scale argument of Appendix~\ref{supp:rough-lip} also needs (Remark~\ref{rem:occ-uniform}).

\begin{corollary}[$\omega$-uniform exponential tightness of the dithered process]\label{cor:tight}
Under (H1)--(H4), with $B$ as in \eqref{eq:supp-occ-B} and $\omega_0$ as in Step~3, for all $\omega\ge\omega_0$ and all $R>0$,
\begin{equation}
    \Prob\Big(\sup_{t\le T}\|\x_\omega(t)-\x_0\|\ge R+\sqrt N\,BT+2\sqrt N\,C_0\sqrt\alpha\Big)\;\le\;2N\exp\!\Big(-\frac{2R^2}{9NC_0^2T}\Big).
    \label{eq:supp-tight}
\end{equation}
\end{corollary}

\begin{proof}
From \eqref{eq:supp-occ-y}, $\y(t)-\y_0=\int_0^t\mathbf b\,ds+\mathbf M_t$ with $\|\int_0^t\mathbf b\,ds\|\le\sqrt N\,BT$ and $\mathbf M_t=\int_0^t\Sigma_\omega\,d\w$. The $n$th coordinate $M^n$ is a continuous martingale with $\langle M^n\rangle_t\le\|\Sigma_\omega\|_{\mathrm{op}}^2t\le\tfrac94C_0^2t$ by \eqref{eq:supp-occ-ellip}, so the exponential martingale inequality \citep[Ch.~3]{karatzas1991brownian} gives $\Prob(\sup_{t\le T}|M^n_t|\ge R/\sqrt N)\le2\exp\big(-2R^2/(9NC_0^2T)\big)$. Since $\x_\omega(t)-\x_0=\y(t)-\y_0+\mathbf v(\x_\omega(t),t)-\mathbf v(\x_0,0)$ and $\|\mathbf v\|\le\sqrt N\,C_0\sqrt{\alpha/\omega}\le\sqrt N\,C_0\sqrt\alpha$ by \eqref{eq:supp-occ-v1}, we have $\sup_{t\le T}\|\x_\omega(t)-\x_0\|\le\sqrt N\,BT+2\sqrt N\,C_0\sqrt\alpha+\sqrt N\max_n\sup_{t\le T}|M^n_t|$, and a union bound over $n$ gives \eqref{eq:supp-tight}.
\end{proof}

\begin{remark}[Uniformity in the two-scale argument]\label{rem:occ-uniform}
Step~(ii) of the proof sketch of Proposition~\ref{prop:lip}(b) needs $\Expec\int_0^Th_\varepsilon(s,\x_\omega(s))\,ds\to0$ along $\varepsilon=\varepsilon(\omega)\to0$, where $h_\varepsilon:=|\nabla J-\nabla J_\varepsilon|$ satisfies $0\le h_\varepsilon\le2L_J$ and $h_\varepsilon\to0$ a.e.\ in $(\x,t)$ (at every Lebesgue point of $\nabla J(\cdot,t)$, hence a.e.\ by Fubini). Pointwise convergence alone does not make $\|h_\varepsilon\|_{L^\infty_tL^p_{\mathrm{uloc}}}$ vanish---the supremum over all unit balls of $\R^N$ and over $t$ need not commute with the limit if the kinks of $J$ accumulate at infinity or move erratically in time---and this is where Theorem~\ref{thm:occ} and Corollary~\ref{cor:tight} are used together. Fix $R>0$ and $q\in(\tfrac{2p}{p-N},\infty)$, and split $h_\varepsilon=h_\varepsilon\mathbf 1_{B_R(\x_0)}+h_\varepsilon\mathbf 1_{B_R(\x_0)^c}$ (Theorem~\ref{thm:occ} was proved conditionally on $\mathcal F_0$, so an $\mathcal F_0$-measurable localizing ball is admissible). By \eqref{eq:supp-occ},
\begin{equation*}
    \Expec\int_0^T\big(h_\varepsilon\mathbf 1_{B_R(\x_0)}\big)\big(s,\x_\omega(s)\big)\,ds\;\le\;C\,\Expec\Big(\int_0^T\|h_\varepsilon(s,\cdot)\|^q_{L^p(B_{R+1}(\x_0))}\,ds\Big)^{1/q},
\end{equation*}
which tends to zero as $\varepsilon\to0$ by dominated convergence (the integrand is bounded by $(2L_J)^q|B_{R+1}|^{q/p}$ and vanishes a.e.), \emph{uniformly in $\omega\ge1$}, because $C$ does not depend on $\omega$. By \eqref{eq:supp-tight}, the complementary term is at most $2L_JT\,\Prob(\sup_{t\le T}\|\x_\omega(t)-\x_0\|\ge R)$, which tends to zero as $R\to\infty$ uniformly in $\omega\ge\omega_0$. Hence $\lim_{\omega\to\infty}\Expec\int_0^Th_{\varepsilon(\omega)}(s,\x_\omega(s))\,ds=0$ along any $\varepsilon(\omega)\to0$, in particular along $\varepsilon=\omega^{-1/8}$ as used in Appendix~\ref{supp:rough-lip}. The same argument applies to every drift-replacement term of the identification step~(2) of that sketch, whose integrands are bounded and vanish a.e.\ as $\varepsilon\to0$. No rate in $\varepsilon$ is available for merely Lipschitz $J$, and none is claimed: this is why the Lipschitz rung of Table~\ref{tab:ladder} comes with convergence but without a rate.
\end{remark}

\begin{remark}[What the transformation does and does not give]\label{rem:occ-scope}
(i)~\emph{Time regularity.} The Lipschitz-in-time hypothesis of (H3) is used exactly once, to bound $\varrho_n$ in \eqref{eq:supp-occ-vt}. If $J$ is only of bounded variation in $t$, as in the standing hypotheses of this appendix, then $\varrho_n\,dt$ becomes a signed measure of total variation $O(\omega^{-1/2}(1+|k|V))$, which is small but singular, and Girsanov's theorem cannot absorb it; mollifying $J$ in $t$ as well does not help, because the mismatch bound $\sqrt\omega\,\|J-J_\varepsilon\|_\infty=O(1)$ then requires a modulus of continuity of $J$ in $t$ that bounded variation does not provide. Conjecture~\ref{conj:occ} therefore remains open for merely bounded-variation time dependence, and only for that reason: the kinks of $J$ in $\x$ play no role in the obstruction. In every instantiation of Sec.~\ref{sec:inst}, and in the tracking energies, the time dependence of $J$ and of $\bsig$ is smooth.
(ii)~\emph{Bounded is not small.} At $\varepsilon=\omega^{-1/2}$ the mismatch term and the It\^o-correction term of $\mathbf b$ are of order one, not $o(1)$; the transformed process is therefore \emph{not} close to the averaged one, and Theorem~\ref{thm:occ} does not by itself prove Proposition~\ref{prop:lip}(b). Making those terms small forces $\varepsilon\gg\omega^{-1/2}$, and the mismatch then has to be integrated by parts as a zero-mean oscillation, which is the two-stream architecture of Appendices~\ref{supp:rough-c11}--\ref{supp:rough-lip}. The transformation supplies the occupation estimate that architecture consumes, on a scale of its own, and nothing more.
(iii)~\emph{The Krylov--R\"ockner class.} The same computation shows why the transformation does not reach \eqref{eq:LPS}. If $\nabla J$ is merely in $L^p_{\mathrm{uloc}}$, $p>N$, then $\|\nabla J_\varepsilon\|_\infty\lesssim\varepsilon^{-N/p}$, $\|\nabla^2J_\varepsilon\|_\infty\lesssim\varepsilon^{-1-N/p}$, and by Morrey $\|J-J_\varepsilon\|_\infty\lesssim\varepsilon^{1-N/p}$. The It\^o-correction term of $\mathbf b$ is then bounded only if $\varepsilon\gtrsim\omega^{-p/(2(p+N))}$, while the mismatch term $\sqrt\omega\,\|J-J_\varepsilon\|_\infty$ is bounded only if $\varepsilon\lesssim\omega^{-p/(2(p-N))}$; since $p/(2(p-N))>p/(2(p+N))$, these windows are disjoint for every $N\ge1$. Two further obstacles are structural rather than technical: \eqref{eq:LPS} constrains $\nabla J$ only, whereas the algorithm evaluates $J$ itself, whose time dependence---an additive $c(t)$, invisible to $\nabla J$---modulates the dither phase directly; and the averaged drift is extracted by differentiating the corrector twice, which requires $\nabla^2J$ in some function class, whereas \eqref{eq:LPS} is a condition on the limit drift alone. This is the precise sense in which \eqref{eq:LPS} marks the boundary of provability for the \emph{limit} (Remark~\ref{rem:boundary}) and not, at present, for the averaging.
\end{remark}

\subsection{How rough can $J$ be?}
\label{supp:rough-boundary}

The hierarchy above is not bottomless, and its floor has a clean description.

\begin{remark}[The boundary of provability]\label{rem:boundary}
The differentiability of $J$ is not the binding constraint; the constraint is that the limiting drift $-\tfrac{k\alpha}2\bsig\bsig^{\!\top}\nabla J$ be an admissible drift for the non-degenerate diffusion. By Krylov--R\"ockner \citep{krylov2005strong} the averaged SDE is strongly well posed whenever the weak gradient satisfies the Ladyzhenskaya--Prodi--Serrin-type condition \eqref{eq:LPS}, i.e.\ $\nabla J\in L^q([0,T];L^p(\R^N))$ with $\frac Np+\frac2q<1$. Condition \eqref{eq:LPS} permits $\nabla J$ to be unbounded, defined only a.e., and merely $p$-integrable in space with $q$-integrability in time, so long as $\nabla J$ remains an $L^p$ function; it is the most general regularity under which the limit is provably well posed by classical means. (Note that \eqref{eq:LPS} forces $p>N$, so Morrey's embedding keeps $J(\cdot,t)$ itself continuous in $\x$ at a.e.\ fixed time---what \eqref{eq:LPS} adds over the Lipschitz rung is genuine local blow-up of the gradient, together with irregularity in time; spatial jumps in $J$ belong to the frontier below.) We expect the averaging convergence to extend throughout \eqref{eq:LPS}: the mollified-corrector proof of Theorem~\ref{thm:sobolev} extends once the corrector lies in the matching parabolic Sobolev class and the occupation estimates of Appendix~\ref{supp:rough-noise} are available at the corresponding integrability, and the critical case $\frac Np+\frac2q=1$ is treated in \citep{krylov2021drift}. We record in Remark~\ref{rem:occ-scope}(iii) why the pathwise transformation of Appendix~\ref{supp:rough-occ}, which settles the Lipschitz rung, does not by itself reach \eqref{eq:LPS}: with $\nabla J$ merely in $L^p$, the scale at which the transformation is a diffeomorphism with bounded second derivatives is incompatible with the scale at which the dither mismatch is bounded, and \eqref{eq:LPS} constrains only the limit drift, not the phase $J$ that the algorithm actually dithers. The genuine outer boundary is reached when $\nabla J$ leaves the function-valued world entirely---a true jump in $J$, whose gradient is a surface measure, i.e.\ $\nabla J\in C^{-\beta}$ in a Besov sense. There the limiting SDE has distributional drift, well posed only in low dimension or under Besov restrictions and resting on regularization-by-noise arguments \citep{flandoli2010wellposedness}. In one line: we can prove convergence with rates up to $W^{2,p}$ energies and without a rate for Lipschitz energies (through the occupation estimate of Theorem~\ref{thm:occ}, under Lipschitz time dependence), expect it up to the Krylov--R\"ockner class \eqref{eq:LPS}, and reach the open frontier at distributional gradients---and in every case it is the non-degenerate sampling noise, not any smoothness of $J$, that carries the argument.
\end{remark}

\section{Choosing Bounded ES Parameters}
\label{supp:ES_params}

\textbf{Frequency ratios $\{r_i\}$.} Any pairwise-distinct positive constants satisfy Assumption~\ref{ass:freq}; we use $r_i$ evenly spaced (or i.i.d.\ uniform) in $[1,N]$ with modest $N\in[1.75,2]$, fixed once. At finite $\omega$ the cross-coordinate error terms decay like $(\omega\,|r_i-r_j|)^{-1/2}$ (Appendix~\ref{supp:proof}), so the minimal gap $\min_{i\ne j}|r_i-r_j| \approx \tfrac{3}{4(d-1)}$ ties the required base frequency to dimension: \emph{$\omega^*$ grows at least linearly in $d$ at fixed accuracy}. Cross-coordinate beat terms average out once the minimum gap frequency $(N{-}1)\omega/(d{-}1)$ (band $[\omega,N\omega]$) exceeds a constant of order ten, giving $\omega^{\ast}(d)\approx O(10)(d-1)/(N-1)$ with the constant empirically $5$--$15$; a fixed sweep window returns the $\omega^{-1/2}$ law while its top satisfies this criterion, degrades toward slope $0$ once it does not, and recovers when the window is raised --- verified up to $d=1024$ (Appendix~\ref{supp:num-threshold}, Table~\ref{tab:supp-windows}). Appendix~\ref{supp:dimension} proves both directions of this law: the threshold is at least linear in $d$ (Proposition~\ref{prop:dim-lower}) and at most polynomial (Proposition~\ref{prop:dim-upper}), and the number of evaluations needed to track the averaged dynamics grows at least linearly in $d$ for every band and horizon (Theorem~\ref{thm:dim-nfe}).

\textbf{Step size $\Delta t$ and evaluation count.} Euler--Maruyama must resolve the fastest dither: we use $\omega r_{\max}\,\Delta t \le 2\pi/10$, i.e., at least $10$ steps per fastest dither period --- a threshold validated directly: measured rates are invariant from $10$ to $100$ steps per period (a $10\times$ range in $\Delta t$) and deviate only at $5$ (Appendix~\ref{supp:num-rate}). Since one step costs exactly one evaluation of $J$ (line 2), the number of function evaluations (NFE) per unit diffusion time is $1/\Delta t \approx 10\,\omega r_{\max}/2\pi$: \emph{NFE scales linearly with $\omega$, while the averaging error decays like $\omega^{-1/2}$ at best}. By contrast, a two-point Gaussian-smoothing estimate \citep{nesterov2017random,roy2022zeroth} costs $2$ evaluations per Langevin step and a coordinate-wise finite difference costs $d+1$; ES costs $1$ but takes more, smaller steps, and measurement noise on $J$ enters only through the dither phase, where it is averaged over the demodulation window, rather than being amplified by a $1/\nu$ divided difference (this claim is quantified --- including a validated noise-floor law and its breakdown point --- in Appendix~\ref{supp:num-optim}). The two measured laws compose: the physical settling time is independent of $\omega$ and equals the averaged system's own (Appendix~\ref{supp:num-conv}), so evaluations-to-converge grow $\propto\omega$ while accuracy improves as $\omega^{-1/2}$ --- accuracy $\sim$ NFE$^{-1/2}$, measured.

\textbf{Gains.} $k\alpha$ sets the strength of the emergent drift (Eq.~\ref{eq:averaged-sde}); the instantiations of Sec.~\ref{sec:inst} fix it to match the target dynamics. Because $k\alpha$ thereby also sets the temperature of the averaged law ($\rho_\infty\propto e^{-k\alpha\E_\theta}$ in (I$_\lambda$)), any dimension-dependent reduction of $k$ adopted for stability at finite $\omega$---as in the mode-seeking benchmark of Appendix~\ref{supp:num-optim}, where the temperature is immaterial---must be compensated by $\alpha$, or by rescaling $J$, whenever the goal is sampling from a fixed target. Excessive $k$ at finite $\omega$ destabilizes the fast dynamics---Fig.~\ref{fig:Failure}(bottom) shows exactly this failure mode---so we sweep $k$ around the nominal value in every experiment.

\begin{figure}[t]
\centering
\includegraphics[width=1.0\textwidth]{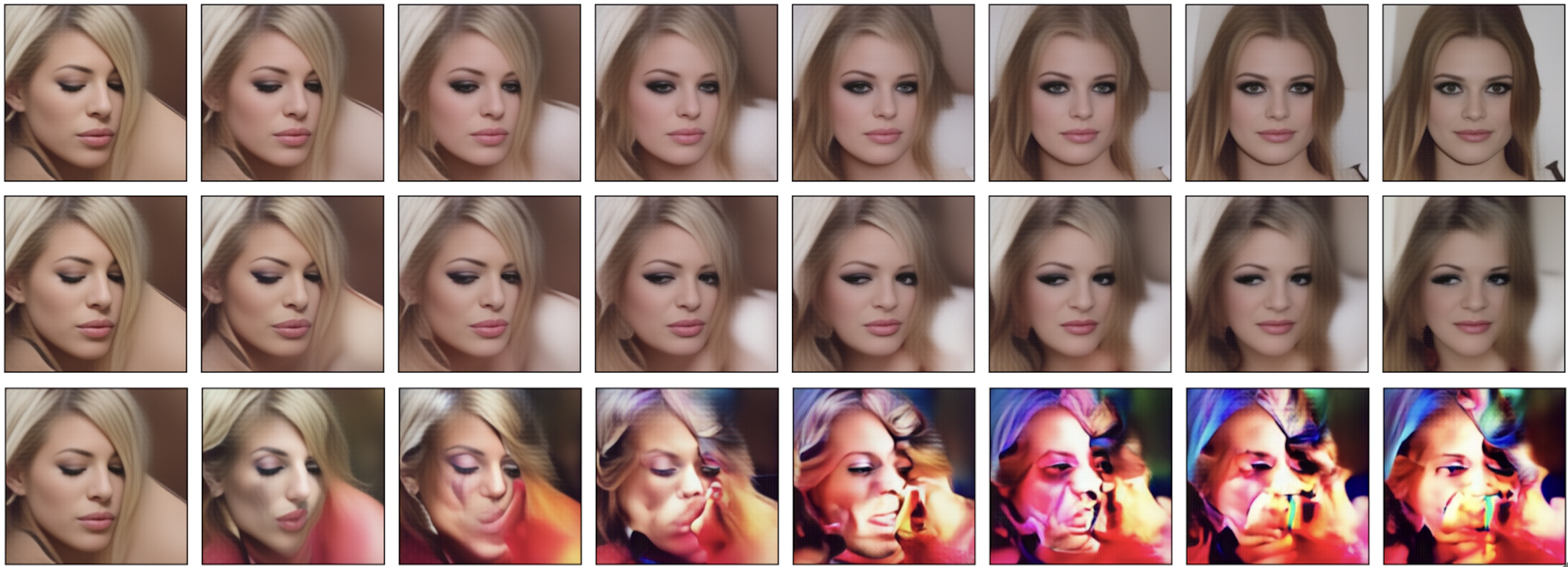}
\caption{Latent-space tracking from 1D RGB projections of time-varying images. The top row shows the true images, the second row shows decent tracking, and the bottom row shows catastrophic instability and failure when the ES gain $k\alpha>0$ is extremely large.}
\label{fig:Failure}
\end{figure}

\section{Detailed numerical studies on the synthetic testbed}
\label{supp:numerics}

This appendix documents, in full, the quantitative studies of the direct verification of Theorem~\ref{thm:main} and its $\omega^{-1/2}$ rate across dimension, noise level, and time discretization (\ref{supp:num-rate}); the finite-$\omega$ transition, its dimension scaling, and the robustness checks behind it (\ref{supp:num-threshold}); the convergence-speed and cost measurements (\ref{supp:num-conv}); the matched-budget comparison against gradient-based and derivative-free samplers (\ref{supp:num-baselines}); a fixed-target benchmark of the mode-seeking configuration against derivative-free
\emph{optimizers} under measurement noise, with a validated noise-floor law
and a readout (iterate-averaging) study (\ref{supp:num-optim}).

\subsection{Testbed, conventions, and the coupled protocol}
\label{supp:num-testbed}

\paragraph{Target.} All synthetic studies use a four-mode Gaussian mixture in
$\R^d$ with modes placed symmetrically in the first two coordinates and
isotropic unit covariance:
\begin{equation}
p(\x)\;=\;\frac14\sum_{m=1}^{4}\mathcal N\!\big(\x;\,\bmu_m,\,I_d\big),
\qquad
\bmu_m\;=\;\big(\pm s,\ \pm s,\ 0,\ \dots,\ 0\big)\in\R^d,
\qquad s=3,
\label{eq:supp-gmm}
\end{equation}
with energy $\E(\x)=-\log p(\x)$. The four modes are well separated
($\|\bmu_m-\bmu_{m'}\|\in\{6,6\sqrt2\}$ against unit component width) in the
first two coordinates, while the remaining $d-2$ coordinates are a product of
standard Gaussians; $\log p$ and $\nabla\log p$ are available in closed form
(via a log-sum-exp and its responsibilities), which supplies exact-gradient
oracles and exact target samples for the metrics below. The sampler, of
course, only ever evaluates the scalar $\E$.

\paragraph{ES configuration.} The studies run the discretization of
\eqref{eq:es-sde} in the convention
\begin{equation}
dx_i\;=\;g\sqrt{\alpha\omega_i}\,\cos\!\big(\omega_i t+k\,J(\x)\big)\,dt
\;+\;g\,\sigma\,dw_i,
\qquad g=1,\ k=\alpha=1,\ \sigma\in\{0,\,0.1,\,1\},
\label{eq:supp-conv}
\end{equation}
with $\omega_i=\omega r_i$, ratios $r_i$ evenly spaced in $[1,N]$ (fixed
once), Euler step $\Delta t=2\pi/(cN\omega)$ ($c$ steps per fastest dither
period), dispersed initialization $\x_0\sim\mathcal N(0,\,2^2 I_d)$, $256$
chains, and a fixed physical horizon $T=4$ (so the step count
$\lceil T/\Delta t\rceil$ grows linearly in $\omega$). In this convention the
averaged dynamics are
$d\bar\x=\tfrac{k\alpha}{2}g^2\nabla\log p\,dt+g\sigma\,d\w$, whose
stationary law is $\propto p^{\,k\alpha/\sigma^2}$ (the zero-flux condition of the stationary Fokker--Planck equation gives $\nabla\log\rho_\infty=2\bar\bmu/(g\sigma)^2=(k\alpha/\sigma^2)\nabla\log p$, in which $g^2$ cancels): the $\sigma=1$ rows
therefore sample $p$ exactly and constitute the sampling setting, while
$\sigma\in\{0,0.1\}$ probe the deterministic-to-stochastic range of the
trajectory statement of Theorem~\ref{thm:main}, which holds for any
$\sigma\ge0$.

\paragraph{Coupled protocol.} Following the coupling in
Theorem~\ref{thm:main}, each ES run is paired with an Euler simulation of its
averaged SDE using the \emph{same} initial condition and the \emph{same}
per-step Gaussian increments $\boldsymbol{\xi}_j$:
\begin{align}
x_{j+1,i}&=x_{j,i}+g\sqrt{\alpha\,\omega_i}\,
\cos\!\big(\omega_i t_j+k\,J(\x_j)\big)\,\Delta t
+g\sigma\sqrt{\Delta t}\,\xi_{j,i},\\
\bar\x_{j+1}&=\bar\x_j+\tfrac{k\alpha}{2}g^2\,
\nabla\log p(\bar\x_j)\,\Delta t
+g\sigma\sqrt{\Delta t}\,\boldsymbol{\xi}_j .
\end{align}
The primary metric is the per-chain pathwise supremum
$\sup_{j\le T/\Delta t}\|\x_j-\bar\x_j\|_2$, averaged over the $256$ chains;
for each seed we fit the slope of $\log(\text{metric})$ against $\log\omega$
by least squares over the sweep window, and report the mean $\pm$ std of the
per-seed slopes. (Time-averaged normalized errors --- NMSE and NMAE --- were
also recorded; the supremum is the theorem-faithful quantity, and since NMSE
is a squared error its fitted slopes are approximately twice those quoted
here.)

\subsection{Rate in $\omega$: dimension, noise level, and discretization}
\label{supp:num-rate}

Tables~\ref{tab:supp-rate-s0}--\ref{tab:supp-rate-s1} report the fitted
slopes over the window $\omega\in\{10,\,31.6,\,100,\,316,\,1000\}$ with band
$N=1.75$, for $d\in\{2,\dots,32\}$, noise $\sigma\in\{0,0.1,1\}$, and
steps-per-period $c\in\{5,10,20,100\}$.

\begin{table}[h]
\centering\small
\caption{Fitted log--log slope of the coupled sup-error vs.\ $\omega$;
$\sigma=0$ (deterministic ES vs.\ gradient flow), window
$\omega\in[10,10^3]$, $N=1.75$. Mean$\pm$std over 5 seeds.}
\label{tab:supp-rate-s0}
\begin{tabular}{@{}rcccc@{}}
\toprule
$d$ & $c=5$ & $c=10$ & $c=20$ & $c=100$\\
\midrule
2  & $-0.410\pm0.026$ & $-0.489\pm0.048$ & $-0.487\pm0.057$ & $-0.489\pm0.056$\\
4  & $-0.494\pm0.035$ & $-0.587\pm0.042$ & $-0.594\pm0.043$ & $-0.596\pm0.038$\\
8  & $-0.574\pm0.020$ & $-0.655\pm0.014$ & $-0.650\pm0.014$ & $-0.655\pm0.015$\\
16 & $-0.652\pm0.005$ & $-0.705\pm0.013$ & $-0.703\pm0.018$ & $-0.701\pm0.018$\\
32 & $-0.641\pm0.006$ & $-0.645\pm0.005$ & $-0.649\pm0.006$ & $-0.652\pm0.006$\\
\bottomrule
\end{tabular}
\end{table}

\begin{table}[h]
\centering\small
\caption{As Table~\ref{tab:supp-rate-s0}, for $\sigma=0.1$.}
\label{tab:supp-rate-s01}
\begin{tabular}{@{}rcccc@{}}
\toprule
$d$ & $c=5$ & $c=10$ & $c=20$ & $c=100$\\
\midrule
2  & $-0.395\pm0.037$ & $-0.499\pm0.050$ & $-0.504\pm0.055$ & $-0.448\pm0.058$\\
4  & $-0.465\pm0.023$ & $-0.550\pm0.050$ & $-0.552\pm0.041$ & $-0.571\pm0.023$\\
8  & $-0.578\pm0.025$ & $-0.653\pm0.016$ & $-0.661\pm0.021$ & $-0.657\pm0.024$\\
16 & $-0.646\pm0.014$ & $-0.705\pm0.019$ & $-0.709\pm0.014$ & $-0.720\pm0.017$\\
32 & $-0.635\pm0.008$ & $-0.642\pm0.007$ & $-0.645\pm0.007$ & $-0.646\pm0.006$\\
\bottomrule
\end{tabular}
\end{table}

\begin{table}[h]
\centering\small
\caption{As Table~\ref{tab:supp-rate-s0}, for $\sigma=1$ (the sampling
setting: the averaged dynamics target $p$ exactly).}
\label{tab:supp-rate-s1}
\begin{tabular}{@{}rcccc@{}}
\toprule
$d$ & $c=5$ & $c=10$ & $c=20$ & $c=100$\\
\midrule
2  & $-0.387\pm0.031$ & $-0.459\pm0.021$ & $-0.466\pm0.009$ & $-0.452\pm0.011$\\
4  & $-0.448\pm0.018$ & $-0.500\pm0.013$ & $-0.493\pm0.010$ & $-0.474\pm0.016$\\
8  & $-0.522\pm0.011$ & $-0.526\pm0.012$ & $-0.514\pm0.014$ & $-0.527\pm0.009$\\
16 & $-0.558\pm0.006$ & $-0.550\pm0.007$ & $-0.532\pm0.010$ & $-0.539\pm0.008$\\
32 & $-0.544\pm0.003$ & $-0.530\pm0.004$ & $-0.509\pm0.009$ & $-0.511\pm0.007$\\
\bottomrule
\end{tabular}
\end{table}

Three conclusions. \textbf{(a) Discretization.} The slopes are invariant, at
every $(d,\sigma)$, across $c\in\{10,20,100\}$ --- a $10\times$ range in
$\Delta t$ --- while $c=5$ deviates systematically toward shallower values:
the resolved-discretization regime is $c\ge10$, and the deviation of $c=5$
identifies the threshold rather than any property of the continuous system.
This validates the default $c=20$ as conservative and $c=10$ as the
cost-optimal resolved choice (halving evaluations per unit time).
\textbf{(b) The sampling setting.} At $\sigma=1$ the fitted rate is
$-0.50\pm0.04$ across the entire range $d\in[2,32]$ with no systematic
dimension dependence: $\omega^{-1/2}$, dimension-independent, as Theorem~\ref{thm:main} requires. \textbf{(c) Low noise.} At $\sigma\in\{0,0.1\}$, mid-dimensional fits steepen to $\approx-0.7$, i.e., the error decays \emph{faster} than the one-sided $\omega^{-1/2}$ bound over the measured window. This excess is invariant under the $10\times$ step refinement above, under randomization of the dither-frequency grid (\ref{supp:num-threshold}), and across the noise levels probed; we attribute it to higher-order error terms of the expansion (the $O(\omega^{-1})$-type contributions in the proof) whose constants grow with $d$, and note that decay faster than the bound is fully consistent with
Theorem~\ref{thm:main}, which is one-sided. The local-slope analysis below
resolves this excess into a transition band that moves with $d$.

\subsection{The finite-$\omega$ transition and its dimension scaling}
\label{supp:num-threshold}

\paragraph{Two clocks.} The iteration \eqref{eq:supp-conv} with
$\Delta t=2\pi/(cN\omega)$ looks self-similar in $\omega$: the fastest dither
always takes $c$ steps per oscillation and the slowest $cN$. The
substitution $\tau=\omega t$ (with the Brownian scaling
$dw(t)=\omega^{-1/2}dW(\tau)$) makes this precise and shows what does change:
\begin{equation}
dx_i \;=\; g\sqrt{\tfrac{\alpha r_i}{\omega}}\,
\cos\!\big(r_i\tau+k\E\big)\,d\tau
\;+\;g\,\sigma\,\omega^{-1/2}\,dW_i(\tau).
\label{eq:supp-fastclock}
\end{equation}
On this fast clock the dither frequencies are the fixed ratios $r_i$ and the
Euler grid $\Delta\tau=2\pi/(cN)$ is $\omega$-independent, but every term
that moves the state carries $\omega^{-1/2}$: per dither period the state
moves only $O(\omega^{-1/2})$, so the phase $k\E(\x)$ is frozen to that order
over each oscillation (the working condition of averaging, and the origin of
the $\omega^{-1/2}$ rate), while each step covers only $2\pi/(cN\omega)$ of
slow time (the origin of the linear-in-$\omega$ cost). $\omega$ is thus the
exchange rate between the clocks: the number of dither oscillations per unit
of the averaged dynamics' own timescale.

\paragraph{Beats and the gap criterion.} Expanding the coordinate ripples
(cf.\ the warm-up of \ref{supp:warmup}) into the phase, the drift of
coordinate $n$ contains products of its dither with every coordinate's
ripple; by the product formulas these oscillate at the sum frequencies
$(r_n{+}r_i)\omega$ (always fast) and the \emph{beat} frequencies
$\Delta_{ni}=\omega|r_n-r_i|$, of which the smallest, for evenly spaced
ratios, is
\begin{equation}
\Delta_{\min}\;=\;\frac{(N-1)\,\omega}{d-1}.
\label{eq:supp-gapfreq}
\end{equation}
The resonant term $i=n$ (zero beat frequency) is the emergent gradient; every
other term is an oscillatory integral
$\int_0^T A(t)\cos(\Delta t)\,dt$ with slowly varying envelope, bounded after
one integration by parts by $(2\|A\|_\infty+T\|\dot A\|_\infty)/\Delta$: it
is negligible precisely when the beat completes many cycles per unit of the
slow timescale. Empirically ``many'' is a factor of order ten, giving the
finite-$\omega$ threshold
\begin{equation}
\Delta_{\min}\;\gtrsim\;O(10)
\qquad\Longleftrightarrow\qquad
\omega^{\ast}(d)\;\approx\;\frac{O(10)\,(d-1)}{N-1},
\label{eq:supp-threshold}
\end{equation}
with the constant measured below to be $5$--$15$. This converts the
existence statement of Theorem~\ref{thm:main} into a design formula: linear
in $d$, controlled by the band spread.

\paragraph{Local slopes locate the transition.} Table~\ref{tab:supp-local}
reports segment-by-segment (local) slopes between consecutive $\omega$'s of
the sweep, per dimension and noise level. At $\sigma\in\{0,0.1\}$ a steep
band (local slopes $-0.8$ to $-1.0$) is visible whose location moves to
higher $\omega$ with $d$: the steepest segments sit at geometric segment
centers $\approx56,\ 178,\ 562$ for $d=8,16,32$ --- in order-of-magnitude
agreement with \eqref{eq:supp-threshold}
($13.3\,(d-1)=93,\,200,\,413$ at $N=1.75$), growing slightly faster than
linearly, unsurprising since the \emph{number} of near-minimal gaps also
grows with $d$. On either side of the band the local rate is near $-1/2$; at
$\sigma=1$ the band is smeared by the diffusion and only a weak bump remains.
Fitted whole-window slopes steeper than $-1/2$
(Tables~\ref{tab:supp-rate-s0}--\ref{tab:supp-rate-s01}) are chords through
this band.

\begin{table}[h]
\centering\small
\setlength{\tabcolsep}{4.5pt}
\caption{Local (segment) slopes of the coupled sup-error between consecutive
sweep frequencies $\omega\in\{10,31.6,100,316,1000\}$ (segment centers
$\approx17.8,\ 56,\ 178,\ 562$), and the fit over the top three points.
Mean over 5 seeds; $N=1.75$.}
\label{tab:supp-local}
\begin{tabular}{@{}rccccc@{}}
\toprule
& \multicolumn{4}{c}{local slope per segment} & \\
\cmidrule(lr){2-5}
$d$ & $10{\to}31.6$ & $31.6{\to}100$ & $100{\to}316$ & $316{\to}1000$ & top-3 fit\\
\midrule
\multicolumn{6}{@{}l}{$\sigma=0$}\\
2  & $-0.47$ & $-0.51$ & $-0.50$ & $-0.40$ & $-0.45$\\
4  & $-0.63$ & $-0.64$ & $-0.51$ & $-0.59$ & $-0.55$\\
8  & $-0.53$ & $-0.77$ & $-0.68$ & $-0.54$ & $-0.61$\\
16 & $-0.43$ & $-0.67$ & $-0.89$ & $-0.74$ & $-0.82$\\
32 & $-0.27$ & $-0.53$ & $-0.80$ & $-0.99$ & $-0.90$\\
\midrule
\multicolumn{6}{@{}l}{$\sigma=0.1$}\\
2  & $-0.45$ & $-0.47$ & $-0.51$ & $-0.56$ & $-0.53$\\
4  & $-0.65$ & $-0.56$ & $-0.57$ & $-0.39$ & $-0.48$\\
8  & $-0.54$ & $-0.77$ & $-0.69$ & $-0.57$ & $-0.63$\\
16 & $-0.43$ & $-0.66$ & $-0.89$ & $-0.79$ & $-0.84$\\
32 & $-0.26$ & $-0.53$ & $-0.79$ & $-0.98$ & $-0.88$\\
\midrule
\multicolumn{6}{@{}l}{$\sigma=1$}\\
2  & $-0.43$ & $-0.52$ & $-0.47$ & $-0.42$ & $-0.44$\\
4  & $-0.46$ & $-0.54$ & $-0.48$ & $-0.47$ & $-0.48$\\
8  & $-0.46$ & $-0.56$ & $-0.51$ & $-0.49$ & $-0.50$\\
16 & $-0.36$ & $-0.63$ & $-0.57$ & $-0.50$ & $-0.53$\\
32 & $-0.16$ & $-0.57$ & $-0.64$ & $-0.57$ & $-0.60$\\
\bottomrule
\end{tabular}
\end{table}

\paragraph{High-dimensional falloff and verification by window shifting.}
Extending the sweep to $d=1024$ on a fixed window makes the threshold's
motion visible from the other side: once even the window top is
sub-threshold, the fitted chord returns the pre-asymptotic plateau and slides
toward $0$. Table~\ref{tab:supp-windows} reports the fitted slopes for two
windows --- A: $\omega\in[10,10^3]$, $N=1.75$; B: $\omega\in[10,10^4]$,
$N=2$, so that the top-of-window gap frequency
$\Delta_{\min}(\omega_{\mathrm{hi}})$ grows by a factor $13.3$ --- at $c=10$
throughout. If \eqref{eq:supp-threshold} is right, raising the window top
must restore the slope at fixed $d$; it does, monotonically, at every $d$.

\begin{table}[h]
\centering\small
\setlength{\tabcolsep}{4pt}
\caption{Whole-window slopes vs.\ dimension for the two sweep windows, with
the top-of-window minimum gap frequency
$\Delta_{\min}(\omega_{\mathrm{hi}})=(N{-}1)\,\omega_{\mathrm{hi}}/(d{-}1)$.
Window A: $\omega_{\mathrm{hi}}=10^3$, $N=1.75$ ($\sigma=1$ shown; the
$d\le32$ block coincides with the $c=10$ column of
Table~\ref{tab:supp-rate-s1}). Window B: $\omega_{\mathrm{hi}}=10^4$, $N=2$,
all three noise levels. Mean$\pm$std over 5 seeds.}
\label{tab:supp-windows}
\begin{tabular}{@{}rcc|cccc@{}}
\toprule
& \multicolumn{2}{c|}{window A ($\sigma{=}1$)} &
\multicolumn{4}{c}{window B}\\
$d$ & $\Delta_{\min}$ & slope & $\Delta_{\min}$ &
slope, $\sigma{=}0$ & slope, $\sigma{=}0.1$ & slope, $\sigma{=}1$\\
\midrule
2    & 750  & $-0.459\pm0.021$ & 10000 & $-0.469\pm0.023$ & $-0.440\pm0.031$ & $-0.431\pm0.005$\\
4    & 250  & $-0.500\pm0.013$ & 3333  & $-0.529\pm0.010$ & $-0.497\pm0.016$ & $-0.467\pm0.019$\\
8    & 107  & $-0.526\pm0.012$ & 1429  & $-0.558\pm0.008$ & $-0.557\pm0.026$ & $-0.498\pm0.009$\\
16   & 50   & $-0.550\pm0.007$ & 667   & $-0.615\pm0.022$ & $-0.642\pm0.004$ & $-0.520\pm0.012$\\
32   & 24   & $-0.530\pm0.004$ & 323   & $-0.673\pm0.013$ & $-0.662\pm0.016$ & $-0.523\pm0.006$\\
64   & 11.9 & $-0.465\pm0.003$ & 159   & $-0.671\pm0.009$ & $-0.673\pm0.005$ & $-0.508\pm0.005$\\
128  & 5.9  & $-0.365\pm0.002$ & 79    & $-0.594\pm0.003$ & $-0.598\pm0.002$ & $-0.464\pm0.004$\\
256  & 2.9  & $-0.241\pm0.001$ & 39    & $-0.495\pm0.003$ & $-0.500\pm0.002$ & $-0.408\pm0.001$\\
512  & 1.5  & $-0.122\pm0.001$ & 19.6  & $-0.392\pm0.001$ & $-0.391\pm0.001$ & $-0.337\pm0.001$\\
1024 & 0.73 & $-0.038\pm0.000$ & 9.8   & $-0.285\pm0.000$ & $-0.289\pm0.001$ & $-0.254\pm0.000$\\
\bottomrule
\end{tabular}
\end{table}

The slope tracks $\Delta_{\min}$, not $d$ per se: wherever the window top
satisfies the criterion \eqref{eq:supp-threshold} the fitted rate sits near
$-1/2$ (or in the overshoot band at low $\sigma$); the largest dimensions
recover monotonically as the window climbs past $\omega^{\ast}(d)$, their
chords still averaging over the sub-threshold lower window (hence the partial
values at $d\ge512$). At $d=1024$ the top of window A gives
$\Delta_{\min}=0.73<1$: adjacent dithers complete less than one relative beat
per unit slow time, no averaging of the closest pairs is possible, and the
measured slope is correspondingly near zero. Full recovery at $d\ge512$ is
predicted at $\omega\gtrsim 10\,(d-1)/(N-1)$, i.e., a $[10^3,10^5]$ window.

\paragraph{Why widening the band does not help.} Resolving the fastest
dither forces $\Delta t=2\pi/(cN\omega)$, so the evaluation cost per unit
slow time is $cN\omega/2\pi=c\,\omega_{\max}/2\pi$ --- set by the \emph{top}
of the band $\omega_{\max}=N\omega$ regardless of how it is split. At fixed
$\omega_{\max}$, widening $N$ buys relative minimum gap $(N{-}1)/N$, which
saturates immediately (already $0.67$ at $N=3$), while the slowest
coordinates' own averaging error remains at
$(\omega_{\max}/N)^{-1/2}$ and the cross-term prefactors
$\sqrt{\omega_i/\omega_n}\le\sqrt N$ grow. We therefore fix a modest band
($N\in[1.75,2]$) and scale the base frequency; the dimension cost is then
intrinsic: reaching the asymptotic regime at dimension $d$ requires
$\omega_{\max}=\Omega(d)$ and hence $\Omega(d)$ evaluations per unit slow
time --- the quantitative basis for the latent-space scoping of
Sec.~\ref{sec:practice}.

\paragraph{Robustness checks.} \emph{(i) Frequency-grid randomization.} The
evenly spaced grid is an arithmetic progression and therefore carries exact
higher-order sum resonances ($r_i+r_k=2r_j$ for every consecutive triple).
Jittering each ratio by up to $20\%$ of the grid spacing (destroying the
exact resonances while keeping the gap scale; minimum gap $0.035$ vs.\
$0.050$) reproduces the linspace errors to within $2\%$ at every $\omega$
and the local slopes to $\pm0.02$ (checked at $d=16$, $\sigma=0$, $c=20$;
$\omega\in\{31.6,100,316,1000\}$: errors
$3.62/1.68/0.616/0.245$ vs.\ $3.60/1.71/0.621/0.251$, local slopes
$-0.67/-0.87/-0.80$ vs.\ $-0.65/-0.88/-0.79$; $2$ seeds, $64$ chains).
Grid design is thus not the origin of the low-noise steepening, and the
practical recommendation ``evenly spaced or i.i.d.\ ratios, fixed once''
carries a data point. \emph{(ii) Floating point.} The largest dither phase in these sweeps is $\omega_{\max}T\le 2\times10^5$ radians, resolved by double precision to $\sim10^{-11}$ radians, and the measured differences $\x-\bar\x$ are $O(10^{-1})$ against states $O(1)$: no cancellation is involved. Together
with the $10\times$ step-refinement invariance of \ref{supp:num-rate}, the
transition structure documented here is a property of the continuous-time
coupled system, not of its implementation.

\subsection{Convergence speed and the cost law}
\label{supp:num-conv}

\paragraph{Settling-time detector.} Convergence speed is measured as the
\emph{last-exit settling time} of the chain-mean energy
$m(t)=\frac{1}{256}\sum_{\text{chains}}J(\x(t))$, the standard settling
convention of control engineering: $m$ is smoothed over
$\min(4c,\ \text{series length}/10)$ steps ($\approx$ two slowest dither
periods) to remove the ripple; the plateau $m_\infty$ and its noise
$s_\infty$ are taken from the final $25\%$ of the run; the band is
$\max\!\big(0.05\,|m(0)-m_\infty|,\ 4\,s_\infty\big)$; and
$t_{\rm conv}$ is the time after which $m$ never again leaves the band. Two
guard flags accompany every measurement: \emph{settled} is false when the
last exit lands in the final $10\%$ of the horizon (plateau not reached), and
\emph{transient} is false when $|m(0)-m_\infty|<3s_\infty$, i.e., there was
no transient to time --- the signature of the sub-$\omega^{\ast}$ regime,
where the effective drift is too weak to move the ensemble; such runs are
excluded from slope fits. The same detector is applied to the coupled
averaged twin, whose settling time is the reference. Horizon $T=4$, $c=20$,
$N=1.75$, all $(\sigma,\omega,d)$ of the main sweep; two clocks are
reported: physical time $t_{\rm conv}$ and steps-to-settle (equal to the
per-chain evaluation count).

\begin{table}[h]
\centering\small
\caption{Convergence-speed summary. Fitted log--log slopes of
$t_{\rm conv}$ and of the evaluations-to-settle vs.\ $\omega$ over
$\omega\in[10,10^3]$, and the ES-to-averaged settling-time ratio at
$\omega=10^3$. Entries $\sigma=0\,/\,0.1\,/\,1$; 5 seeds; all $75$
configurations settled $5/5$.}
\label{tab:supp-conv-summary}
\begin{tabular}{@{}rccc@{}}
\toprule
$d$ & slope of $t_{\rm conv}$ & slope of evals-to-settle & ES/averaged at $\omega{=}10^3$\\
\midrule
2  & $+0.02\,/\,+0.02\,/\,+0.02$ & $+1.02\,/\,+1.02\,/\,+1.02$ & $1.00\,/\,1.00\,/\,0.99$\\
4  & $+0.02\,/\,+0.02\,/\,+0.07$ & $+1.02\,/\,+1.02\,/\,+1.07$ & $1.00\,/\,1.00\,/\,1.01$\\
8  & $+0.01\,/\,+0.01\,/\,+0.02$ & $+1.01\,/\,+1.01\,/\,+1.02$ & $1.00\,/\,1.00\,/\,1.00$\\
16 & $+0.05\,/\,+0.05\,/\,+0.08$ & $+1.05\,/\,+1.05\,/\,+1.08$ & $1.00\,/\,1.00\,/\,1.00$\\
32 & $+0.04\,/\,+0.04\,/\,+0.04$ & $+1.04\,/\,+1.04\,/\,+1.04$ & $0.99\,/\,0.99\,/\,1.00$\\
\bottomrule
\end{tabular}
\end{table}

\paragraph{Results.} Table~\ref{tab:supp-conv-summary}: the physical
settling time is independent of $\omega$ (all fitted slopes $\le+0.08$ over
two decades) and coincides with the averaged system's own settling time
(ratio $0.99$--$1.01$) for every $(d,\sigma)$ --- convergence speed is
governed by the averaged dynamics, an independent verification of
Theorem~\ref{thm:main} in a different observable than the trajectory error.
Correspondingly the evaluation count to settle grows linearly in $\omega$
(slopes $+1.01$ to $+1.08$). Combining the two measured laws with the rate of
\ref{supp:num-rate} gives the efficiency scaling quoted in
Sec.~\ref{sec:practice}: accuracy $\sim$ (evaluations)$^{-1/2}$.
Tables~\ref{tab:supp-conv-s0}--\ref{tab:supp-conv-s1} list the full
measurements at $\sigma=0$ and $\sigma=1$; $\sigma=0.1$ matches $\sigma=0$
within the quoted uncertainties cell by cell and is omitted. The small
common upward drift of $t_{\rm conv}$ at low $\omega$ (e.g., $2.38\to2.60$
for $d=2$) is finite detector resolution --- the smoothing window is fixed in
steps, hence $O$(one dither period) wide in time at small $\omega$ --- and
affects ES and twin identically, which is why the per-$\omega$ ratio is the
trustworthy statistic. The isolated low-$\omega$/large-$d$ cells where ES
settles \emph{early} to a degraded plateau ($\sigma=0$: $d\ge16$,
$\omega\le31.6$; $\sigma=1$: $d=16$, $\omega=31.6$, and the flagged $d=8$,
$\omega=10$ cell with transient $2/5$) lie below the transition threshold of
\ref{supp:num-threshold} --- the two studies locate the same band
independently.

\begin{table}[h]
\centering\small
\setlength{\tabcolsep}{5pt}
\caption{Full settling measurements, $\sigma=0$ ($c=20$; 5 seeds; all cells
settled $5/5$, transient $5/5$). Per cell: $t_{\rm conv}$ and
evaluations-to-settle.}
\label{tab:supp-conv-s0}
\begin{tabular}{@{}rccccc@{}}
\toprule
$d$ & $\omega{=}10$ & $\omega{=}31.6$ & $\omega{=}100$ & $\omega{=}316$ & $\omega{=}10^3$\\
\midrule
2  & $2.38{\pm}0.02$; $132$ & $2.40{\pm}0.01$; $423$ & $2.54{\pm}0.01$; $1412$ & $2.58{\pm}0.02$; $4551$ & $2.60{\pm}0.02$; $14478$\\
4  & $2.33{\pm}0.01$; $130$ & $2.33{\pm}0.00$; $411$ & $2.48{\pm}0.00$; $1384$ & $2.55{\pm}0.01$; $4485$ & $2.56{\pm}0.01$; $14261$\\
8  & $2.58{\pm}0.01$; $144$ & $2.26{\pm}0.02$; $399$ & $2.43{\pm}0.00$; $1355$ & $2.51{\pm}0.00$; $4425$ & $2.54{\pm}0.00$; $14143$\\
16 & $2.09{\pm}0.05$; $117$ & $2.07{\pm}0.07$; $364$ & $2.36{\pm}0.00$; $1312$ & $2.48{\pm}0.00$; $4368$ & $2.52{\pm}0.00$; $14049$\\
32 & $2.11{\pm}0.01$; $118$ & $2.11{\pm}0.03$; $372$ & $2.39{\pm}0.01$; $1334$ & $2.42{\pm}0.00$; $4258$ & $2.50{\pm}0.00$; $13943$\\
\bottomrule
\end{tabular}
\end{table}

\begin{table}[h]
\centering\small
\setlength{\tabcolsep}{5pt}
\caption{As Table~\ref{tab:supp-conv-s0}, for $\sigma=1$ (the sampling
setting). All cells settled $5/5$; transient $5/5$ except $(d{=}8,\
\omega{=}10)$ with $2/5$ --- the below-threshold cell excluded from the
slope fits of Table~\ref{tab:supp-conv-summary}.}
\label{tab:supp-conv-s1}
\begin{tabular}{@{}rccccc@{}}
\toprule
$d$ & $\omega{=}10$ & $\omega{=}31.6$ & $\omega{=}100$ & $\omega{=}316$ & $\omega{=}10^3$\\
\midrule
2  & $2.15{\pm}0.28$; $120$ & $2.13{\pm}0.17$; $375$ & $2.22{\pm}0.29$; $1237$ & $2.05{\pm}0.08$; $3604$ & $2.43{\pm}0.19$; $13513$\\
4  & $1.68{\pm}0.31$; $94$  & $2.30{\pm}0.17$; $405$ & $2.41{\pm}0.19$; $1342$ & $2.32{\pm}0.13$; $4095$ & $2.47{\pm}0.19$; $13734$\\
8  & $0.46{\pm}0.92$; $26$  & $2.32{\pm}0.13$; $409$ & $2.29{\pm}0.12$; $1277$ & $2.46{\pm}0.14$; $4326$ & $2.46{\pm}0.07$; $13709$\\
16 & $2.15{\pm}0.04$; $120$ & $1.31{\pm}0.11$; $232$ & $2.39{\pm}0.07$; $1333$ & $2.49{\pm}0.12$; $4380$ & $2.48{\pm}0.06$; $13835$\\
32 & $2.13{\pm}0.05$; $118$ & $2.21{\pm}0.15$; $389$ & $2.17{\pm}0.05$; $1209$ & $2.49{\pm}0.12$; $4378$ & $2.51{\pm}0.08$; $14001$\\
\bottomrule
\end{tabular}
\end{table}

\subsection{Matched-budget comparison with gradient-based and
derivative-free samplers}
\label{supp:num-baselines}

\paragraph{Protocol.} The comparison is designed around two symmetries
identified during pilot experiments. \emph{Equal effective sample size:}
metrics are computed on the \emph{final states} of $1024$ chains for every
method --- no trajectory pooling, whose effective sample size depends on each
method's decorrelation time. The finite-sample floor of the metric (sliced-$W_2$ between an independent $1024$-sample draw of the target and the reference set;
$0.245$ at $d{=}2$, $0.110$ at $d{=}10$) is reported, and entries within
$\sim\!15\%$ of it are marked unresolved. \emph{Symmetric tuning:} every
method's hyperparameters are chosen by an identical pilot on a held-out seed
(seed $0$; evaluation uses seeds $1$--$5$): step size
$h\in\{3{\times}10^{-3},10^{-2},3{\times}10^{-2},10^{-1}\}$ for the Langevin
samplers, smoothing radius $\nu\in\{10^{-3},10^{-2},10^{-1}\}$ for the
zeroth-order and finite-difference baselines, and the time/frequency split
$T\in\{2,4,8,16,32\}$ (with $\omega=2\pi\,\mathrm{NFE}/(c\cdot1.75\,T)$
implied by the budget, $c=10$, and $\omega$ constrained above the measured
threshold) for ES. All methods target $p$ of \eqref{eq:supp-gmm}; the ES
rows use $g=\sqrt2$, $\sigma=1$, so that the averaged dynamics are the
standard-speed Langevin diffusion.

\paragraph{Methods and per-step evaluation accounting.} Oracle ULA,
$\x\leftarrow\x+h\nabla\log p+\sqrt{2h}\,\boldsymbol\xi$ ($1$ gradient);
oracle MALA, the same proposal with Metropolis correction ($1$ gradient);
ZO-ULA with the two-point Gaussian-smoothing estimate
$\hat{\mathbf g}=-\mathbf u\,[J(\x{+}\nu\mathbf u)-J(\x{-}\nu\mathbf u)]/(2\nu)$,
$\mathbf u\sim\mathcal N(0,I)$ \citep{nesterov2017random,roy2022zeroth}
($2$ evaluations); FD-ULA with forward coordinate differences
($d{+}1$ evaluations); and ES ($1$ evaluation, all coordinates). Each method
receives the same per-chain evaluation budget; its step count is the budget
divided by its per-step cost.

\begin{table}[h]
\centering\small
\setlength{\tabcolsep}{4pt}
\caption{Sliced-$W_2$ to ground truth on final states of $1024$ chains
(mean$\pm$std, seeds $1$--$5$). Metric floors: $0.245$ ($d{=}2$), $0.110$
($d{=}10$); entries within $15\%$ of the floor are at metric resolution. The
ES pilot selected $(T,\omega)=(8,90),(4,1795),(4,180),(16,449)$ for the four
columns, respectively. All methods recovered $4/4$ modes in every run, and
unbiased RBF-MMD$^2$ values were zero within noise throughout.}
\label{tab:supp-baselines}
\begin{tabular}{@{}lcccc@{}}
\toprule
& \multicolumn{2}{c}{$d=2$} & \multicolumn{2}{c}{$d=10$}\\
\cmidrule(lr){2-3}\cmidrule(lr){4-5}
Method & NFE $2{\times}10^3$ & NFE $2{\times}10^4$ & NFE $2{\times}10^3$ & NFE $2{\times}10^4$\\
\midrule
oracle ULA (1 grad)   & $0.191\pm0.038$ & $0.215\pm0.069$ & $0.114\pm0.011$ & $0.110\pm0.018$\\
oracle MALA (1 grad)  & $0.213\pm0.056$ & $0.239\pm0.053$ & $0.113\pm0.025$ & $0.102\pm0.008$\\
ZO-ULA (2 ev)         & $0.295\pm0.052$ & $0.218\pm0.041$ & $0.123\pm0.010$ & $0.135\pm0.030$\\
FD-ULA ($d{+}1$ ev)   & $0.308\pm0.054$ & $0.259\pm0.095$ & $0.307\pm0.007$ & $0.118\pm0.012$\\
ES, ours (1 ev)       & $0.211\pm0.026$ & $0.232\pm0.042$ & $0.126\pm0.007$ & $0.118\pm0.022$\\
\bottomrule
\end{tabular}
\end{table}

\paragraph{Findings.} ES is statistically indistinguishable from the tuned
exact-gradient oracles in all four settings, at one evaluation per step; it
resolves ahead of both derivative-free baselines at
$d{=}2$, $\mathrm{NFE}{=}2{\times}10^3$; and the one clearly resolved
deficit in the table is FD-ULA at $d{=}10$, $2{\times}10^3$, whose $d{+}1$
evaluations per step leave only $\mathrm{NFE}/11$ steps and show as
incomplete burn-in (mean final energy $20.4$ vs.\ $\approx15.7$ for all
other methods). We state the scope plainly: a smooth, low-dimensional
testbed with every sampler tuned --- the table supports parity at matched
budget, not superiority.

\subsection{The mode-seeking configuration: a fixed-target benchmark against
derivative-free optimizers under measurement noise}
\label{supp:num-optim}

The comparison of \ref{supp:num-baselines} evaluates the \emph{sampling}
configuration ($\sigma=1$: the averaged dynamics are the Langevin diffusion
and the metric is distributional). This subsection evaluates the
complementary \emph{mode-seeking} configuration: the diffusion is taken to
(numerically) zero, so that by the tempering corollary
(Corollary~\ref{cor:tempering}, stationary law $\propto p^{\,k\alpha/\sigma^2}$)
the averaged dynamics degenerate to the gradient flow
$\dot{\bar\x}=-\tfrac{k\alpha}{2}\nabla\E(\bar\x)$ and the algorithm is
being used as a derivative-free \emph{optimizer} of the energy
$\E=-\log p$. This is the operating mode relevant to the tracking
experiments, where the
objective is to sit on the moving minimizer rather than to cover the
posterior, and it is the mode in which the noise-robustness claim of
Sec.~\ref{sec:practice} (``measurement noise on $J$ enters only through the
dither phase'') can be tested quantitatively. The benchmark below therefore
answers three questions on the testbed \eqref{eq:supp-gmm}: (i)~at matched
total evaluation budgets, how does ES compare with the standard
derivative-free optimization baselines --- simultaneous-perturbation
stochastic approximation (SPSA)~\citep{spall1992spsa}, Gaussian-smoothing
(GS) zeroth-order (ZO) gradient estimates
\citep{nesterov2017random,ghadimi2013stochastic,duchi2015optimal,salimans2017es},
coordinate-wise finite differences (FD)~\citep{kiefer1952stochastic}, and
adaptive random search~\citep{rechenberg1973evolutionsstrategie} --- when
every evaluation of $\E$ is corrupted by i.i.d.\ (independent and
identically distributed) measurement noise; (ii)~\emph{why} ES behaves as it
does under noise, via a derived and validated error-floor law; and
(iii)~what fraction of each method's residual error is recoverable by
cost-free iterate averaging, which isolates the coherent (dither-ripple)
character of the ES residual.

\paragraph{Why a fixed-target protocol.} Distributional metrics cannot
score this regime: for the temperature-one target, the stationary energy gap
is $\chi^2$-distributed to excellent accuracy (the modes are $\ge6$
component widths apart), so a perfect \emph{sampler} has
$\Expec\,\Delta\E=d/2$ by equipartition and essentially never enters a fixed
energy ball --- e.g.\ $\Prob(\Delta\E<\tfrac12)=\Prob(\chi^2_{16}<1)\approx
6\times10^{-8}$ at $d=16$. Optimizers and samplers must therefore be
benchmarked on different axes, and here we adopt the fixed-target
(hitting-time) methodology standard in black-box optimization
\citep{hansen2021coco}: fix an accuracy target, record when each chain first
reaches it, and report quantiles of the hitting-time distribution. The
Langevin samplers of \ref{supp:num-baselines} are correspondingly excluded
from this table by construction, not by defeat.

\paragraph{Success criterion, dwell, and hitting times.} With
$\E^\star=\min_m \E(\bmu_m)$ (all four modes are equivalent), define the
energy gap
\begin{equation}
\Delta\E(\x)\;=\;\E(\x)-\E^\star,
\qquad
\Delta\E(\x)\;\approx\;\tfrac12\,\|\x-\bmu_{m(\x)}\|^2
\ \ \text{near a mode},
\label{eq:supp-opt-gap}
\end{equation}
so the target $\Delta\E<\varepsilon$ with $\varepsilon=\tfrac12$ is a ball
of radius $\approx\sqrt{2\varepsilon}=1$ (one component width) around some
mode, in every dimension. Let $e$ count the noisy oracle evaluations a
method has consumed for a given chain (its only clock), and let $\x(e)$ be
that chain's state, piecewise constant between the method's updates. The
per-chain hitting time is the first entry into a visit that \emph{persists},
\begin{equation}
\tau \;=\; \min\Big\{\,e \;:\; \Delta\E\big(\x(e')\big)<\varepsilon
\ \ \text{for all}\ e'\in[e,\,e+D]\,\Big\},
\qquad D=20\ \text{evaluations},
\label{eq:supp-opt-hit}
\end{equation}
where the dwell requirement $D$ disqualifies transits in which a large
dither excursion merely slices through the ball without settling (without
it, hitting times can be gamed by amplitude). We report the success rate
$\Prob(\tau\le B)$ over chains at the total per-chain budget $B$, and the
median hitting time $E_{50}$ (the empirical $50\%$ quantile of $\tau$,
in evaluations), quoting $E_{50}$ only when at least half the chains
succeed. Scoring uses the \emph{exact} $\E$ --- available because the
testbed is synthetic --- and is not charged to any budget; the methods
themselves never see it.

\paragraph{Noisy oracle.} Every method interacts with the energy only
through the corrupted oracle
\begin{equation}
\widehat\E(\x)\;=\;\E(\x)+\sigma_J\,n,
\qquad n\sim\mathcal N(0,1)\ \text{i.i.d.\ per evaluation and per chain},
\label{eq:supp-opt-noise}
\end{equation}
at noise levels $\sigma_J\in\{0.03,\,0.1,\,0.3,\,1.0\}$; for reference, the
initial gaps are $\Delta\E\approx6d$ and the target is
$\varepsilon=\tfrac12$, so a single read at $\sigma_J=0.1$ already carries noise equal to $20\%$ of the
target resolution. This is the synthetic stand-in for a
mini-batch or measurement-limited evaluation of a model's energy.

\paragraph{Readout: last iterate vs.\ iterate averaging.} Under a noisy
oracle, scoring the final iterate of a stochastic trajectory is a
high-variance convention; the standard remedy in stochastic approximation is
Polyak--Ruppert iterate averaging
\citep{ruppert1988efficient,polyak1992acceleration}, whose deep-learning
analogue is weight averaging \citep{izmailov2018averaging}. We therefore
maintain, purely as a \emph{readout}, an exponential moving average (EMA) in
evaluation time,
\begin{equation}
\bar\x \;\leftarrow\; \bar\x + \lambda\,(\x-\bar\x),
\qquad \lambda=\min(1,\,\Delta e/\tau_{\mathrm{f}}),
\qquad \tau_{\mathrm{f}}=300\ \text{evaluations},
\label{eq:supp-opt-ema}
\end{equation}
updated whenever the method updates ($\Delta e$ = evaluations since the last
update), and score $\Delta\E(\bar\x)$. Three properties keep this fair:
it is applied \emph{identically} to every method; it costs zero
evaluations (it reuses the trajectory already paid for); and it is
feedback-free ($\bar\x$ never enters any method's dynamics). Its one side
effect is a warm-up lag: the average retains weight
$\approx e^{-(B-e_{\mathrm{conv}})/\tau_{\mathrm{f}}}$ on pre-convergence
states, so all EMA hitting times carry a common offset of order
$\tau_{\mathrm{f}}\ln(\Delta\E_0/\varepsilon)$, and a method that converges
only near the end of its budget can show a \emph{larger} EMA error than raw
error. The clean-gradient oracle row below, which converges in $\le40$
updates, measures this lag directly ($E_{50}=514$--$1211$ across $d$) and
serves as the ruler; entries dominated by warm-up rather than by an error
floor are marked $\dagger$.

\paragraph{Baselines: updates, per-update cost, and tuning.} All baselines
consume the oracle \eqref{eq:supp-opt-noise} and pay their true evaluation
price per update; their step counts are the budget divided by that price.
\begin{itemize}
\item \emph{Adaptive random search} (the $(1{+}1)$ evolution strategy, i.e.\
one-parent/one-offspring; equivalently zero-temperature random-walk
Metropolis with step adaptation). Proposal
$\x'=\x+s\,\boldsymbol\xi$, $\boldsymbol\xi\sim\mathcal N(0,I_d)$; accept
iff $\widehat\E(\x')<\widehat\E_{\mathrm{inc}}$; per-chain step adaptation
by the $1/5$th-success rule
$s\leftarrow s\,e^{\eta(\mathbf 1_{\mathrm{acc}}-1/5)}$
\citep{rechenberg1973evolutionsstrategie}, with $s_0=0.5$, $\eta=0.1$.
\emph{Cached-incumbent} variant: the incumbent value is measured once and
stored ($1$ evaluation per proposal). \emph{Re-evaluating} variant: the
incumbent is re-measured at every comparison ($2$ evaluations per proposal),
the standard remedy for the elitist noise pathology
\citep{arnold2002local} in which a single lucky noise draw makes the cached
incumbent unbeatable.
\item \emph{FD-GD} (central-difference gradient descent)
\citep{kiefer1952stochastic}:
\begin{equation}
\hat g_i=\frac{\widehat\E(\x+h\mathbf e_i)-\widehat\E(\x-h\mathbf e_i)}{2h},
\qquad \x\leftarrow\x-\eta_{\mathrm{g}}\,\hat{\mathbf g},
\label{eq:supp-opt-fd}
\end{equation}
$2d$ evaluations per update, $\eta_{\mathrm{g}}=0.1$. The divided difference
amplifies oracle noise to $\sigma_J/(\sqrt2\,h)$ per gradient component, so
$h$ must scale with the noise: we use $h=0.5$ (the bias is $O(h^2)$ and
negligible on the near-quadratic modes); $h=10^{-4}$, appropriate for a
clean oracle, diverges immediately at any $\sigma_J>0$.
\item \emph{ZO-GD} (GS gradient estimate with $m$ antithetic pairs)
\citep{nesterov2017random,salimans2017es}:
\begin{equation}
\hat{\mathbf g}=\frac1m\sum_{j=1}^{m}\mathbf u_j\,
\frac{\widehat\E(\x+\nu\mathbf u_j)-\widehat\E(\x-\nu\mathbf u_j)}{2\nu},
\qquad \mathbf u_j\sim\mathcal N(0,I_d),
\label{eq:supp-opt-gs}
\end{equation}
then $\x\leftarrow\x-\eta_{\mathrm{g}}\hat{\mathbf g}$; $2m$ evaluations per
update, $m=10$, $\eta_{\mathrm{g}}=0.05$,
$\nu=\max(0.1,\sigma_J)$ (probe scaled to the noise, the standard
bias--variance remedy).
\item \emph{Two-point Gaussian-smoothing estimator (SPSA-style)}: \eqref{eq:supp-opt-gs} with
$m=1$, the one-random-direction scheme in the spirit of \citet{spall1992spsa} (classical SPSA uses Rademacher directions, which satisfy its bounded-inverse-moment condition, and decaying gains; our variant uses Gaussian directions and constant gains, so it is a two-point Gaussian-smoothing gradient estimator rather than SPSA proper). The
single-probe estimator is unbiased but heavy-tailed: on a unit-curvature
quadratic, one step contracts the squared error by
$1-2\eta_{\mathrm{g}}+\eta_{\mathrm{g}}^2(d+2)$, so stability requires
$\eta_{\mathrm{g}}<2/(d+2)$ with one-step optimum
$\eta_{\mathrm{g}}=1/(d+2)$; we use $\eta_{\mathrm{g}}=1/d$ throughout,
which is essential --- a $d$-independent step ($\eta_{\mathrm{g}}=0.05$)
diverges for $d\ge64$. With this tuning SPSA is the strongest
derivative-free baseline in the tables below, consistent with the
information-theoretic optimality of two-point estimators for this query
model \citep{duchi2015optimal,jamieson2012query}.
\item \emph{Clean-gradient descent} (reference only):
$\x\leftarrow\x+\eta_{\mathrm{g}}\nabla\log p(\x)$, $\eta_{\mathrm{g}}=0.1$,
using the \emph{exact, noiseless} gradient. It is not purchasable in the
zeroth-order query model at any price and is reported as an upper bound and
as the EMA-lag ruler.
\end{itemize}

\paragraph{ES configuration.} Algorithm~\ref{alg:es} in the convention
\eqref{eq:supp-conv} with $g=1$, dither amplitude $\alpha=1.5$, base
frequency $\omega=100$, band $N=1.75$ (ratios $r_i$ evenly spaced in
$[1,1.75]$), $c=5$ steps per fastest dither period, hence
$\Delta t=2\pi/(cN\omega)\approx7.18\times10^{-3}$ and one oracle
evaluation per step; diffusion $\sigma=10^{-3}$ (the numerical mode-seeking
limit); and the dimension-dependent feedback gain
\begin{equation}
k(d)\;=\;\frac{4}{\sqrt d}.
\label{eq:supp-opt-es}
\end{equation}
The scaling \eqref{eq:supp-opt-es} is forced by the same mechanism as the
threshold \eqref{eq:supp-threshold}: the phase carries $k\E$, and $\E$ is
extensive --- both its fluctuations under the target
($\mathrm{std}\approx\sqrt{d/2}$) and its per-step increments far from the
modes grow with $d$ --- so at fixed $k$ the per-step phase modulation grows
with dimension until it scrambles the demodulation. Keeping
$k\sqrt d$ constant keeps it bounded. Furthermore: (i)~$c=5$ sits below the $c\ge10$ recommendation of Sec.~\ref{sec:practice}, which concerns the
trajectory-error \emph{rate measurement}; here the criterion is the
end-task hitting time and the setting is flagged rather than re-derived;
(ii)~the initialization is $\x_0\sim\mathrm{Unif}([-6,6]^d)$ (dispersed
across all four basins; initial gap $\Delta\E_0\approx6d$), not the
Gaussian initialization of \ref{supp:num-testbed}. We run both the base
variant ($d$ distinct frequencies) and the \emph{quadrature-pair} variant of
\eqref{eq:supp-track-sde} ($\lceil d/2\rceil$ carriers, each driving a
$(\cos,\sin)$ coordinate pair at phases $0,\pi/2$).

\paragraph{The averaged dynamics predict the ES constants.} On the
near-quadratic basins, the averaged flow gives
$\tfrac{d}{dt}\Delta\E=-k\alpha\,\Delta\E$, hence the parameter-free
prediction
\begin{equation}
E_{50}\;\approx\;\frac{\ln(\Delta\E_0/\varepsilon)}{k\alpha\,\Delta t}
\;=\;\frac{\ln(12d)\,\sqrt d}{4\cdot1.5\cdot\Delta t}.
\label{eq:supp-opt-rate}
\end{equation}
Against the clean-oracle, last-iterate measurement (single replicate, $2000$
chains): predicted vs.\ observed $E_{50}=180$/$155$, $488$/$502$,
$781$/$856$, $1234$/$1493$ at $d=4,16,32,64$ --- the discretized dither runs
at the speed of its own averaged flow to within $\pm15\%$ for $d\le32$, with
a $21\%$ finite-gain drag appearing at $d=64$. Under noise the \emph{rate}
is nearly unchanged, as the phase-averaging claim of
Sec.~\ref{sec:practice} predicts: from $\sigma_J=0$ to $0.1$ the raw
$E_{50}$ moves $155\to175$, $502\to536$, $856\to925$ at $d=4,16,32$
($+13/{+}7/{+}8\%$; single matched-stream replicates), while --- for
contrast --- the same noise destroys both random-search variants at
$d\ge32$ (Table~\ref{tab:supp-opt-s01}).

\paragraph{Budgets, replication, and evaluation accounting.} The per-chain
budget is identical for every method: $B=2000$ evaluations for $d\le64$,
$4000$ at $d=128$, $8000$ at $d=256$. Headline tables are mean$\pm$sd
(standard deviation) over $5$ replicates of $500$ chains; a replicate
re-draws the initialization, every method's internal randomness, and the
oracle noise stream, while \emph{within} a replicate all methods share the
same initialization and the same noise stream (common random numbers, so
method-to-method comparisons are paired). Total evaluations is the fair
currency for ES specifically because its evaluations are inherently
\emph{sequential} (each step's evaluation feeds the next state), whereas the
$2m$ probes of ZO-GD, the $2d$ of FD, and the $2$ of SPSA are embarrassingly
parallel across evaluators; wall-clock comparisons on parallel hardware
would favor the batched estimators, and we say so rather than hide it. The
converse structural advantage of the one-evaluation-per-step design ---
every coordinate updates after every single measurement, with no staleness
window --- is exercised by the time-varying benchmarks of
\ref{supp:num-celeba}, not by this static one.

\begin{table}[h]
\centering\small
\setlength{\tabcolsep}{3.2pt}
\caption{Fixed-target benchmark at $\sigma_J=0.03$, EMA readout
\eqref{eq:supp-opt-ema}: success rate in $\%$ (mean$\pm$sd over $5$
replicates of $500$ chains; sd $0.0$ omitted) with the median hitting time
$E_{50}$ in evaluations in parentheses (seed-to-seed sd of $E_{50}$
$\le18$ evaluations in every entry shown); ``---'' where fewer than half
the chains succeed. Budget $B=2000$ ($d\le64$), $4000$ ($d{=}128$), $8000$
($d{=}256$). All EMA hitting times include the common filter warm-up
measured by the clean-gradient row (which converges in $\le40$ updates).
$\dagger$: warm-up artifact of the filtered readout only --- under the
last-iterate readout FD-GD succeeds at $d{=}32$ ($99.9\%$, final gap
$0.267$; Table~\ref{tab:supp-opt-readout}) but its
$\lfloor B/2d\rfloor=31$ updates arrive too late in the budget for the
$\tau_{\mathrm{f}}=300$ filter to settle; at $d\ge64$ FD-GD fails under
both readouts.}
\label{tab:supp-opt-s003}
\scriptsize
\resizebox{\textwidth}{!}{%
\begin{tabular}{@{}lcccccc@{}}
\toprule
Method (evals/update) & $d{=}4$ & $d{=}16$ & $d{=}32$ & $d{=}64$ & $d{=}128$ & $d{=}256$\\
\midrule
rand.\ search, cached ($1$)  & $100$\,(526) & $99.1{\pm}0.3$\,(884) & $39.3{\pm}2.6$\,(---) & $0$ & $0$ & $0$\\
rand.\ search, re-eval ($2$) & $100$\,(530) & $100$\,(908) & $100$\,(1346) & $0$ & $0$ & $0$\\
FD-GD ($2d$)                 & $100$\,(589) & $100$\,(1248) & $0.2{\pm}0.1$\,(---)$^\dagger$ & $0$ & $0$ & $0$\\
ZO-GD, $m{=}10$ ($20$)       & $100$\,(1060) & $100$\,(1460) & $100$\,(1660) & $79.0{\pm}0.8$\,(1920) & $100$\,(2440) & $100$\,(4784)\\
SPSA, $\eta_{\mathrm g}{=}1/d$ ($2$) & $100$\,(526) & $100$\,(811) & $100$\,(967) & $100$\,(1246) & $100$\,(2111) & $100$\,(4303)\\
bounded ES ($1$)             & $100$\,(608) & $100$\,(1036) & $99.8{\pm}0.1$\,(1322) & $95.9{\pm}0.7$\,(1797) & $99.4{\pm}0.4$\,(3084) & $47.7{\pm}1.0$\,(---)\\
bounded ES, quadrature ($1$) & $100$\,(580) & $100$\,(1002) & $100$\,(1278) & $99.1{\pm}0.7$\,(1723) & $99.9{\pm}0.1$\,(2927) & $78.0{\pm}0.9$\,(7506)\\
\midrule
clean-grad.\ descent (oracle) & $100$\,(514) & $100$\,(782) & $100$\,(893) & $100$\,(1000) & $100$\,(1106) & $100$\,(1211)\\
\bottomrule
\end{tabular}}
\end{table}

\begin{table}[h]
\centering\small
\setlength{\tabcolsep}{3.2pt}
\caption{As Table~\ref{tab:supp-opt-s003}, at $\sigma_J=0.1$ (seed-to-seed
sd of $E_{50}$ $\le16$ evaluations except re-evaluating random search at
$d{=}32$, $\pm48$).}
\label{tab:supp-opt-s01}
\scriptsize
\resizebox{\textwidth}{!}{%
\begin{tabular}{@{}lcccccc@{}}
\toprule
Method (evals/update) & $d{=}4$ & $d{=}16$ & $d{=}32$ & $d{=}64$ & $d{=}128$ & $d{=}256$\\
\midrule
rand.\ search, cached ($1$)  & $99.8{\pm}0.2$\,(540) & $30.5{\pm}1.7$\,(---) & $0$ & $0$ & $0$ & $0$\\
rand.\ search, re-eval ($2$) & $100$\,(530) & $100$\,(948) & $56.4{\pm}3.5$\,(1878) & $0$ & $0$ & $0$\\
FD-GD ($2d$)                 & $100$\,(586) & $100$\,(1248) & $0.1{\pm}0.1$\,(---)$^\dagger$ & $0$ & $0$ & $0$\\
ZO-GD, $m{=}10$ ($20$)       & $100$\,(1060) & $100$\,(1464) & $100$\,(1672) & $74.0{\pm}0.9$\,(1940) & $100$\,(2480) & $100$\,(5228)\\
SPSA, $\eta_{\mathrm g}{=}1/d$ ($2$) & $100$\,(532) & $100$\,(818) & $100$\,(980) & $100$\,(1269) & $100$\,(2173) & $100$\,(4485)\\
bounded ES ($1$)             & $100$\,(620) & $100$\,(1054) & $99.8{\pm}0.2$\,(1349) & $90.7{\pm}1.2$\,(1850) & $97.4{\pm}1.2$\,(3328) & $2.5{\pm}1.0$\,(---)\\
bounded ES, quadrature ($1$) & $100$\,(590) & $100$\,(1016) & $100$\,(1300) & $98.0{\pm}0.7$\,(1768) & $99.7{\pm}0.2$\,(3104) & $12.5{\pm}1.3$\,(---)\\
\midrule
clean-grad.\ descent (oracle) & $100$\,(514) & $100$\,(782) & $100$\,(893) & $100$\,(1000) & $100$\,(1106) & $100$\,(1211)\\
\bottomrule
\end{tabular}}
\end{table}

\paragraph{Findings at low and moderate noise
(Tables~\ref{tab:supp-opt-s003}--\ref{tab:supp-opt-s01}).} Tuned SPSA is
the fastest derivative-free method at every dimension, as its query-model
optimality leads one to expect; we state that plainly. The ES result is of
a different kind: it reaches the target at \emph{every} $d\le128$ at both
noise levels (quadrature variant: success $\ge97.4\%$ everywhere), with
$E_{50}$ within $1.1$--$1.5\times$ of SPSA's ($1.10$--$1.43\times$ for the
quadrature variant across $d=4$--$128$), while operating under the strictly
sequential one-evaluation-per-step constraint that is the point of the
design. Every other baseline breaks somewhere in this range: cached-incumbent
random search noise-locks (its stored incumbent value is a lucky order
statistic that fresh proposals cannot beat \citep{arnold2002local}) already
at $(d,\sigma_J)=(32,0.03)$; the re-evaluating repair pays $2\times$ per
step and still fails for $d\ge64$, because accepting on noisy comparisons
turns the elitist walk into a diffusion whose drift vanishes once the
per-step improvement falls below the $\sqrt2\,\sigma_J$ comparison noise;
FD-GD is budget-starved by its $2d$-evaluation tax ($\le15$ updates at
$d\ge64$) on top of its noise amplification; and ZO-GD with $m=10$ is
squeezed by budget quantization ($\lfloor B/20\rfloor$ updates --- visible
as the $74$--$79\%$ success at $d=64$, where it affords only $100$ updates).
The quadrature variant is never worse than the base variant and is
consistently better in the tails (e.g.\ $99.1$ vs.\ $95.9\%$ at
$(64,0.03)$; $78.0$ vs.\ $47.7\%$ at $(256,0.03)$), consistent with its
elimination of the nearest-beat cross terms at fixed band. At $d=256$ the
current configuration reaches its measured boundary: the dither-envelope
floor derived below ($\approx0.0028\,d=0.72$ at $d=256$) exceeds
$\varepsilon$ under the raw readout, the gain schedule
\eqref{eq:supp-opt-es} crosses the finite-gain ceiling discussed below, and
$E_{50}=7506$ presses the $B=8000$ budget --- three compounding effects,
all of which the frequency-scaling prescription of
\ref{supp:num-threshold} (raise $\omega$ with $d$) is designed to relieve;
we report the boundary rather than re-tune past it.

\paragraph{High-noise regime and the limits of the comparison.} At
$\sigma_J=0.3$ the ES success degrades ($98.3/94.7\%$ at $d=16$ for the
quadrature/base variants, $48.1{\pm}4.4$/$31.3{\pm}2.6\%$ at $d=32$, zero
beyond), and at $\sigma_J=1$ ES fails at all $d$ while SPSA and ZO-GD
retain $100\%$ success with essentially $\sigma$-independent hitting times.
Two mechanisms, one per side. On the ES side, the per-evaluation phase
jitter is $k\sigma_J$ radians: once it approaches $O(1)$ the demodulation
itself decoheres --- with \eqref{eq:supp-opt-es}, $k\sigma_J\approx0.71$ at
$(d{=}32,\sigma_J{=}1)$ --- and, below that, the injected noise floor
derived next crosses $\varepsilon$. On the baseline side, the smoothing
estimators were granted the standard noise adaptation $\nu=\max(0.1,\sigma_J)$,
which on this testbed is free of charge: the modes are near-quadratic, so
enlarging the probe adds no smoothing bias --- a testbed-specific gift that
would expire on landscapes with curvature variation at the probe scale. ES,
by contrast, ran a single gain schedule \eqref{eq:supp-opt-es} across all
$\sigma_J$; the symmetric remedy --- reducing $k$ with $\sigma_J$ to hold
$k\sigma_J$ bounded, the ES analogue of SPSA's noise-scaled probe --- is
the natural extension and is left to future work. The honest summary of
the regime map: at $k\sigma_J\lesssim0.3$ the lock-in structure delivers
near-noiseless rates; well beyond it, noise-adapted stochastic smoothing
dominates.

\begin{table}[h]
\centering\small
\setlength{\tabcolsep}{4.5pt}
\caption{Identical-trajectory readout study at $d=32$, $B=2000$: final
median gap $\Delta\E$ under the last-iterate (raw) readout $\to$ under the
EMA readout \eqref{eq:supp-opt-ema}, for the \emph{same} runs (the two
scorers observe one trajectory per method per replicate). Mean over $5$
replicates of $500$ chains; seed-to-seed sd $\le0.01$ for entries below
$1$ and $\le3\%$ of the mean for entries above $1$. The raw$\to$EMA ratio
measures each method's \emph{recoverable} error --- zero-mean scatter about
the optimum --- as opposed to dynamical failure, which no readout repairs.
$\dagger$: EMA exceeds raw due to residual filter warm-up weight
$e^{-(B-e_{\mathrm{conv}})/\tau_{\mathrm{f}}}$ on pre-convergence states
(FD-GD and ZO-GD converge late in their update-sparse budgets), not an
error floor.}
\label{tab:supp-opt-readout}
\begin{tabular}{@{}lcccc@{}}
\toprule
Method & $\sigma_J{=}0.03$ & $\sigma_J{=}0.1$ & $\sigma_J{=}0.3$ & $\sigma_J{=}1.0$\\
\midrule
rand.\ search, cached  & $0.575\to0.572$ & $1.91\to1.90$ & $5.66\to5.64$ & $19.0\to19.0$\\
rand.\ search, re-eval & $0.532\to0.177$ & $1.78\to0.572$ & $5.37\to1.73$ & $18.0\to5.82$\\
FD-GD                  & $0.267\to0.756^\dagger$ & $0.282\to0.766^\dagger$ & $0.413\to0.856^\dagger$ & $1.91\to1.89$\\
ZO-GD, $m{=}10$        & $0.017\to0.106^\dagger$ & $0.037\to0.117^\dagger$ & $0.037\to0.117^\dagger$ & $0.039\to0.123^\dagger$\\
SPSA                   & $0.023\to0.005$ & $0.260\to0.045$ & $0.260\to0.045$ & $0.264\to0.047$\\
bounded ES             & $0.113\to0.015$ & $0.224\to0.070$ & $1.25\to0.609$ & $25.8\to21.6$\\
bounded ES, quadrature & $0.108\to0.009$ & $0.218\to0.060$ & $1.25\to0.563$ & $25.3\to21.1$\\
\bottomrule
\end{tabular}
\end{table}

\paragraph{The ES residual is coherent
(Table~\ref{tab:supp-opt-readout}).} The readout study answers a question
the headline tables cannot: \emph{what kind} of error does each method leave
behind? Averaging the trajectory can only remove zero-mean scatter about
the right place; it does nothing for a chain that never arrived. The table
shows exactly this taxonomy. The noise-locked random search is untouched by
the filter ($0.575\to0.572$): its error is dynamical, and the experiment
thereby also certifies that the noise is \emph{not} being filtered before
the optimizers --- if it were, every method would revert to its clean
numbers. SPSA improves $4.6\times$ at $\sigma_J=0.03$ (its residual is a
weakly correlated stochastic wander). ES improves $7.5\times$ (base) and
$12.0\times$ (quadrature, $0.108\to0.009$): its residual is dominated by
the deterministic dither ripple --- a zero-mean, phase-locked oscillation of
known period ($5$--$9$ steps here) that any window much longer than a beat
cancels coherently. This is the same structure exploited by a lock-in
amplifier's output low-pass filter \citep{scofield1994lockin}: the dithered
phase demodulation of Sec.~\ref{sec:practice} is the mixer, and
\eqref{eq:supp-opt-ema} is the missing output stage. The practical
consequence is that the correct deployment of the mode-seeking
configuration reports the averaged iterate, at zero additional evaluations
--- which is also what converts the $d=128$ column of
Table~\ref{tab:supp-opt-s01} from a raw-readout floor sitting near
$\varepsilon$ into a $97$--$100\%$ success rate.

\paragraph{A quantitative error-floor model.} The stationary gap of the ES
iterates (raw readout) decomposes into three terms:
\begin{equation}
F(d,\sigma_J)\;\approx\;
\underbrace{\frac{\alpha}{4}\sum_{i=1}^{d}\frac1{\omega_i}}_{F_{\mathrm{osc}}
\,\approx\,0.0028\,d}
\;+\;F_{\mathrm{dec}}(k,d)
\;+\;
\underbrace{C_n\,k\,\sigma_J^2\,d}_{F_{\mathrm{noise}}},
\label{eq:supp-opt-floor}
\end{equation}
in which each term has a mechanism. $F_{\mathrm{osc}}$ is the energy stored
in the dither ripple itself: to leading order
$x_i\approx\bar x_i+\sqrt{\alpha/\omega_i}\sin(\omega_i t+k\E)$
(cf.\ \ref{supp:warmup}), so each coordinate carries variance
$\alpha/(2\omega_i)$ and, on the unit-curvature basin
\eqref{eq:supp-opt-gap}, gap $\alpha/(4\omega_i)$; for our band this is
$0.00280\,d$, independent of $k$. $F_{\mathrm{noise}}$ is
noise-injected heating: each evaluation kicks the phase by
$k\sigma_J n$, hence the velocity by
$\approx-\sqrt{\alpha\omega_i}\,\sin(\omega_it+k\E)\,k\sigma_J n$, an
injected position variance of
$\tfrac12\alpha\omega_i(k\sigma_J)^2\Delta t^2$ per step per coordinate;
balancing against the averaged relaxation (rate $k\alpha/2$ in position)
gives $F_{\mathrm{noise}}\approx\tfrac14\bar\omega\Delta t\,k\sigma_J^2 d$
with $\bar\omega\Delta t\approx0.99$ here, i.e.\ the stated form with
$C_n\approx0.25$ predicted; the constant fitted once, globally, is
$C_n=0.55$ --- the derivation captures the functional form (linear in $k$
and $d$, quadratic in $\sigma_J$) and the order of the constant.
$F_{\mathrm{dec}}$ is the finite-gain decoherence residual, the term the
schedule \eqref{eq:supp-opt-es} is chosen to suppress; empirically it is
the small remainder of the clean floors above $F_{\mathrm{osc}}$ (below)
and scales $\propto k^2$ across the two gain settings we tested at
$d=128$. Table~\ref{tab:supp-opt-floor} confronts
\eqref{eq:supp-opt-floor} with measurement: the noise-induced excess
matches the derived law within $3\%$ for $32\le d\le128$ at
$\sigma_J=0.1$, and within its error bars across
$\sigma_J\in\{0.03,0.1,0.3\}$ at $d=32$; at $\sigma_J=1$ the law
underpredicts by $\approx2\times$, precisely where the small-jitter
expansion behind it fails ($k\sigma_J\approx0.7$ rad) --- the same
boundary as the success collapse above, giving the validity condition
$k\sigma_J\lesssim0.3$ a second, independent measurement.

\begin{table}[h]
\centering\small
\setlength{\tabcolsep}{5pt}
\caption{Validation of the floor model \eqref{eq:supp-opt-floor} under the
raw (last-iterate) readout. Panel (a): dimension scan at $\sigma_J=0.1$
(single matched-stream replicate, $2000$ chains: the clean and noisy runs
share initialization; the gain schedule \eqref{eq:supp-opt-es} is in
force, so the prediction is $0.55\,k(d)\,\sigma_J^2\,d=2.2\,\sigma_J^2\sqrt
d$). Panel (b): noise scan at $d=32$ ($5$ replicates of $500$ chains,
mean$\pm$sd; excess relative to the clean floor $0.102$). The
$\sigma_J{=}1$ row lies outside the small-jitter validity of the law
($k\sigma_J\approx0.71$) and is included to locate the breakdown.}
\label{tab:supp-opt-floor}
\begin{tabular}{@{}rccccc@{}}
\toprule
\multicolumn{6}{@{}l}{(a) dimension scan, $\sigma_J=0.1$}\\
$d$ & clean floor & $F_{\mathrm{osc}}$ pred. & noisy floor &
excess (obs.) & excess (pred.)\\
\midrule
4   & $0.013$ & $0.011$ & $0.048$ & $0.035$ & $0.044$\\
16  & $0.042$ & $0.045$ & $0.124$ & $0.082$ & $0.088$\\
32  & $0.102$ & $0.090$ & $0.223$ & $0.121$ & $0.124$\\
64  & $0.240$ & $0.179$ & $0.416$ & $0.176$ & $0.176$\\
128 & $0.456$ & $0.358$ & $0.712$ & $0.256$ & $0.249$\\
\midrule
\multicolumn{6}{@{}l}{(b) noise scan, $d=32$ ($5$ seeds)}\\
$\sigma_J$ & \multicolumn{2}{c}{noisy floor} &
\multicolumn{2}{c}{excess (obs.)} & excess (pred.)\\
\midrule
0.03 & \multicolumn{2}{c}{$0.113\pm0.001$} & \multicolumn{2}{c}{$0.011$} & $0.011$\\
0.1  & \multicolumn{2}{c}{$0.224\pm0.002$} & \multicolumn{2}{c}{$0.122$} & $0.124$\\
0.3  & \multicolumn{2}{c}{$1.252\pm0.013$} & \multicolumn{2}{c}{$1.150$} & $1.120$\\
1.0  & \multicolumn{2}{c}{$25.8\pm0.26$}   & \multicolumn{2}{c}{$25.7$}  & $12.4$\\
\bottomrule
\end{tabular}
\end{table}

\paragraph{Finite-gain ceiling.} The counterpart of the frequency threshold
\eqref{eq:supp-threshold} on the gain side: sweeping $k$ at fixed
$(\alpha,c,\omega)$ (single replicate, $500$ chains, clean oracle, dither
diffusion $\sigma=0.02$, $B=2000$, raw readout), convergence is reliable up
to a ceiling that halves as $d$ doubles --- at $d=32$, $k=1$ gives
$E_{50}=713$ ($97.0\%$ success; quadrature $696$, $99.2\%$) while $k=2$
collapses to $2.8\%$; at $d=64$, $k=0.5$ gives $E_{50}=1521$ ($99.4\%$)
while $k=1$ collapses to $15.0\%$ ($20.4\%$ quadrature). That is,
$k_{\max}\,d\approx30$--$60$ for this configuration, consistent with the
per-step phase-slew mechanism above ($k\,|\Delta\E\ \text{per step}|$
growing $\propto k\,d$ in the far field). The schedule
\eqref{eq:supp-opt-es} has $k(d)\,d=4\sqrt d$, which sits below this
ceiling for $d\lesssim100$, grazes it near $d\approx128$ (the observed
$21\%$ rate drag), and exceeds it at $d=256$ --- one of the three
compounding boundary effects noted above. Below the ceiling, the rate is
linear in $k$ as \eqref{eq:supp-opt-rate} predicts ($E_{50}=1182\to713$ for
$k=0.5\to1$ at $d=32$, i.e.\ $83\%$ of exact $1/k$ scaling).

\paragraph{Scope.} We state the limits of this benchmark plainly, as in
\ref{supp:num-baselines}. The testbed is smooth with near-quadratic,
unit-curvature basins, which grants the smoothing baselines their free
noise-scaled probes and makes tuned SPSA close to its query-model optimum;
none of the results here claim ES superiority on clean, smooth objectives,
and the parallelism asymmetry noted in the accounting paragraph favors the
batched estimators on parallel hardware. What the benchmark does establish,
with five-seed error bars: (i)~a strictly sequential,
one-evaluation-per-step method reaches the fixed target at every
$d\le128$ under evaluation noise, within $1.1$--$1.5\times$ the evaluation
cost of the strongest (noise-tuned, parallelizable) baseline, where every
other sequential-friendly baseline fails; (ii)~the noise-robustness
mechanism claimed in Sec.~\ref{sec:practice} is quantitative --- rates
within $8$--$13\%$ of noiseless up to $k\sigma_J\approx0.2$, with the
residual cost captured by the derived floor law \eqref{eq:supp-opt-floor}
to a few percent over an order of magnitude in $d$ and in
$\sigma_J^2$; and (iii)~the ES residual is coherent rather than
stochastic, so a zero-cost averaged readout recovers up to $12\times$
accuracy --- the lock-in signature that distinguishes the dither mechanism from stochastic smoothing.

\subsection{CelebA-HQ: steady-state lag and paired comparison}
\label{supp:num-celeba}

This appendix supplements the CelebA-HQ tracking study of Sec.~\ref{sec:tracking}: the first paragraph reads the checkpoint-by-checkpoint NMSE of Table~\ref{tab:supp-celeba-track500} (target drift vs.\ bounded ES) as a quantitative measurement of the tracking mechanism, and the second returns to the matched-budget benchmark of Table~\ref{tab:supp-celeba-df} to explain the ranking it reports.

\paragraph{Steady-state-lag signature.} The checkpoint increments expose
the tracking mechanism quantitatively. Between consecutive checkpoints the
target's NMSE advances by
$\Delta_{\mathrm{drift}}=.053,.068,.075,.084,.095,.110$ (the drift
accelerates), while the tracker's residual advances by
$\Delta_{\mathrm{ES}}=.009,.010,.011,.012,.015,.016$; the ratio
$\Delta_{\mathrm{ES}}/\Delta_{\mathrm{drift}}
=.170,.147,.147,.143,.158,.145$ is constant at $0.152\pm0.010$ across all
six intervals, after a standing-start first checkpoint at level ratio
$0.40$. A constant incremental ratio is the signature of steady-state lag
tracking: the residual is proportional to the target's local speed with a
fixed effective response time, the stochastic analogue of the
deterministic ES tracking bounds of~\citep{scheinker2016bounded}, and it
implies the $6\times$ headline reduction is not a plateau but improves
whenever the target moves more slowly.

Tuning improves the baselines considerably, SPSA most of all, yet bounded ES remains the strongest forward-only method: $1.8\times$
lower final NMSE than best-tuned SPSA, $3.5\times$ lower than FD,
$15\times$ lower than GS-ES, and within $1.7\times$ of the exact-gradient
oracle. Because all methods are run on the same $219$ pairs and seeds, the
comparison is paired: \emph{bounded ES achieves the lowest final NMSE
among the four forward-only methods on $100\%$ of the $219$ pairs} ---
$219/219$ head-to-head wins against each of SPSA, FD, and GS-ES --- while
outperforming even the exact-gradient oracle on $\approx9\%$ of pairs
(exact instantaneous gradients are not optimal for a moving target: Adam's
momentum can overshoot or lag the target's turns, and on those pairs the
ES trajectory ends closer). The ranking is explained by the update rate
against a moving target rather than by per-estimate variance. Bounded ES
commits an update after every single measurement ($8192$ updates); SPSA
after every $2$ ($4096$); GS-ES only after $2m$ consecutive measurements
--- its antithetic pairs are evaluated at different times against a moving
target, so temporal drift contaminates the spatial differences, and the
tuned $m$ trades estimator variance against staleness with no good
operating point at this drift rate: its final NMSE ($1.17$) exceeds the
targets' own final drift ($\approx0.51$, Table~\ref{tab:supp-celeba-track500}),
i.e., it ends farther from the target than the untouched initial image.
Beyond the raw evaluation-count tax (cf.~Table~\ref{tab:supp-baselines}),
coordinate FD suffers the structural staleness penalty in its purest form:
with $d=256$ parameters, central differences must hold the state
essentially fixed for $512$ measurements to assemble one gradient
estimate, so its entire budget purchases $16$ steps, each applying the
gradient of a target that no longer exists --- the estimate answers a
question the system is no longer asking --- whereas ES updates every
coordinate after every single measurement.

\section{CIFAR-10 Additional Details}
\label{supp:CIFAR10}
Latents are standardized per channel to zero mean and standard deviation $0.5$ with the statistics of $45{,}000$ training latents ($5{,}000$ are held out for validation). On the $10{,}000$ test images the VAE reconstructs with
PSNR $=28.2$~dB, SSIM $=0.90$~\cite{wang2004ssim} and a mean absolute error of $0.060$ on the $[-1,1]$ pixel scale, so the decoder's own reconstruction error is about $2.6\times$ smaller than the tracker's final-frame error against the decoded targets reported below. Encoder and decoder use exponential-moving-average weights, and every image in the tracking study is a decoder output $D_\theta(z)$, so the targets lie on the decoder's image manifold by construction. 

The latent diffusion U-Net has $17.0$~M parameters; base width $128$, two residual blocks per level, multi-head self-attention at every level and at bottleneck.

\begin{figure}[t]
\centering

\includegraphics[width=0.49\textwidth]{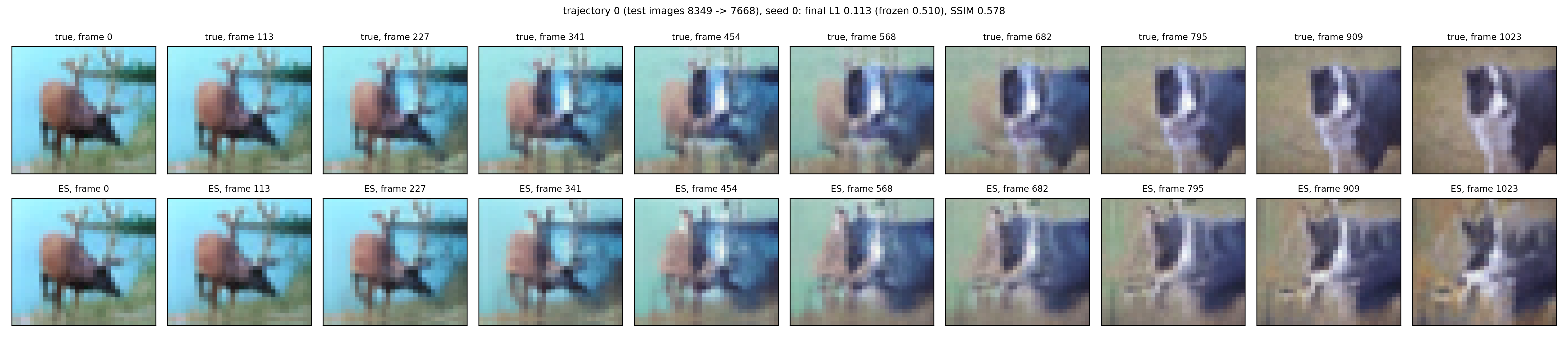}
\includegraphics[width=0.49\textwidth]{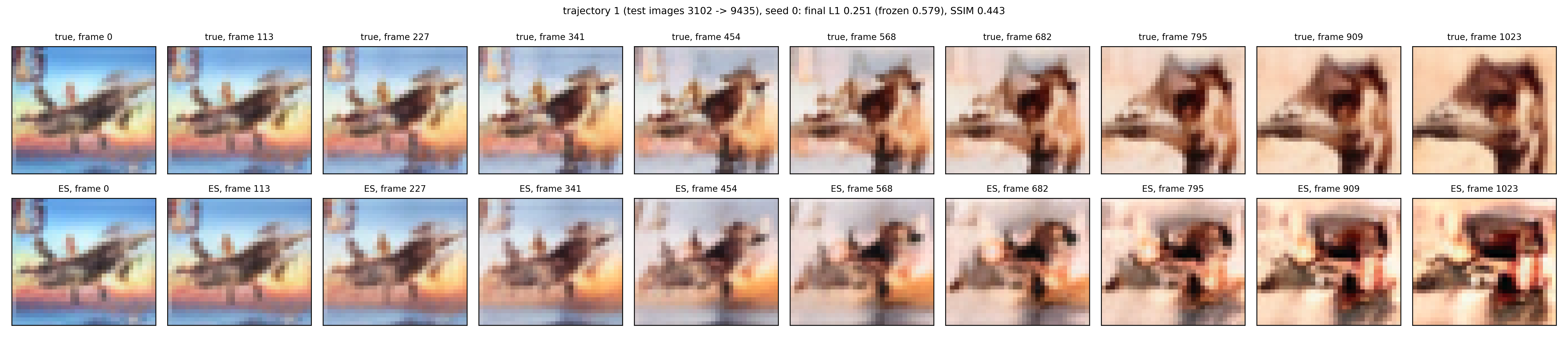}
\includegraphics[width=0.49\textwidth]{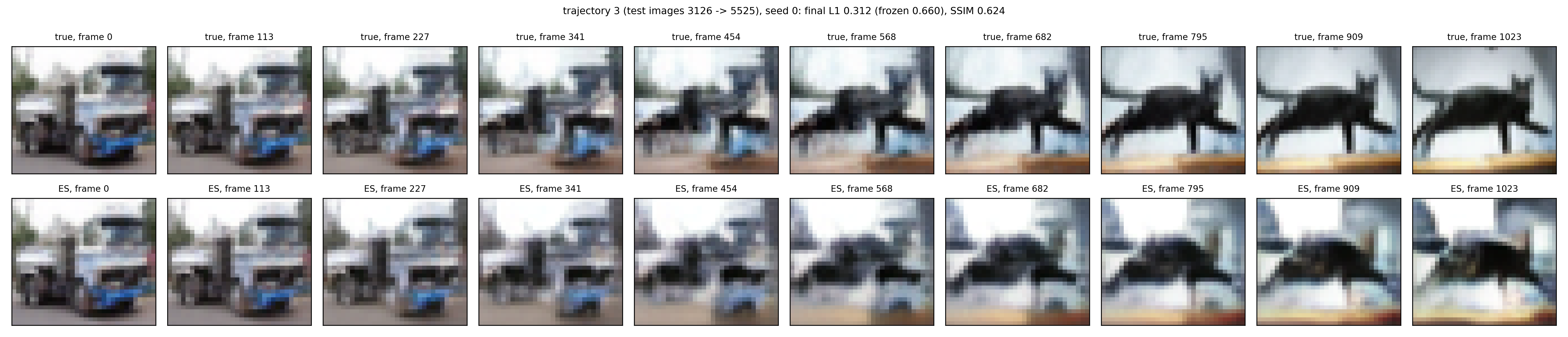}
\includegraphics[width=0.49\textwidth]{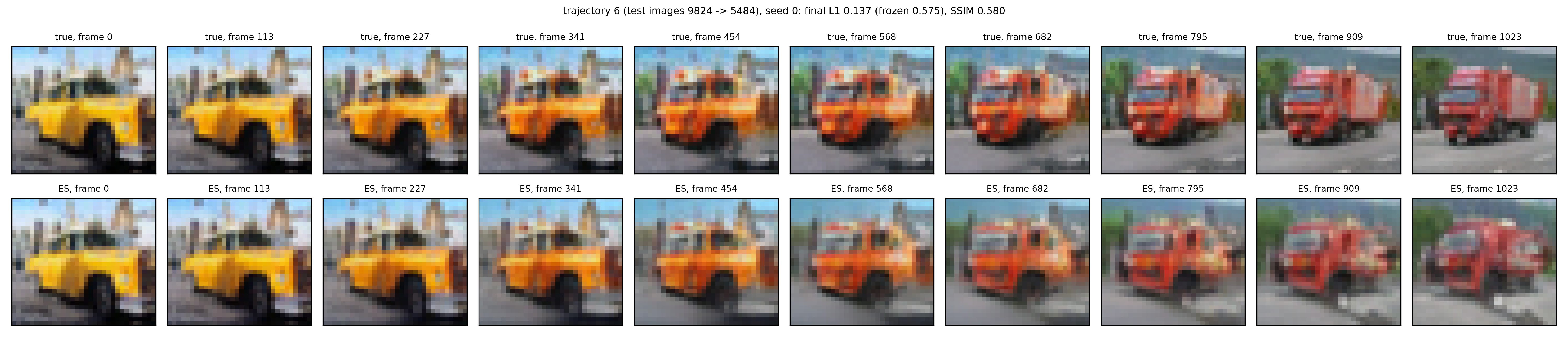}
\includegraphics[width=0.49\textwidth]{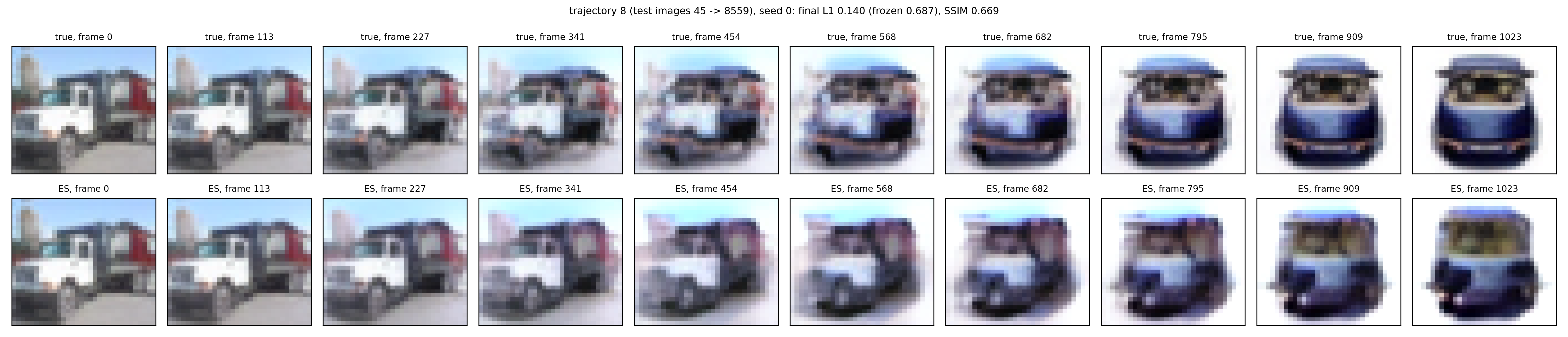}
\includegraphics[width=0.49\textwidth]{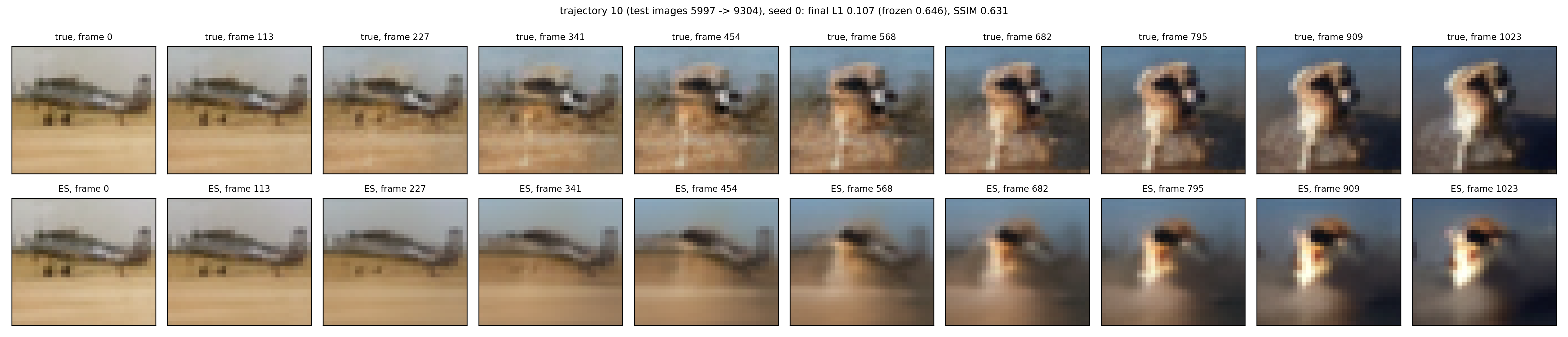}
\includegraphics[width=0.49\textwidth]{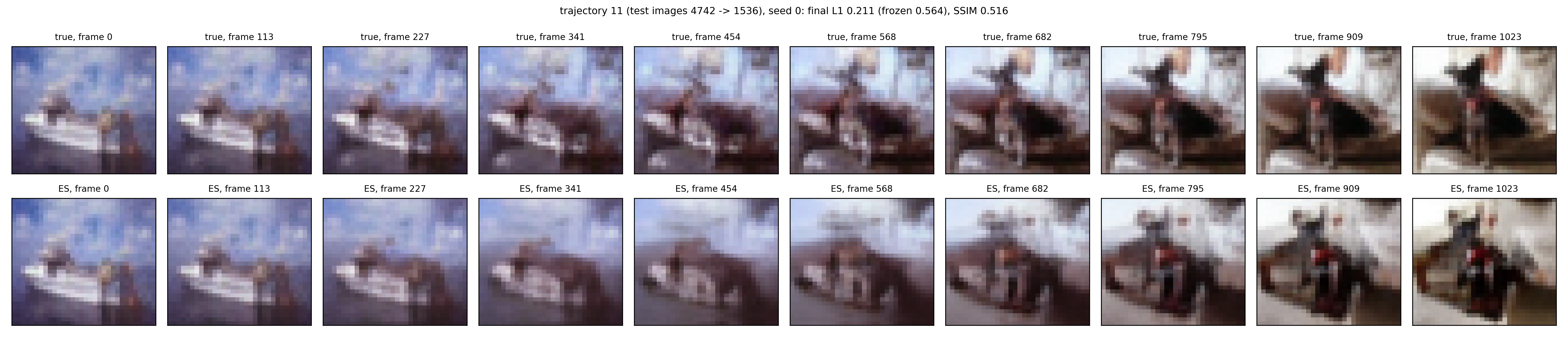}
\includegraphics[width=0.49\textwidth]{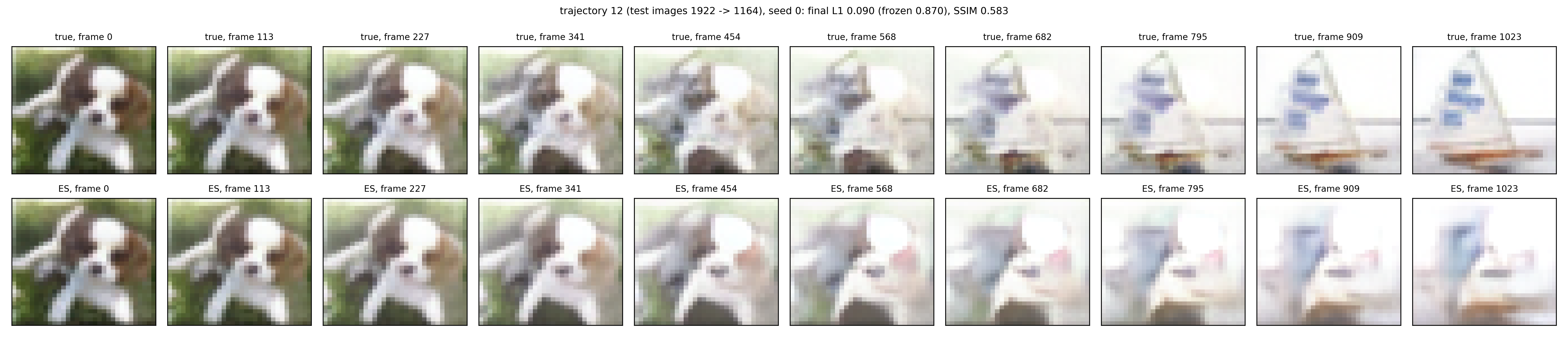}

\includegraphics[width=0.49\textwidth]{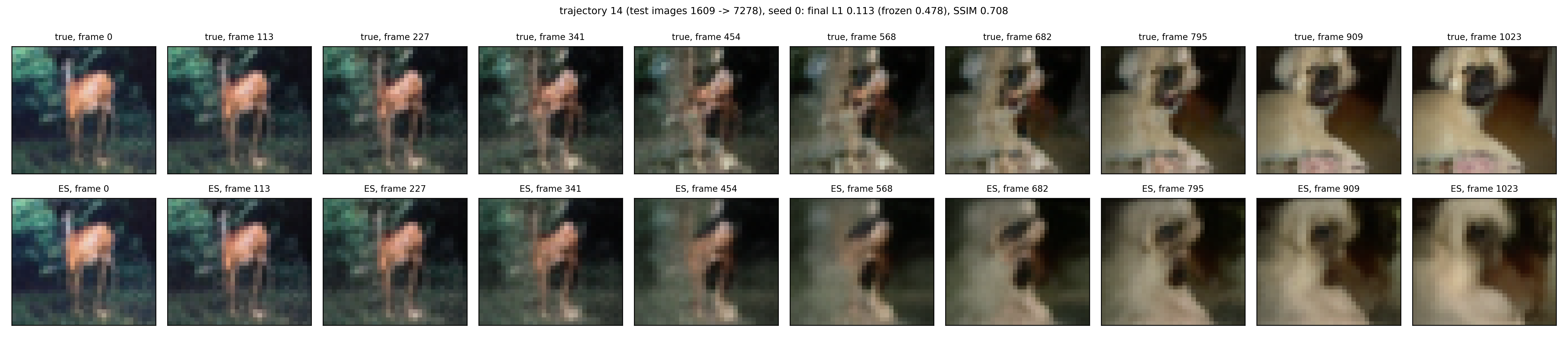}
\includegraphics[width=0.49\textwidth]{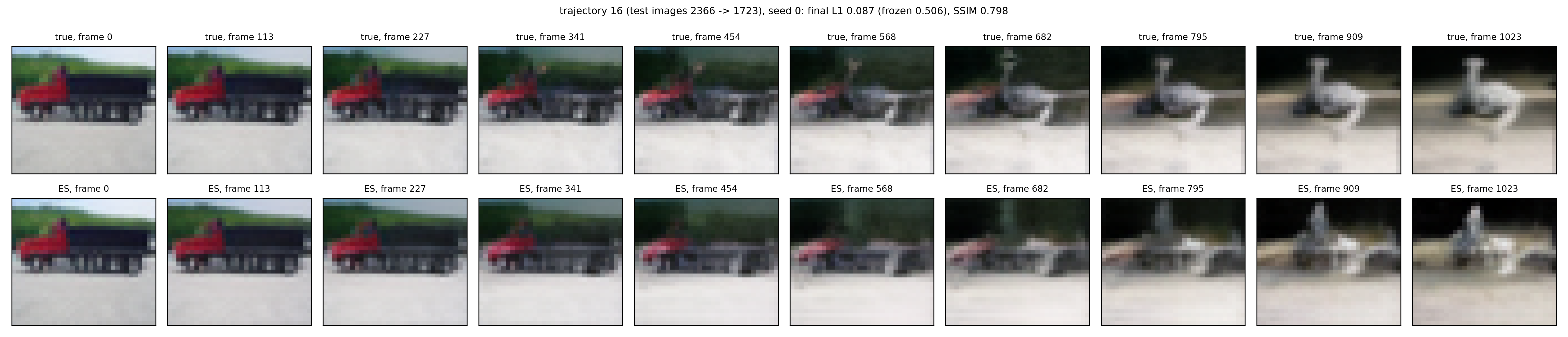}
\includegraphics[width=0.49\textwidth]{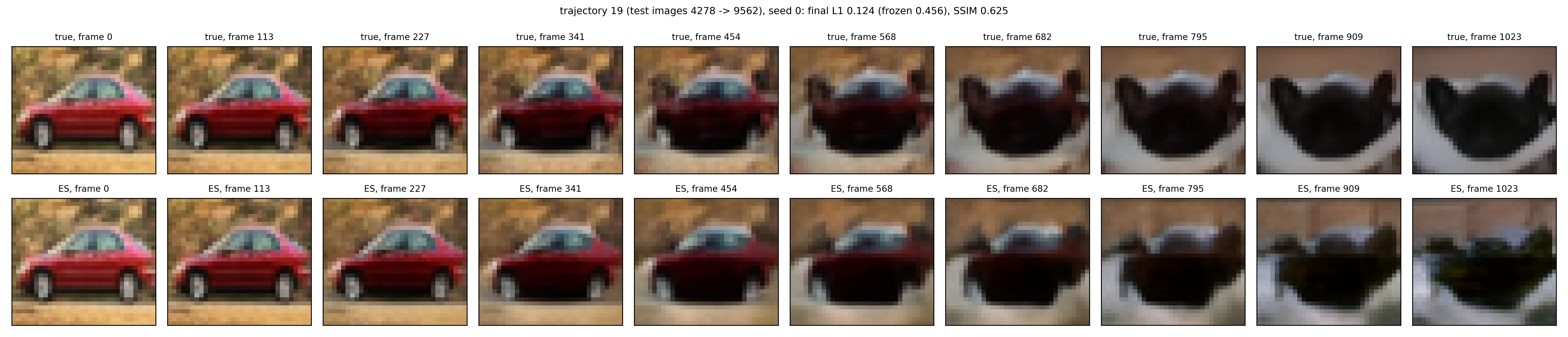}
\includegraphics[width=0.49\textwidth]{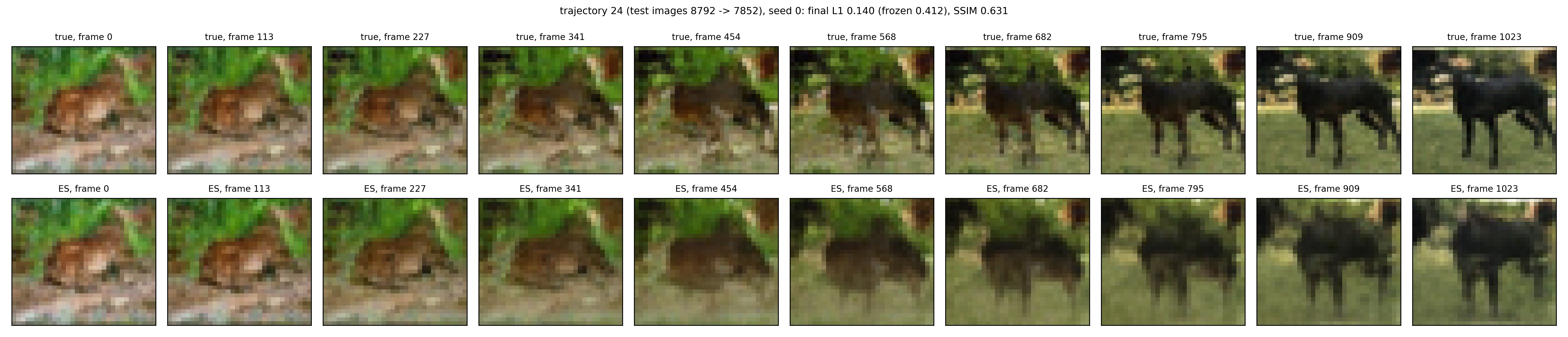}
\includegraphics[width=0.49\textwidth]{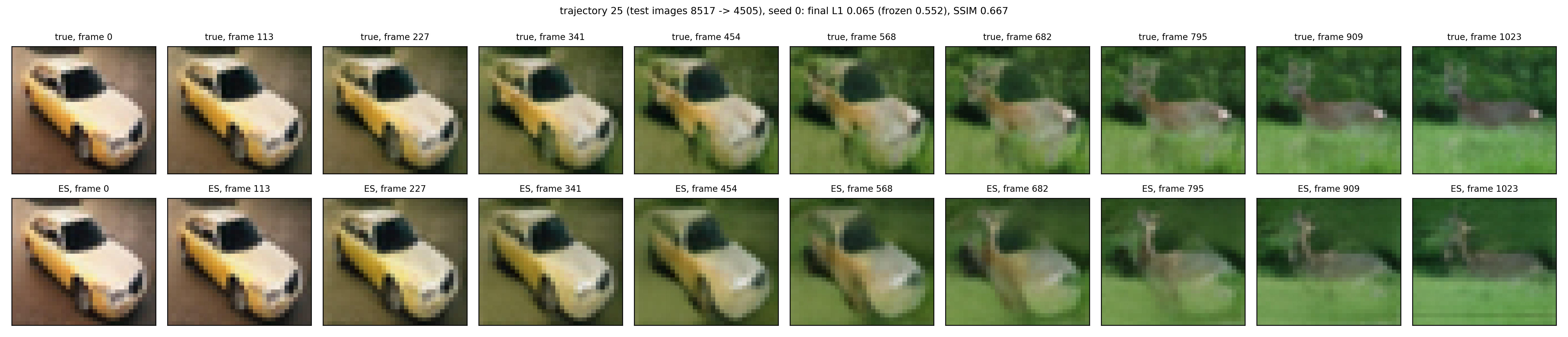}
\includegraphics[width=0.49\textwidth]{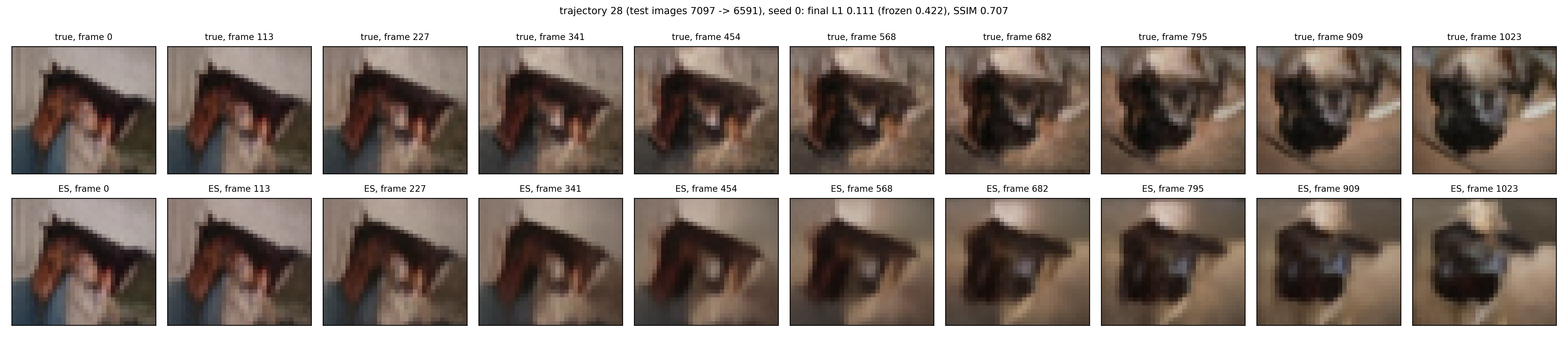}
\includegraphics[width=0.49\textwidth]{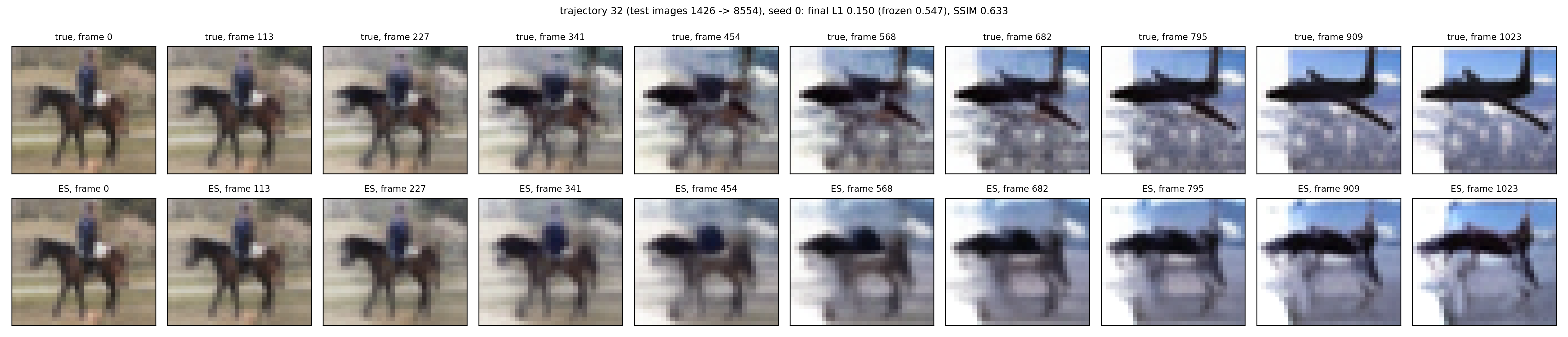}
\includegraphics[width=0.49\textwidth]{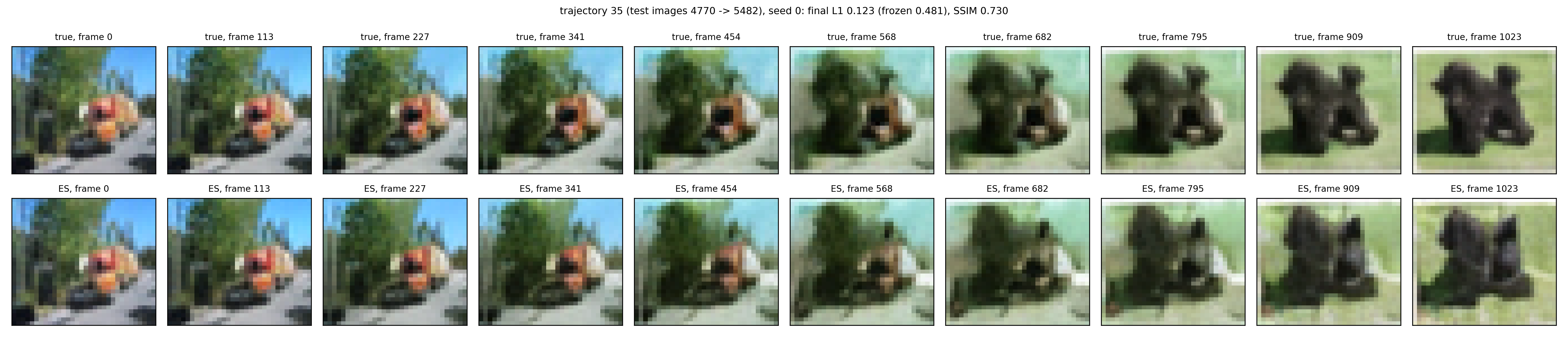}

\includegraphics[width=0.49\textwidth]{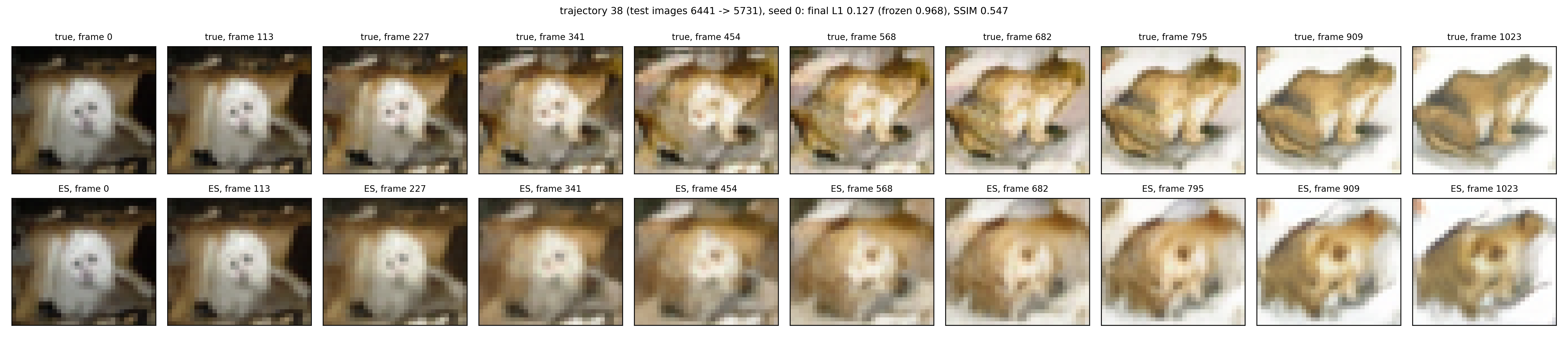}
\includegraphics[width=0.49\textwidth]{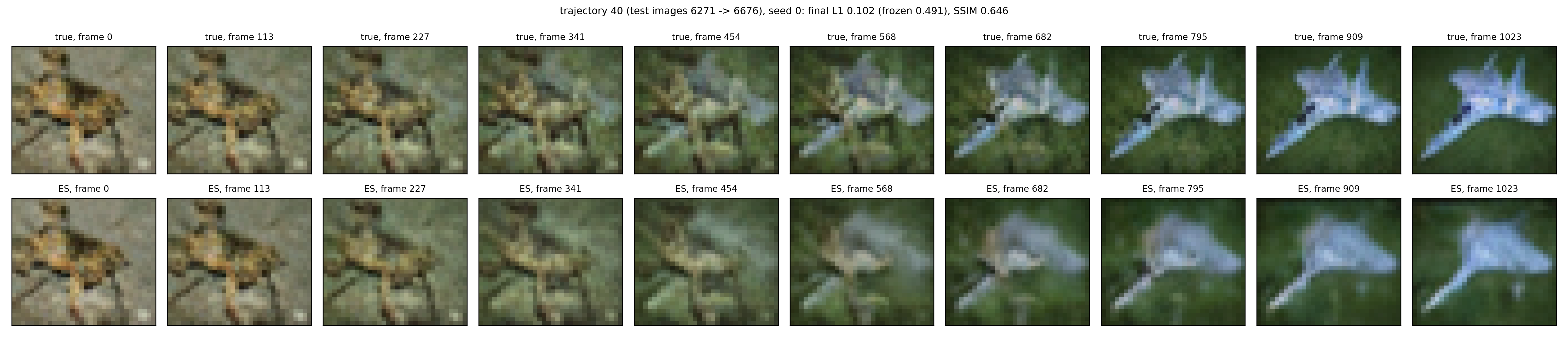}
\includegraphics[width=0.49\textwidth]{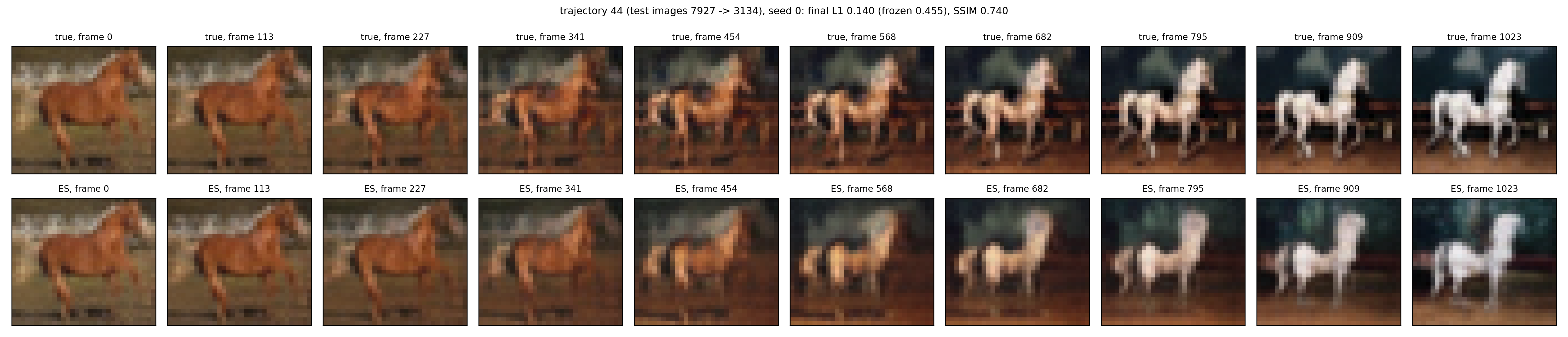}
\includegraphics[width=0.49\textwidth]{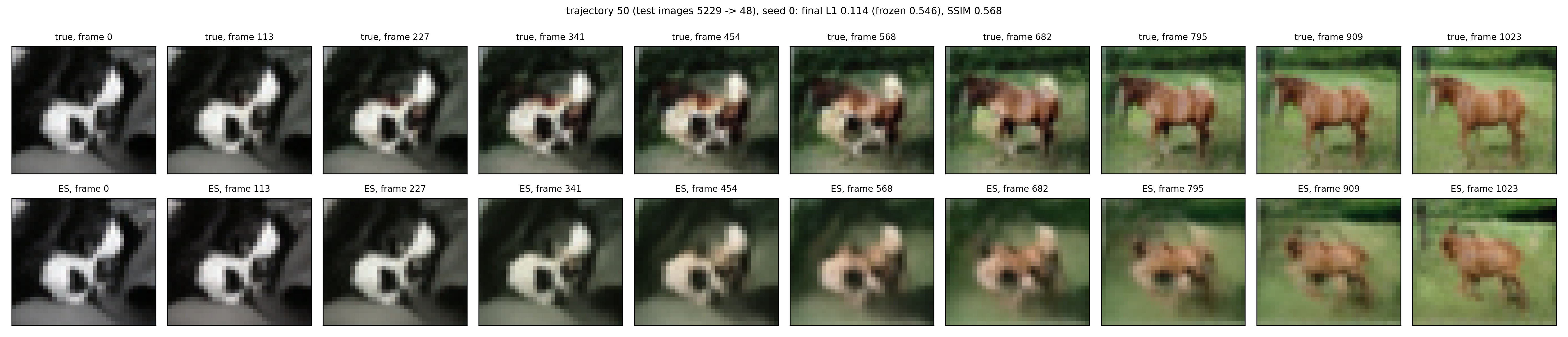}
\includegraphics[width=0.49\textwidth]{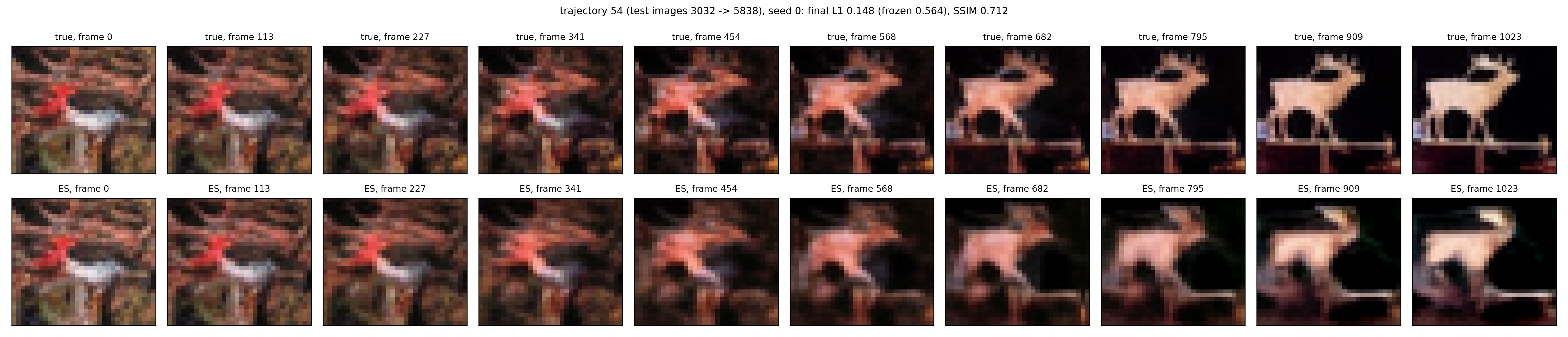}
\includegraphics[width=0.49\textwidth]{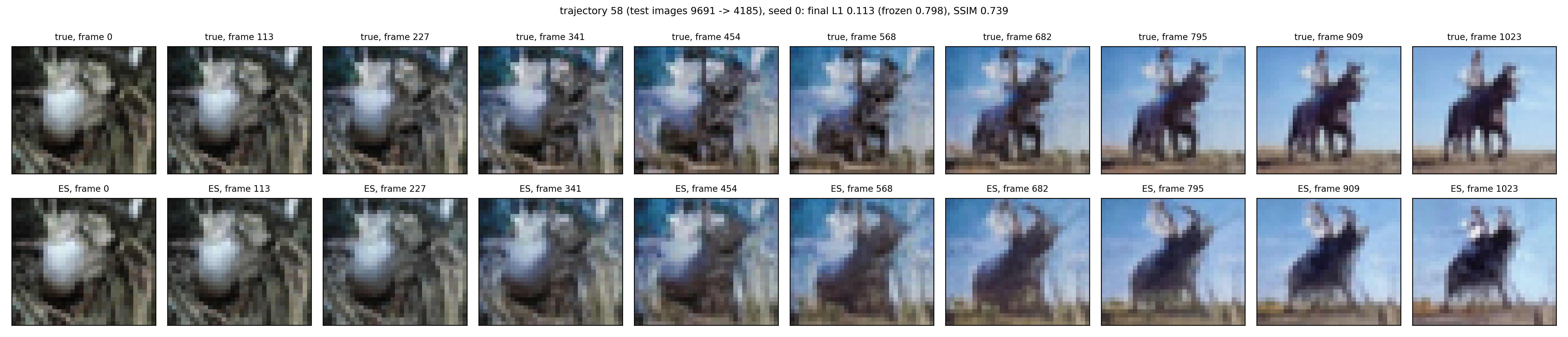}
\includegraphics[width=0.49\textwidth]{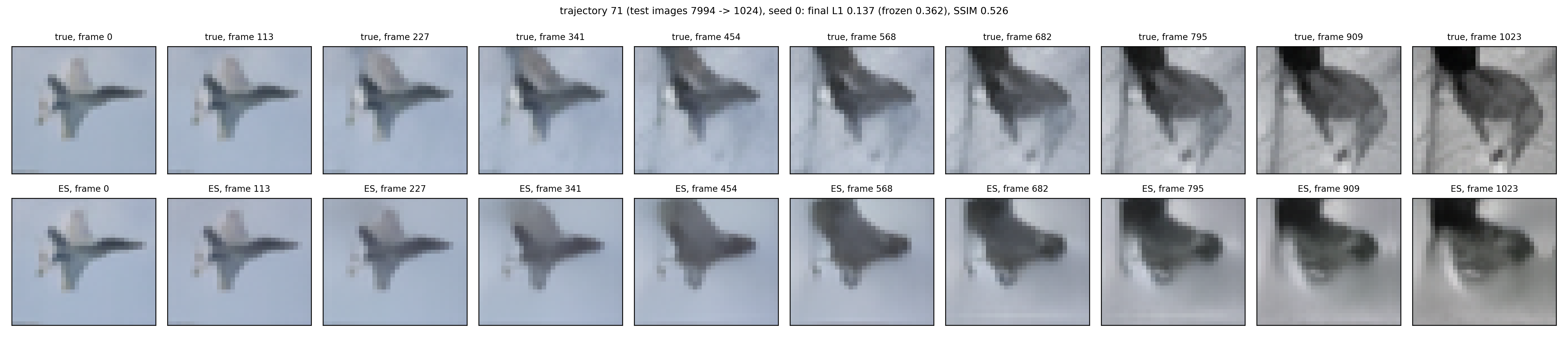}
\includegraphics[width=0.49\textwidth]{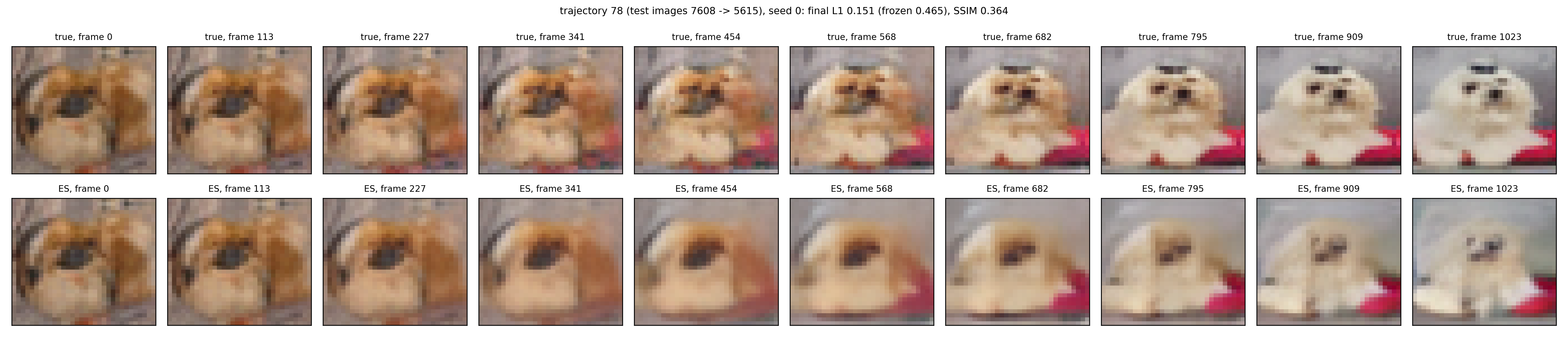}

\includegraphics[width=0.49\textwidth]{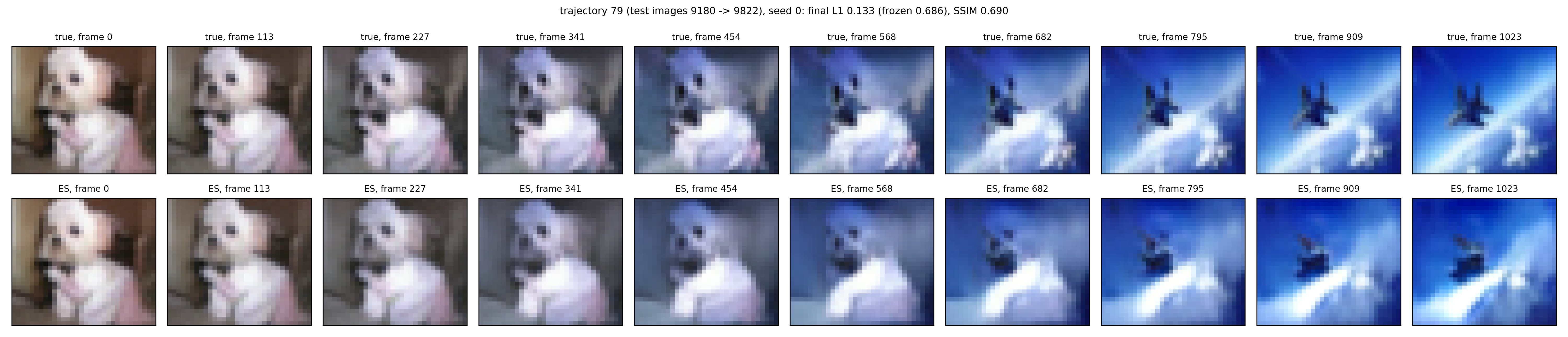}
\includegraphics[width=0.49\textwidth]{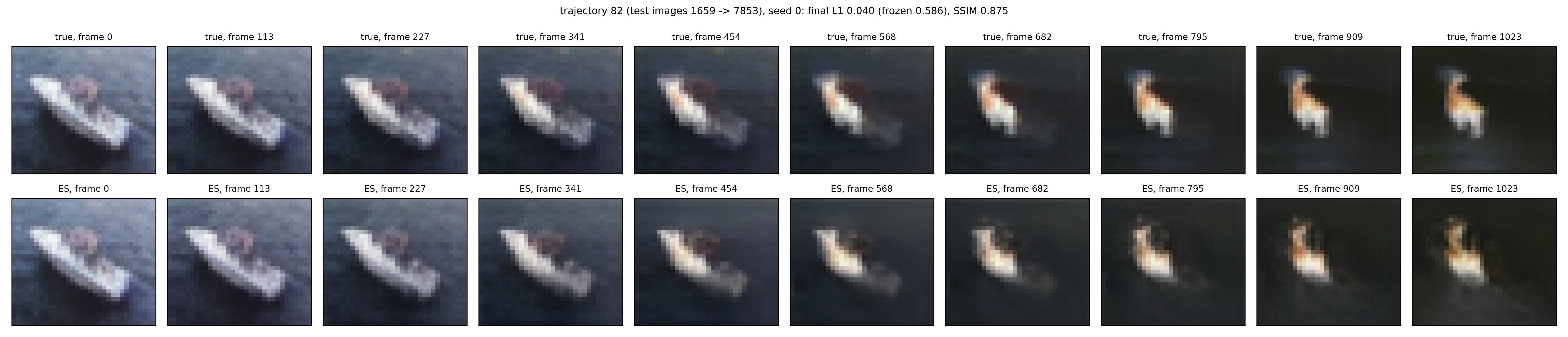}
\includegraphics[width=0.49\textwidth]{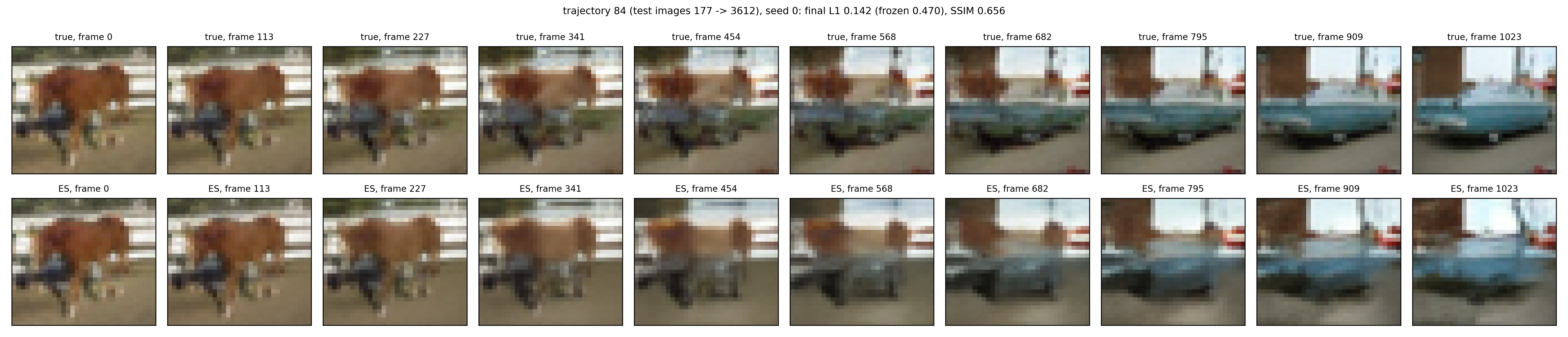}
\includegraphics[width=0.49\textwidth]{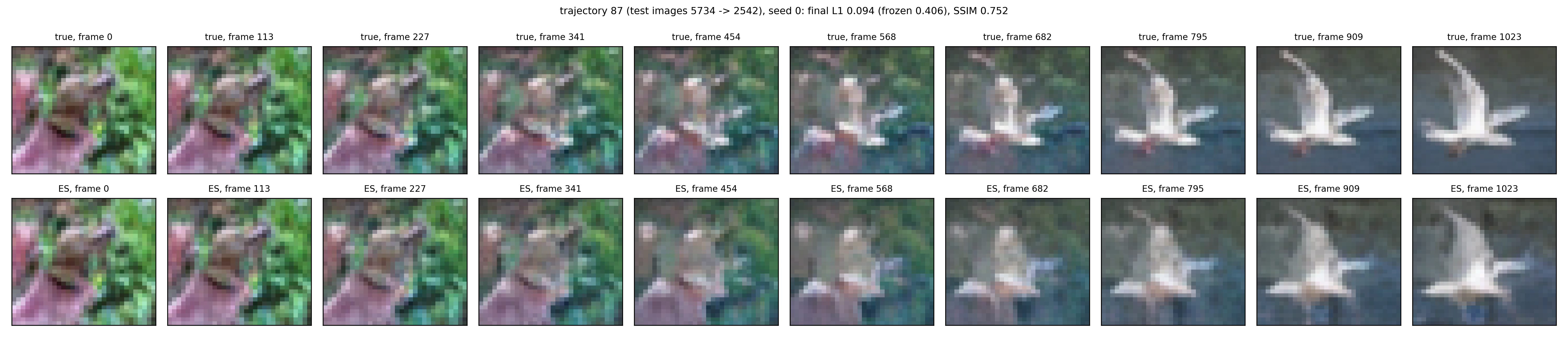}

\caption{ES dynamics results for tracking time-varying CIFAR-10 images based on 1D projections.}
\label{fig:CIFAR10_more1}
\end{figure}

\section{Dimension dependence of the finite-frequency threshold and
the evaluation cost of ES}
\label{supp:dimension}

Theorem~\ref{thm:main} is dimension-free \emph{in form}: for every
$(\delta,\varepsilon,T)$ there is a finite $\omega^*$, whatever the state
dimension. All dependence on the dimension is compressed into
$\omega^*(\delta,\varepsilon,T)$ through the constant $C_R'$ of
\eqref{eq:supp-gronwall}, which silently sums the $O(d^2)$
cross-coordinate oscillatory integrals of Step~2, and through the
frequency gaps $|r_n-r_i|$ entering the constants of
Lemma~\ref{lem:osc-main-ref}. Read naively, the theory appears to work
equally well in any dimension; in experiments the finite-$\omega$
threshold grows linearly with $d$
(\eqref{eq:supp-threshold}, Table~\ref{tab:supp-windows}). This
appendix closes the gap between those two statements. We prove:
(i)~an explicit finite-$\omega$ error bound whose constants carry the
dimension polynomially, so that a \emph{sufficient} threshold
$\omega^*(d)$ is polynomial in $d$
(Proposition~\ref{prop:dim-upper});
(ii)~a matching-in-$d$ \emph{lower} bound: for the evenly spaced
frequency design there are $C^\infty$ targets for which no
$\omega\lesssim d/T$ suffices, so
$\omega^*(d)=\Omega\!\big(d/((r_{\max}-1)T)\big)$
(Proposition~\ref{prop:dim-lower}), the rigorous counterpart of the
measured law \eqref{eq:supp-threshold}; and
(iii)~the consequence the practitioner feels: combined with the
step-resolution requirement of Algorithm~\ref{alg:es}, the number of
model evaluations needed to track the averaged dynamics grows at least
linearly in $d$, for \emph{every} choice of the frequency band
(Theorem~\ref{thm:dim-nfe}). The averaged limit itself never sees the
dimension; the entire cost is paid through the exchange rate $\omega$
between the fast and slow clocks of \eqref{eq:supp-fastclock}.

Throughout this appendix $d$ denotes the state dimension (the $N$ of
Appendix~\ref{supp:notation}), and the frequency ratios are
$\{r_i\}_{i=1}^d\subset[1,r_{\max}]$; to avoid a notation collision we
write $r_{\max}$ for the top of the band, denoted $N$ in
Appendix~\ref{supp:num-threshold}. For the evenly spaced design,
adjacent ratios differ by
$\delta_r=(r_{\max}-1)/(d-1)$, so the smallest beat frequency is
$\Delta_{\min}=\omega\delta_r$ as in \eqref{eq:supp-gapfreq}.

\subsection{Dimension-normalized setting}

Comparing constants across dimensions is meaningful only if the
regularity of the target is normalized per coordinate. We therefore
work, in this appendix, under the following strengthening of
Assumptions~\ref{ass:reg}--\ref{ass:growth}, chosen to cover the
sampling instantiations (I), (I$_\lambda$), (III) of
Sec.~\ref{sec:inst} and the synthetic testbed of
Appendix~\ref{supp:num-testbed}.

\begin{assumption}[Dimension-normalized regularity; constant diffusion]
\label{ass:dim}
$\f\equiv\mathbf 0$ and $g_n\equiv g>0$ for all $n$ (constant scalar
diffusion, $\bsig=gI$). $J\in C^{2,1}(\R^d\times[0,T])$ with the
\emph{global} bounds
\begin{equation}
\sup_{\x,t}\,|\partial_i J|\le G_1\ \ \forall i,
\qquad
\sup_{\x,t}\,\|\nabla_\x J\|_2\le \bar G_1,
\qquad
\sup_{\x,t}\,\|\nabla^2_\x J\|_{\mathrm{op}}\le H_2,
\qquad
\sup_{\x,t}\,|\partial_t J|\le G_t .
\label{eq:supp-dim-primitives}
\end{equation}
\end{assumption}

The four primitives are deliberately kept separate: $G_1$, $H_2$, and
$G_t$ are per-coordinate (dimension-robust) quantities, while
$\bar G_1\le\sqrt d\,G_1$ may grow with $d$ in the worst case but is
$O(1)$ for targets whose gradient is carried by $O(1)$ coordinates. On
the Gaussian-mixture testbed \eqref{eq:supp-gmm}, restricted to any
fixed ball, all four are bounded uniformly in $d$. Global (rather than
compact-set) bounds let us drop the localization $\tau_R$ and keep the
radius $R$ out of the dimension accounting; on compacts, every
statement below holds verbatim with the primitives evaluated on
$\bar B_R\times[0,T]$ and the stopping time of \eqref{eq:supp-tau}
reinstated.

\subsection{A frequency-resolved oscillatory bound}

The first step is to expose the two hidden dependences of
Lemma~\ref{lem:osc-main-ref}: on the oscillation frequency $|c|\omega$
(through the block averages) and on the dimension (through the speed
$\|\bmu_\omega\|_2=O(\sqrt{\alpha\omega r_{\max} d}\,)$ of the dithered
state and the size $O(\sqrt{d\Delta})$ of its Brownian increments).

\begin{lemma}[Oscillatory integrals, explicit in $|c|$ and $d$]
\label{lem:dim-osc}
Let Assumption~\ref{ass:dim} hold, let
$\Phi:\R^d\times[0,T]\to\R$ be bounded by $C_\Phi$ and jointly
Lipschitz with constant $L_\Phi$ (Euclidean norm in $\x$), and let
$c\neq0$. Then for all $\omega\ge1$,
\begin{eqnarray}
\Expec\Big[\sup_{t\le T}\Big|\int_0^{t}
\Phi\big(\x_\omega(s),s\big)\,e^{\mathrm i c\omega s}\,ds\Big|^2\Big]
\;\le\;
\kappa\,(1+d)\Big(1+\tfrac{1}{|c|}\Big)\,\omega^{-1/2}, \\
\kappa:=c_0(1+T)^2\big[C_\Phi^2+L_\Phi^2 g^2(1+\alpha r_{\max})\big],
\label{eq:supp-dim-osc}
\end{eqnarray}
with $c_0$ a universal constant. If $\Phi$ depends on $\x$ only through
a fixed subset $S$ of coordinates, then $(1+d)$ may be replaced by
$(1+|S|)$.
\end{lemma}

\begin{proof}
Repeat the proof of Lemma~\ref{lem:osc-main-ref} with the block length
chosen per frequency, $\Delta=|c|^{-1/2}\omega^{-3/4}$ (for
$\Delta>T$ or $\omega$ so small that the right side of
\eqref{eq:supp-dim-osc} exceeds $(C_\Phi T)^2$, the claim is trivial
after enlarging $c_0$). The frozen part obeys
$\sup_t|A|^2\le 8C_\Phi^2T^2/(\Delta^2c^2\omega^2)
=8C_\Phi^2T^2/(|c|\,\omega^{1/2})$. For the Lipschitz part, with
$\f\equiv0$ the drift is the dither alone, so on each block
$\|\x_s-\x_{t_j}\|\le g\sqrt{\alpha\omega r_{\max} d}\,\Delta
+\|\mathbf N_j(s)\|$ with
$\Expec\|\mathbf N_j\|^2\le g^2 d\,\Delta$; if $\Phi$ depends only on
coordinates $S$, the same holds with $d$ replaced by $|S|$ because only
the increments of those coordinates enter. Exactly as in
Lemma~\ref{lem:osc-main-ref},
\begin{align}
\Expec\sup_{t\le T}|B|^2
&\le 3L_\Phi^2T^2\big[g^2\alpha r_{\max}\,d\,\omega\Delta^2+\Delta^2\big]
+3L_\Phi^2T^2 g^2 d\,\Delta
\nonumber\\
&\le 3L_\Phi^2T^2 g^2(1+\alpha r_{\max})\,d
\Big[\tfrac{1}{|c|\,\omega^{1/2}}
+\tfrac{1}{|c|^{1/2}\omega^{3/4}}\Big]+3L_\Phi^2T^2\Delta^2 .
\end{align}
Since $\omega\ge1$,
$|c|^{-1/2}\omega^{-3/4}\le\tfrac12(|c|^{-1}\omega^{-1/2}+\omega^{-1})
\le(1+|c|^{-1})\omega^{-1/2}$, and $\Delta^2\le\omega^{-1/2}$;
collecting terms gives \eqref{eq:supp-dim-osc}.
\end{proof}

\subsection{Upper bound: a sufficient threshold polynomial in $d$}

Define the \emph{spectral crowding} of a ratio set,
\begin{equation}
\mathfrak S(r)\;:=\;\max_{1\le n\le d}\ \sum_{i\ne n}
\frac{1}{\sqrt{|r_n-r_i|}}\,,
\qquad
\mathfrak S(r)\;\le\;\frac{4d}{\sqrt{r_{\max}-1}}
\quad\text{(evenly spaced)},
\label{eq:supp-dim-crowding}
\end{equation}
the bound following from
$\sum_{i\ne n}|r_n-r_i|^{-1/2}
\le 2\delta_r^{-1/2}\sum_{g=1}^{d-1}g^{-1/2}
\le 4\sqrt{(d-1)/\delta_r}$.
$\mathfrak S$ is the quantity through which the frequency design enters
the error: it is small when the ratios are few or well separated, and
grows when $d$ frequencies crowd a fixed band.

\begin{proposition}[Explicit finite-$\omega$ error bound]
\label{prop:dim-upper}
Let Assumptions~\ref{ass:freq} and~\ref{ass:dim} hold, let $\x_\omega$
and $\bar\x$ be as in Theorem~\ref{thm:main}, and suppose the moment
bound $\Expec\sup_{t\le T}\|\bar\x\|^2<\infty$ of
Assumption~\ref{ass:growth}. There is a constant
$Q=Q(k,\alpha,g,r_{\max},T,G_1,\bar G_1,H_2,G_t)$, independent of $d$
and $\omega$, such that for all $\omega\ge1$,
\begin{equation}
\Expec\Big[\sup_{t\le T}\big\|\x_\omega(t)-\bar\x(t)\big\|^2\Big]
\;\le\;
Q\,e^{\Gamma T}
\Big[\;
\underbrace{d\,\big(1+d^{2}H_2^{2}+\bar G_1^{4}\big)\,\omega^{-1}}_{\text{ripple, boundary, martingale}}
\;+\;
\underbrace{d\,(1+d)\,\big(\mathfrak S(r)+d\big)^{2}\,\omega^{-1/2}}_{\text{beat pairs}}
\;\Big],
\label{eq:supp-dim-upper}
\end{equation}
with $\Gamma:=3(T+4)\big(k\alpha g^{2}H_2/2\big)^{2}$.
Consequently, for the evenly spaced design, a sufficient condition for
$\Prob\big(\sup_{t\le T}\|\x_\omega-\bar\x\|>\delta\big)\le\varepsilon$
is
\begin{equation}
\omega\;\ge\;
\omega^{*}_{\mathrm{suff}}(d)
\;=\;
C_{\mathrm{suff}}\,
\frac{d^{8}}{(r_{\max}-1)^{2}\,\delta^{4}\varepsilon^{2}}\,,
\label{eq:supp-dim-suff}
\end{equation}
with $C_{\mathrm{suff}}=C_{\mathrm{suff}}(k,\alpha,g,r_{\max},T,G_1,\bar
G_1,H_2,G_t)$. The exponent $8$ treats the four primitives as $d$-free,
which holds for coordinate-sparse gradients ($\bar G_1=O(1)$); for dense
isotropic energies with $\bar G_1=\Theta(\sqrt d)$ the Lipschitz constant
$L_\Phi\lesssim H_2+|k|G_1\bar G_1$ of the slow factors grows like
$\sqrt d$, the beat-pair term of \eqref{eq:supp-dim-upper} acquires an
extra factor $d$ through $\kappa$, and the sufficient exponent becomes
$10$ (and $5$ rather than $4$ under Remark~\ref{rem:dim-sharp}).
\end{proposition}

\begin{proof}
We retrace Steps~1--3 of Appendix~\ref{supp:proof-main}, tracking the
dimension. \emph{Step 1 terms.} With $g$ constant, $u^{c}=g\cos kJ$,
$u^{s}=g\sin kJ$, so $\|\nabla u\|_2\le kg\bar G_1$ and
$|\operatorname{tr}(\bsig\bsig^{\!\top}\nabla^2u)|
\le g^{3}k\big(dH_2+k\bar G_1^{2}\big)$ (the trace of a symmetric
matrix with operator norm $H_2$ is at most $dH_2$). Following the
estimates leading to \eqref{eq:supp-E1}, each coordinate obeys
$\Expec\sup_t|E^{(1)}_n|^2
\le C\big(1+d^{2}H_2^{2}+\bar G_1^{4}+G_t^{2}\big)\,\omega^{-1}$
with $C=C(k,\alpha,g,r_{\max},T)$, and summing over $n$ multiplies by
$d$; this is the first bracket of \eqref{eq:supp-dim-upper}.
\emph{Step 2 terms.} For each ordered pair $(n,i)$ the nonresonant
integrals of \eqref{eq:supp-cross} have envelopes
$\Phi\in\{\partial_iu^{c/s}\,g\cos kJ,\ \partial_iu^{c/s}\,g\sin kJ\}$
with $C_\Phi\le kg^{2}G_1$ and $L_\Phi\le
Ckg^{2}(H_2+k G_1\bar G_1)$, oscillating at the frequencies
$(r_n\mp r_i)\omega$ (and $2r_n\omega$ for $i=n$), with prefactor
$\alpha\sqrt{r_i/r_n}\le\alpha\sqrt{r_{\max}}$.
Lemma~\ref{lem:dim-osc} and the triangle inequality in
$L^2(\sup_t)$ give
\begin{align}
\big\|\,\sup_t|E^{(2)}_n|\,\big\|_{L^2}
&\;\le\;
C\,\sqrt{\kappa}\,(1+d)^{1/2}\,\omega^{-1/4}
\Big[\sum_{i\ne n}\Big(1+\tfrac{1}{|r_n-r_i|}\Big)^{1/2}
+\sum_{i\ne n}\Big(1+\tfrac{1}{r_n+r_i}\Big)^{1/2}+2\Big]
\nonumber\\
&\;\le\;
C'\sqrt{\kappa}\,(1+d)^{1/2}\big(\mathfrak S(r)+d\big)\,\omega^{-1/4},
\end{align}
and summing squares over $n$ contributes the second bracket.
\emph{Step 3.} With $\f\equiv0$ and $\bsig=gI$ the averaged drift is
$\bar\bmu=-\tfrac{k\alpha g^{2}}{2}\nabla J$, whose Euclidean Lipschitz
constant is exactly $\tfrac{k\alpha g^{2}}{2}H_2$ --- crucially free of
$d$ under the operator-norm bound of \eqref{eq:supp-dim-primitives} ---
and the diffusion difference vanishes. The Gr\"onwall step of
\eqref{eq:supp-gronwall} then multiplies the remainder by
$3e^{\Gamma T}$ with $\Gamma$ as stated, giving
\eqref{eq:supp-dim-upper}. For \eqref{eq:supp-dim-suff}, insert
\eqref{eq:supp-dim-crowding}, note that the second bracket dominates
for $\omega\ge1$ up to a constant, apply Chebyshev's inequality, and
absorb numerical factors into $C_{\mathrm{suff}}$.
\end{proof}

\begin{remark}[Sharper $\omega$-power under one more derivative]
\label{rem:dim-sharp}
If in addition $J\in C^{3}$ with
$\sup_{\x,t}\|\nabla_\x\nabla^2_\x J\|_{\mathrm{op}}$ bounded, the per-pair estimate can
be run through the It\^{o} integration by parts of
Remark~\ref{rem:rate-supp} instead of the partition argument, replacing
$(1+|c|^{-1})\omega^{-1/2}$ in \eqref{eq:supp-dim-osc} by
$(1+|c|^{-2})\omega^{-1}$ up to constants involving
$\|\nabla J\|_1\le dG_1$. Since
$\sum_{i\ne n}|r_n-r_i|^{-1}\le 2(d-1)(1+\ln(d-1))/(r_{\max}-1)$ for
the evenly spaced design, this trades powers between $d$ and $\omega$
and yields the alternative sufficient threshold
$\omega^{*}_{\mathrm{suff}}\le
C\,d^{4}(1+\ln d)^{2}/\big((r_{\max}-1)^{2}\delta^{2}\varepsilon\big)$.
We do not optimize further: the role of
Proposition~\ref{prop:dim-upper} is to certify that the hidden
constants are \emph{polynomial} in $d$, and Remark~\ref{rem:dim-gap}
explains why the true threshold is smaller still.
\end{remark}

\subsection{Lower bound: the beat obstruction is real}

Distinct frequencies (Assumption~\ref{ass:freq}) are not a
technicality. We first record what happens when they fail, then turn
the same mechanism, quantitatively, into a lower bound on
$\omega^{*}$.

\begin{proposition}[A coincident pair changes the limit]
\label{prop:dim-pair}
Let Assumption~\ref{ass:dim} hold and suppose $r_n=r_m$ for one pair
$n\neq m$, all other ratios pairwise distinct and distinct from
$r_n$. Then the conclusion of Theorem~\ref{thm:main} holds with the
averaged SDE \eqref{eq:supp-averaged} replaced by
\begin{equation}
d\bar\x=-\tfrac{k\alpha g^{2}}{2}\,M\,\nabla_{\bar\x}J(\bar\x,t)\,dt
+g\,d\w,
\qquad
M:=I+\mathbf e_n\mathbf e_m^{\!\top}+\mathbf e_m\mathbf e_n^{\!\top}.
\label{eq:supp-dim-M}
\end{equation}
$M$ is singular, with eigenvalue $0$ along
$(\mathbf e_n-\mathbf e_m)/\sqrt2$. In particular, for instantiation
(I) with the quadratic energy $\E(\x)=\tfrac12\|\x\|^{2}$ the modified
dynamics have \emph{no invariant probability law}: along
$\mathbf e_n-\mathbf e_m$ the drift vanishes identically and the state
performs free Brownian motion.
\end{proposition}

\begin{proof}
Rerun Step 2 of Appendix~\ref{supp:proof-main} for coordinate $n$. The
term $i=m$ of \eqref{eq:supp-cross} now oscillates at the frequencies
$0$ and $2r_n\omega$, exactly as the term $i=n$ does; retaining its
zero-frequency half and applying the same algebra that produced
\eqref{eq:supp-resonant} (with $\partial_m$ in place of $\partial_n$;
the $\partial g$ terms are absent since $g$ is constant) yields, in
addition to \eqref{eq:supp-resonant}, the resonant contribution
$-\tfrac{k\alpha g^{2}}{2}\int_0^t\partial_mJ(\x_\omega(s),s)\,ds$, and
symmetrically for coordinate $m$. All remaining estimates are
unchanged, giving \eqref{eq:supp-dim-M}. The spectral claim is
immediate; for the quadratic energy,
$M\nabla\E(\x)=M\x$ vanishes on $\ker M$, so the
$(\mathbf e_n-\mathbf e_m)$-component of the state solves
$d\zeta=\sqrt2\,g\,d\beta$ for a scalar Brownian motion $\beta$, which
is null recurrent.
\end{proof}

At finite $\omega$ with a small but nonzero gap, the extra term of
Proposition~\ref{prop:dim-pair} does not vanish; it merely oscillates
at the beat frequency $\Delta=\omega|r_n-r_m|$, and until the beat has
completed a fraction of a cycle the trajectory cannot distinguish the
correct limit from \eqref{eq:supp-dim-M}. Quantifying this yields the
lower bound. The construction is deliberately \emph{two-body}: a
linear target supported on a single coordinate makes every cross term
except the designated pair vanish identically, so the obstruction can
be isolated with dimension-free constants.

\begin{proposition}[$\Omega(d)$ frequency threshold for the evenly
spaced design]
\label{prop:dim-lower}
Fix $k,\alpha,g,T>0$, $\gamma>0$, and $r_{\max}>1$, and let the ratios
be evenly spaced on $[1,r_{\max}]$. Let $J(\x,t)=\gamma x_1$, which
satisfies Assumption~\ref{ass:dim} with $G_1=\bar G_1=\gamma$,
$H_2=G_t=0$. Set
\begin{equation}
\delta_0:=\frac{k\alpha g^{2}\gamma\,T}{2\pi\sqrt{r_{\max}}}\,.
\end{equation}
There exist constants $\omega_1,V<\infty$ and
$\varepsilon_0:=\delta_0^{2}/(4V)>0$, all independent of $d$ and
$\omega$, such that for every $d\ge d_0$ and every
\begin{equation}
\omega_1\;\le\;\omega\;\le\;\frac{\pi\,(d-1)}{2T\,(r_{\max}-1)}
\label{eq:supp-dim-window}
\end{equation}
one has
$\Prob\big(\sup_{t\le T}\|\x_\omega(t)-\bar\x(t)\|>\delta_0/2\big)
\ge\varepsilon_0$. Consequently
\begin{equation}
\omega^{*}(\delta,\varepsilon,T)\;\ge\;
\frac{\pi\,(d-1)}{2T\,(r_{\max}-1)}
\qquad\text{for all }\delta<\delta_0/2,\ \varepsilon<\varepsilon_0,\
d\ge d_0,
\label{eq:supp-dim-lower}
\end{equation}
where $d_0:=1+\lceil 2T(r_{\max}-1)\,\omega_1/\pi\rceil$. The
threshold of Theorem~\ref{thm:main} therefore grows at least linearly
in the dimension at fixed band, matching the measured law
\eqref{eq:supp-threshold} up to the constant.
\end{proposition}

\begin{proof}
Take the adjacent pair $(1,2)$, so
$\Delta:=\omega(r_2-r_1)=\omega(r_{\max}-1)/(d-1)$, and let
$\x_\omega,\bar\x$ share the Brownian motion and start from
$\x_0=\mathbf 0$. Since $\partial_iJ=0$ for $i\ge2$, the averaged
dynamics of coordinate $2$ are pure noise,
$d\bar x_2=g\,dw_2$, and $u^{c}=g\cos(k\gamma x_1)$,
$u^{s}=g\sin(k\gamma x_1)$ depend on $\x$ only through $x_1$, with
$\partial_iu^{c/s}=0$ for $i\ge2$.

Apply Steps~1--2 of Appendix~\ref{supp:proof-main} to coordinate
$n=2$. Because $\nabla u$ has a single nonzero component, every
constant in the Step-1 estimate \eqref{eq:supp-E1} is independent of
$d$, and in Step 2 the sum over $i$ in \eqref{eq:supp-cross} reduces
to the single term $i=1$. Expanding its trigonometric products, the
sum-frequency parts oscillate at $(r_1+r_2)\omega\ge2\omega$ and the
self term of coordinate 2 at $2r_2\omega$; all are covered by
Lemma~\ref{lem:dim-osc} with envelopes depending on $\x$ only through
$x_1$ (so with $(1+d)$ replaced by $2$) and with
$C_\Phi\le k\gamma g^{2}$,
$L_\Phi\le k\gamma g^{2}(1+k\gamma)$. The difference-frequency part is
special: writing $S=\sin k\gamma x_1$, $C=\cos k\gamma x_1$, the two
groups of \eqref{eq:supp-cross} combine to
\begin{equation}
k\alpha g^{2}\gamma\sqrt{\tfrac{r_1}{r_2}}
\int_0^{t}\Big[
SC\,\sin\big((\omega_1{+}\omega_2)s\big)
-\tfrac12\big(S^{2}{+}C^{2}\big)\cos\big(\Delta s\big)
+\tfrac12\big(S^{2}{-}C^{2}\big)\cos\big((\omega_1{+}\omega_2)s\big)
\Big]ds ,
\end{equation}
in which the identity $S^{2}+C^{2}=1$ removes the state from the beat
term entirely: the slow contribution is the \emph{deterministic}
integral
$-\tfrac{k\alpha g^{2}\gamma}{2}\sqrt{r_1/r_2}\int_0^t\cos(\Delta
s)\,ds$. Collecting, there is a $d$-free constant
$C_4=C_4(k,\alpha,\gamma,g,r_{\max},T)$ with
\begin{equation}
x_{\omega,2}(t)-\bar x_2(t)
\;=\;
-\frac{k\alpha g^{2}\gamma}{2}\sqrt{\frac{r_1}{r_2}}
\int_0^{t}\cos(\Delta s)\,ds\;+\;\mathcal R(t),
\qquad
\Big\|\sup_{t\le T}|\mathcal R(t)|\Big\|_{L^2}\le C_4\,\omega^{-1/4}.
\label{eq:supp-dim-rep}
\end{equation}
Under \eqref{eq:supp-dim-window}, $\Delta T\le\pi/2$, so
$\int_0^{T}\cos(\Delta s)\,ds=\sin(\Delta T)/\Delta\ge(2/\pi)T$, and
$\sqrt{r_1/r_2}\ge r_{\max}^{-1/2}$; hence
$|\Expec[\,x_{\omega,2}(T)-\bar x_2(T)\,]|
\ge 2\delta_0-C_4\omega^{-1/4}\ge\delta_0$ once
$\omega\ge\omega_1:=(C_4/\delta_0)^{4}$. Moreover the deterministic
term in \eqref{eq:supp-dim-rep} is bounded by
$\bar m:=\tfrac{k\alpha g^{2}\gamma}{2}T$ pathwise, so
$\Expec\big[(x_{\omega,2}(T)-\bar x_2(T))^{2}\big]\le
V:=2\bar m^{2}+2C_4^{2}$. Writing
$X=|x_{\omega,2}(T)-\bar x_2(T)|$, the Paley--Zygmund inequality gives
$\Prob(X\ge\tfrac12\Expec X)\ge(\Expec X)^{2}/(4\,\Expec X^{2})
\ge\delta_0^{2}/(4V)=\varepsilon_0$, and since
$\sup_{t\le T}\|\x_\omega-\bar\x\|\ge X$ pathwise, the displayed
probability bound follows. \eqref{eq:supp-dim-lower} is its
contrapositive against the definition of
$\omega^{*}(\delta,\varepsilon,T)$ in Theorem~\ref{thm:main}.
\end{proof}

\begin{remark}[Phase averaging of the Brownian ripple]
\label{rem:dim-noise}
The proof extends verbatim to any diffusion level because the beat
term in \eqref{eq:supp-dim-rep} is deterministic; the Brownian motion
influences only $\mathcal R$. Had the beat phase retained the state
(as the sum-frequency terms do), the noise in $x_1$ would multiply the
mean by the Gaussian phase-averaging factor
$e^{-k^{2}\gamma^{2}g^{2}s/2}$ --- the same mechanism by which
measurement noise entering the ES phase is demodulated rather than
amplified (Appendix~\ref{supp:num-optim}). The cancellation
$S^{2}+C^{2}=1$ that removes the state from the beat phase is thus
what makes the obstruction robust at every noise level.
\end{remark}

\subsection{Evaluation complexity: convergence speed versus dimension}

Propositions~\ref{prop:dim-upper} and~\ref{prop:dim-lower} bound the
frequency; the practitioner pays in \emph{model evaluations}.
Algorithm~\ref{alg:es} performs one evaluation of $J$ per step and
must resolve the fastest dither: with $c\ge1$ steps per fastest
period, $\Delta t\le 2\pi/(c\,r_{\max}\omega)$, so the number of
evaluations on the horizon $[0,T]$ is
$\mathrm{NFE}=T/\Delta t\ge(c/2\pi)\,r_{\max}\,\omega\,T$ (the
validated choice throughout this paper is $c=10$;
Appendix~\ref{supp:num-rate}).

\begin{theorem}[Linear-in-dimension evaluation cost]
\label{thm:dim-nfe}
In the setting of Proposition~\ref{prop:dim-lower}, consider a run of
Algorithm~\ref{alg:es} with evenly spaced ratios on $[1,r_{\max}]$,
any $r_{\max}>1$, any horizon $T>0$, and $c\ge1$ steps per fastest
dither period, simulating the continuous process \eqref{eq:supp-sde}
at a base frequency $\omega\ge\omega_1$ in the averaging regime of
Proposition~\ref{prop:dim-lower}. If that continuous process achieves
\begin{equation}
\Prob\Big(\sup_{t\le T}\|\x_\omega(t)-\bar\x(t)\|>\delta_0/2\Big)
<\varepsilon_0
\end{equation}
for the target of Proposition~\ref{prop:dim-lower} with $d\ge d_0$
(equivalently, $T\le\pi(d-1)/(2(r_{\max}-1)\omega_1)$, so that the window
\eqref{eq:supp-dim-window} is nonempty), then the run must use
\begin{equation}
\mathrm{NFE}
\;\ge\;
\frac{c}{2\pi}\,r_{\max}\,\omega\,T
\;\ge\;
\frac{c}{4}\,(d-1)\,\frac{r_{\max}}{r_{\max}-1}
\;\ge\;
\frac{c}{4}\,(d-1)\,.
\label{eq:supp-dim-nfe}
\end{equation}
The bound is algebraically independent of the horizon and of the band:
widening $[1,r_{\max}]$ relaxes the frequency threshold
\eqref{eq:supp-dim-lower} and tightens the step size in exact
proportion, so no choice of the band escapes the linear cost; the horizon
enters only through the applicability condition $d\ge d_0(T)$. The
restriction $\omega\ge\omega_1$ excludes the sub-threshold regime
$\omega<\omega_1$, which Proposition~\ref{prop:dim-lower} does not
analyze (there the remainder bound in \eqref{eq:supp-dim-rep} exceeds
$\delta_0$, and the dither is too slow to average at all); the accuracy
condition is placed on the continuous process, and the step count is
what Algorithm~\ref{alg:es} needs to resolve it at $c$ steps per period.
Conversely, by Proposition~\ref{prop:dim-upper} (or
Remark~\ref{rem:dim-sharp}), a polynomial budget
$\mathrm{NFE}=O\big(\mathrm{poly}(d)\,T/(\delta^{p}\varepsilon^{q})\big)$
always suffices under Assumption~\ref{ass:dim}.
\end{theorem}

\begin{proof}
By Proposition~\ref{prop:dim-lower}, since $\omega\ge\omega_1$ and the
window \eqref{eq:supp-dim-window} is nonempty for $d\ge d_0$, the
accuracy requirement forces $\omega>\pi(d-1)/(2T(r_{\max}-1))$; insert this into
$\mathrm{NFE}\ge(c/2\pi)r_{\max}\omega T$ and use
$r_{\max}/(r_{\max}-1)>1$.
\end{proof}

Theorem~\ref{thm:dim-nfe} is the precise sense in which
``convergence speed depends on dimension'' for stochastic bounded ES,
and it is invisible at the level of Theorem~\ref{thm:main}: the
\emph{averaged} dynamics, and hence the physical settling time, are
dimension-free (measured directly in
Appendix~\ref{supp:num-conv}); what grows with $d$ is the number of
scalar evaluations needed per unit of averaged time, because $\omega$
--- the exchange rate between the fast and slow clocks of
\eqref{eq:supp-fastclock} --- must itself grow linearly with $d$ to
keep the beat spectrum resolvable. Equivalently, in the language of
lock-in detection~\citep{scofield1994lockin}: the band
$[\omega,r_{\max}\omega]$ must host $d$ carrier channels whose pairwise
beats each complete $\Omega(1)$ cycles within the horizon, a bandwidth
budget $(r_{\max}-1)\,\omega\ge\pi(d-1)/(2T)$ that is exactly
\eqref{eq:supp-dim-lower}.

\begin{remark}[Consistency with zeroth-order information bounds]
\label{rem:dim-info}
A linear-in-$d$ evaluation cost is not an artifact of the dither
design: any procedure that extracts a $d$-dimensional descent
direction from one scalar evaluation per step is subject to
$\Omega(d)$-type oracle lower bounds
\citep{jamieson2012query,duchi2015optimal}, the same accounting that
gives simultaneous-perturbation and smoothing estimators their
dimension factors \citep{spall1992spsa,nesterov2017random}. What
Proposition~\ref{prop:dim-lower} and Theorem~\ref{thm:dim-nfe} add is
\emph{where} the universal cost surfaces in ES --- not in the number
of evaluations per gradient estimate, which remains one, but in the
frequency budget of the beat spectrum --- and that the constants of
the mechanism match the measured design law
\eqref{eq:supp-threshold} and Table~\ref{tab:supp-windows}.
\end{remark}

\begin{remark}[The gap between the bounds, and the measured law]
\label{rem:dim-gap}
Between the sufficient threshold
$\omega^{*}_{\mathrm{suff}}=O\big(d^{4}\mathrm{polylog}(d)\big)$ of
Remark~\ref{rem:dim-sharp} and the necessary threshold
$\Omega(d)$ of Proposition~\ref{prop:dim-lower}, the measurements of
Appendix~\ref{supp:num-threshold} place the truth at the lower end:
$\omega^{*}(d)\approx O(10)\,(d-1)/(r_{\max}-1)$, linear up to the
mildly superlinear drift of the transition centers in
Table~\ref{tab:supp-local}. The lossy step in
Proposition~\ref{prop:dim-upper} is identifiable: the triangle
inequality over the $O(d^{2})$ beat pairs treats their oscillations as
if they could align, whereas distinct beat frequencies are mutually
incoherent; replacing the triangle inequality by the orthogonality
heuristic replaces $\mathfrak S(r)^{2}$ by
$\sum_{i\neq n}|r_n-r_i|^{-1}=O\big(d\log d/(r_{\max}-1)\big)$ and
predicts $\omega^{*}=O(d\,\mathrm{polylog}(d))$, consistent with the
measured law --- and the growing number of near-minimal gaps supplies
the observed superlinear correction. Converting the heuristic into a
proof requires a joint (rather than pairwise) treatment of the beat
family; we record it as an open problem. None of this affects the
conclusions used in the main text, which rely only on the two proved
directions: the threshold is finite and polynomial
(Proposition~\ref{prop:dim-upper}) and at least linear
(Proposition~\ref{prop:dim-lower}), and the evaluation cost of ES
grows at least linearly with dimension (Theorem~\ref{thm:dim-nfe}).
\end{remark}

\section{Additional CelebA-HQ tracking results}
\label{supp:celeba}

Figure \ref{fig:FaceTurnMore1} shows additional
time-varying tracking runs on CelebA-HQ $256\times256$ under the 1-D
projection objective.

\begin{figure}[h]
\centering
  \includegraphics[width=1.0\textwidth]{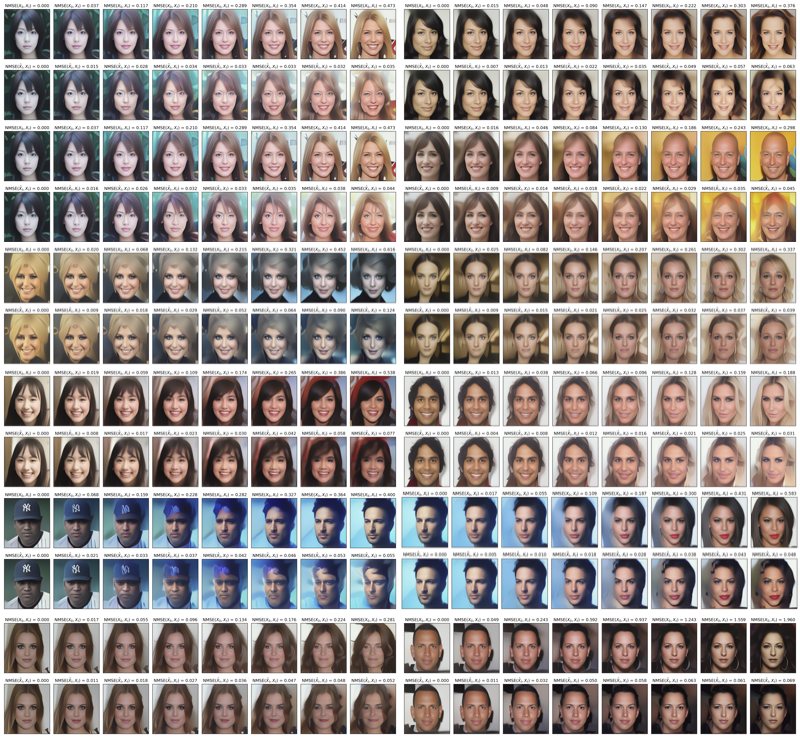}
  \caption{Additional CELEBA-HQ tracking examples.}
  \label{fig:FaceTurnMore1}
\end{figure}

\section{$\lambda$-tempering example}

We numerically demonstrate the $\lambda$-tempering result of Eq. \ref{eq:tempering} with a numerical example as shown in Figure \ref{fig:alpha}.

\begin{SCfigure}[50][t]
  \centering
  \includegraphics[width=0.5\textwidth]{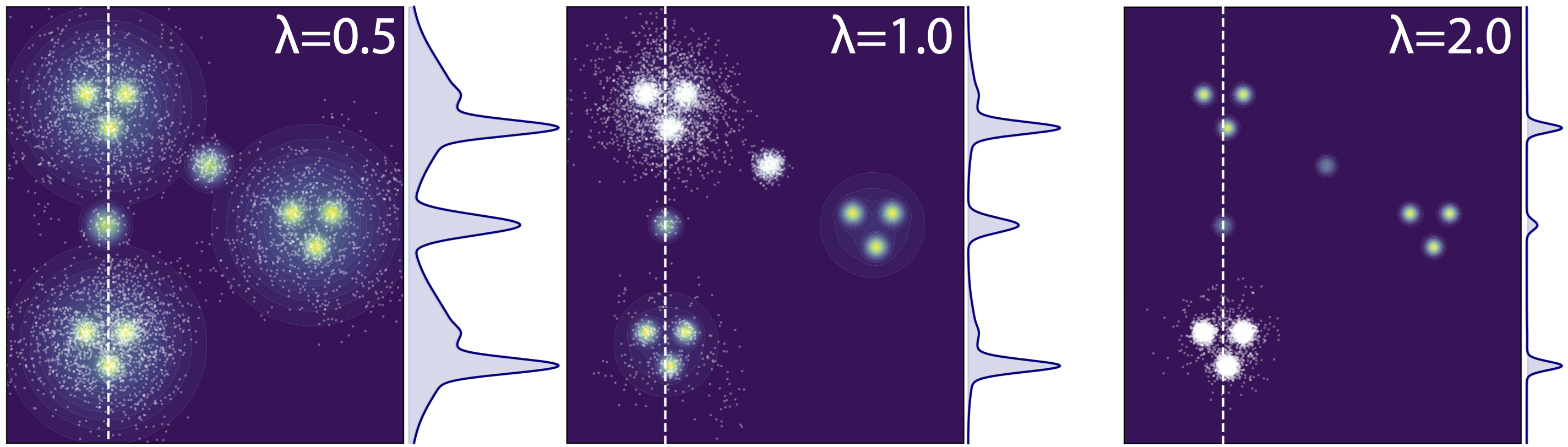}
  \caption{Effect on $\rho^\lambda$ of $\lambda$-tempering ES sampling of 2D multimodal target (white dots: visited locations). As $\lambda$ decreases, local peaks and valleys are escaped, and the sampler becomes mass-covering; large $\lambda$ concentrates on modes.}
\label{fig:alpha}
\end{SCfigure}

\end{document}